\documentclass{article}

\usepackage{microtype}
\usepackage{graphicx}
\usepackage{subfigure}
\usepackage{booktabs} 
\usepackage{multirow}

\usepackage{hyperref}
\usepackage[dvipsnames]{xcolor}

\usepackage[accepted]{icml2025}

\usepackage{amsmath}
\usepackage{amssymb}
\usepackage{mathtools}
\usepackage{amsthm}
\usepackage{caption}
\newcommand{\mathbbm}[1]{\text{\usefont{U}{bbm}{m}{n}#1}}

\usepackage[capitalize,noabbrev]{cleveref}

\theoremstyle{plain}
\newtheorem{theorem}{Theorem}[section]
\newtheorem{proposition}[theorem]{Proposition}
\newtheorem{lemma}[theorem]{Lemma}

\theoremstyle{definition}
\newtheorem{definition}[theorem]{Definition}

\theoremstyle{remark}
\newtheorem{remark}[theorem]{Remark}
\newcommand{\indep}{\perp \!\!\! \perp}

\usepackage[textsize=tiny]{todonotes}

\icmltitlerunning{WATCH: Adaptive Monitoring for AI Deployments via Weighted-Conformal Martingales}

\begin{document}

\twocolumn[
\icmltitle{
\texorpdfstring{WATCH: Adaptive Monitoring for AI Deployments \\ via Weighted-Conformal Martingales}{WATCH: Adaptive Monitoring for AI Deployments via Weighted-Conformal Martingales}
}




\icmlsetsymbol{equal}{*}

\begin{icmlauthorlist}
\icmlauthor{Drew Prinster}{jhu}
\icmlauthor{Xing Han}{jhu}
\icmlauthor{Anqi Liu}{jhu}
\icmlauthor{Suchi Saria}{jhu}
\end{icmlauthorlist}

\icmlaffiliation{jhu}{Department of Computer Science, Johns Hopkins University, Baltimore, MD, USA}


\icmlcorrespondingauthor{Drew Prinster}{drewprinster@gmail.com}
\icmlcorrespondingauthor{Xing Han}{aaronhan223@gmail.com}

\icmlkeywords{Machine Learning, ICML}

\vskip 0.3in
]



\printAffiliationsAndNotice{}

\begin{abstract}
Responsibly deploying artificial intelligence (AI) / machine learning (ML) systems in high-stakes settings arguably requires not only proof of system reliability, but also continual, post-deployment monitoring to quickly detect and address any unsafe behavior. Methods for nonparametric sequential testing---especially conformal test martingales (CTMs) and anytime-valid inference---offer promising tools for this monitoring task. However, existing approaches are restricted to monitoring limited hypothesis classes or ``alarm criteria'' (e.g., detecting data shifts that violate certain exchangeability or IID assumptions), do not allow for online adaptation in response to shifts, and/or cannot diagnose the cause of degradation or alarm. In this paper, we address these limitations by proposing a weighted generalization of conformal test martingales (WCTMs), which lay a theoretical foundation for online monitoring for any unexpected changepoints in the data distribution while controlling false-alarms.
For practical applications, we propose specific WCTM algorithms that adapt online to mild covariate shifts (in the marginal input distribution), quickly detect harmful shifts, and diagnose those harmful shifts as concept shifts (in the conditional label distribution) or extreme (out-of-support) covariate shifts that cannot be easily adapted to. On real-world datasets, we demonstrate improved performance relative to state-of-the-art baselines.
\end{abstract}

\section{Introduction}

As AI/ML systems become integral to real-world applications, ensuring their safety and utility under evolving conditions is essential for responsible deployment. However, even meticulously trained models with apparent reliability guarantees can fail abruptly when shifts in the data distribution or operational environment violate the underlying conditions they were designed for \citep{amodei2016concrete}. 
That is, unforeseen deployment conditions can make it impossible to guarantee reliability for all situations in advance. 
Consequently, there is growing recognition for the need to continuously monitor deployed AI systems for determining when model updates are required to mitigate downstream harm \citep{vovk2021retrain, podkopaev2021tracking, feng2022clinical, feng2025not}. In this work, we moreover argue that such monitoring methods should ideally perform at least three key functions:
 (1) maintain end-user reliability and minimize unnecessary alarms by adapting online to mild or benign data shifts; (2) rapidly detect more extreme or harmful shifts that necessitate updates; and (3) identify the root-cause of degradation to inform appropriate recovery.

For example, consider a healthcare use case where the goal is to predict the risk of sepsis, a life-threatening infection, $Y$ from electronic health record inputs $X$ (e.g., vital signs, lab tests, medical history, demographics) using an AI/ML system
(e.g., \citet{adams2022prospective}). Various clinical data shifts pose challenges to monitoring in practice \citep{finlayson2021clinician}.
Figure \ref{fig1:synthetic_examples} illustrates hypothetical synthetic-data shifts that can be interpreted through this sepsis example, where for simplicity $X$ only represents a patient's age. An example of a benign shift is a mild shift in patient demographics primarily toward young adults; Figure \ref{fig1:synthetic_examples}a shows a corresponding covariate shift in the marginal $X$ distribution. Such a mild shift to a younger population would be benign, as younger individuals are well-represented in the training data and also tend to have lower, less variable (easier to predict) sepsis risk---so, there is no need for an alarm. On the other hand, harmful shift examples include if the AI tool were deployed with a much older population than observed in the training data (e.g., Figure \ref{fig1:synthetic_examples}b) or if a new microbial strain were to arise that was especially severe in children (e.g., Figure \ref{fig1:synthetic_examples}c). In any such harmful shift case, swift detection and root-cause analysis is essential to initiate and inform retraining, to ultimately minimize harm.

\begin{figure*}[!htb]
\centering
    \includegraphics[width=0.95\textwidth]{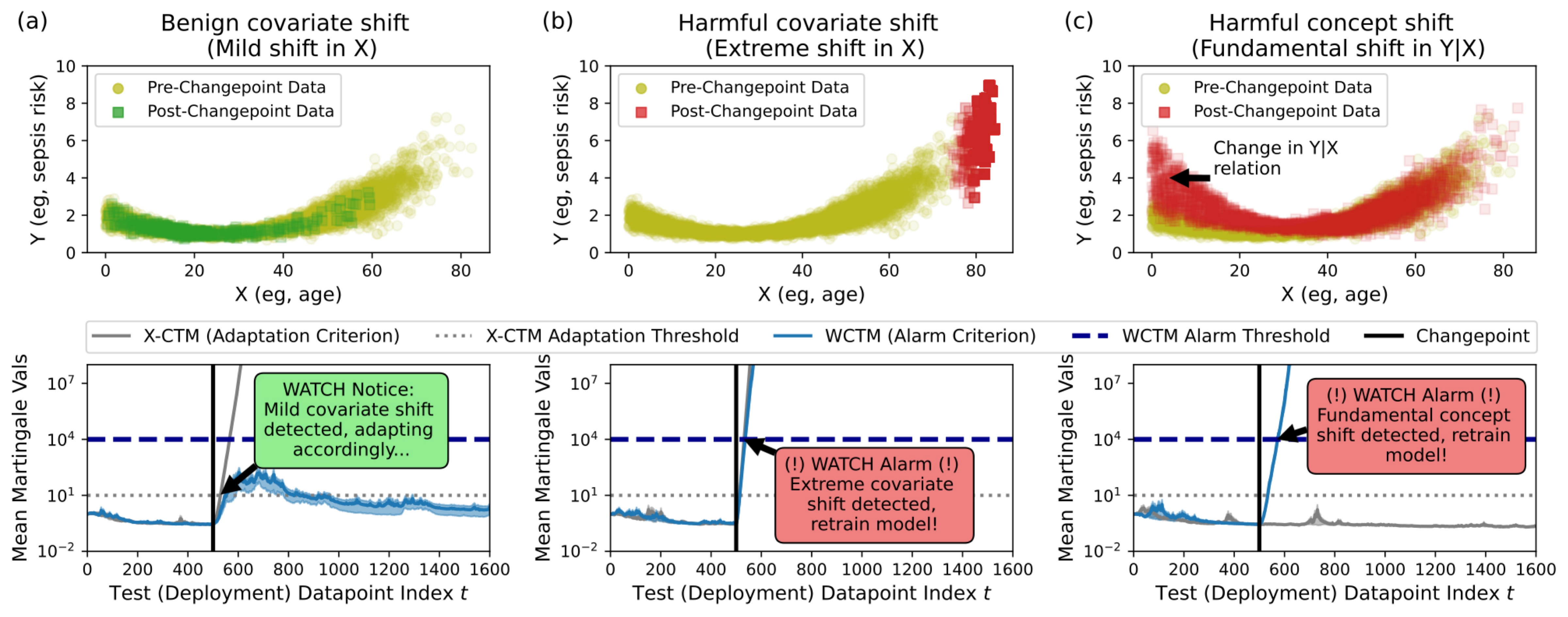}
    \vspace{-1em}
\caption{\small 
Each column represents a data shift scenario: the top row is a simulated shift example and the bottom row shows WATCH's response, averaged over 20 random seeds. WATCH raises an alarm to retrain the AI/ML once the WCTM (blue) exceeds its alarm threshold; meanwhile, an $X$-CTM (gray)---a standard CTM that only depends on inputs $X$, and thus only detects covariate shifts---dynamically initiates the WCTM's adaptation phase and aids in root-cause analysis. In (a), the $X$-CTM starts the WCTM's adaptation phase, which allows the WCTM to avoid raising an unnecessary alarm. In (b), the extreme covariate shift causes the WCTM to raise an alarm, indicating that the covariate shift is too severe to be adapted to. In (c), the illustrated concept shift causes WATCH to raise an alarm, but without the $X$-CTM detecting a shift in covariates $X$---this allows WATCH to diagnose the root-cause of the alarm as a concept shift in $Y\mid X$.
}
\label{fig1:synthetic_examples}
\end{figure*}

Recent advances in anytime-valid inference \citep{ramdas2023game} and especially conformal test martingales (CTMs) \citep{volkhonskiy2017inductive, vovk2021testing} offer promising tools for AI monitoring with sequential, nonparametric guarantees.
However, existing CTM monitoring methods (e.g., \citet{vovk2021retrain}) all rely on some form of \textit{exchangeability} (e.g., IID) assumption in their null hypotheses---informally, meaning that the data distribution is the same across time or data batches---and as a result, standard CTMs can raise unnecessary alarms even when a shift is mild or benign (e.g., Figure \ref{fig1:synthetic_examples}a). Meanwhile, existing comparable monitoring methods for directly tracking the risk of a deployed AI (e.g., \citet{podkopaev2021tracking}) tend to be less efficient than CTMs in their speed of detecting harmful shifts, data usage, and/or computational complexity (Sec. \ref{subsec:harmful}). 

Our paper's contributions can be summarized as follows:
\vspace{-0.25cm}
\begin{itemize}
    \item Our main theoretical contribution is to propose weighted-conformal test martingales (WCTMs), constructed from sequences of online weighted-conformal $p$-values, which generalize their standard conformal precursors. WCTMs lay a theoretical foundation for sequential and continual testing of a broad class of null hypotheses beyond exchangeability, such as shifts that one aims to model and adapt to.\footnote{Earlier versions of this paper erroneously omitted a key assumption (bag sufficiency) from the theory, which we corrected; practical and experimental claims are unaffected. See Remark \ref{remark:version_change}.}
    \item For practical applications, we propose 
    WATCH: \textbf{W}eighted \textbf{A}daptive \textbf{T}esting for \textbf{C}hangepoint \textbf{H}ypotheses, a framework for AI monitoring using WCTMs. WATCH continuously adapts to mild or benign distribution shifts (e.g., Figure \ref{fig1:synthetic_examples}a) to maintain end-user safety and utility (and avoid unnecessary alarms), while quickly detecting harmful shifts (e.g., Figure \ref{fig1:synthetic_examples}b \& c) and enabling root-cause analysis.
\end{itemize}

\section{Background}
\label{sec:background}
\textbf{Notation:} Assume an initial dataset $Z_{1:n}:= \{Z_i\}_{i=1}^n$, where each datapoint is a feature-label pair, $Z_i := (X_i, Y_i) \in \mathcal{X} \times \mathcal{Y}=\mathcal{Z}$. For simplicity, further assume that an AI/ML model $\widehat{\mu}$ is pretrained on a separate dataset.\footnote{This corresponds to a split conformal setting \citep{papadopoulos2008inductive}, 
but our theory also extends to full conformal \citep{vovk2022algorithmic}, which avoids splitting data at a heavy computational cost.} After deploying the AI/ML model, the test points are observed sequentially at each time $t=1, ..., T$ (batch data can be given random ordering), though for simpler exposition, we initially focus on $t=1$. We abbreviate indices $[m]:=\{1, ..., m\}$. Random variables are denoted with capital letters (e.g., $Z_i$) and observed values with lowercase (e.g., $z_i$). For $Z := Z_{1:(n+t)}$, we let $F_Z$ denote the joint distribution function, $f_Z$ the joint density function. For a set of distributions $\mathbf{F_Z}$, we write $Z_1, Z_2, ...\sim \mathbf{F_Z}$ to mean the sequence has some unknown distribution $F_Z\in\mathbf{F_Z}$.

\subsection{Conformal Prediction and Conformal $p$-Values}
\label{subsec:standard_conformal_pvals}

Standard conformal prediction (CP) \citep{vovk2022algorithmic} is an approach to \textit{predictive inference}: the task of converting a black-box AI/ML prediction, $\widehat{\mu}(X_{n+1})$, into a \textit{predictive confidence interval/set}, $\widehat{C}_{[n], \alpha}(X_{n+1})$, that should contain the true label with a user-specified rate, $1-\alpha \in (0, 1)$ (e.g., 90\%). This objective is called valid (marginal) \textit{coverage}:\footnote{Here, marginal means on average over the draw of calibration and test data; see, e.g., \citet{foygel2021limits} for more details.}
\begin{align}
    \mathbb{P}\big\{Y_{n+1} \in \widehat{C}_{[n], \alpha}(X_{n+1})\big\} \geq 1 - \alpha. \label{eq:coverage_condition}
\end{align}
Constructing the CP set, $\widehat{C}_{[n], \alpha}(X_{n+1})$, requires a labeled calibration dataset,\footnote{For $t>1$, the calibration data might be kept fixed (i.e., $Z_{1:n}$), or it may include past test points (i.e., $Z_{1:(n+t-1)}$); for now, we focus on $t=1$ to avoid this distinction and simplify exposition.} $Z_{1:n}$, and a ``nonconformity''\textit{score function}, $\mathcal{\widehat{S}}: \mathcal{X}\times\mathcal{Y} \rightarrow \mathbb{R}$, which generally uses the prefit ML predictor to quantify how ``strange'' a point $(x, y)$ is relative to the training data. A common example is the absolute-value residual score, $\mathcal{\widehat{S}}(x, y) = |y - \widehat{\mu}(x)|$. 

Though not always framed as such, the CP set $\widehat{C}_{[n], \alpha}(X_{n+1})$ can be computed via a \textit{conformal p-value}, which places the ``strangeness'' of the test point's score in context of the calibration-data scores. That is, with $v_i := \widehat{\mathcal{S}}(x_{i}, y_{i})$ for all $i\in [n+1]$, \citet{vovk2022algorithmic} defines the conformal $p$-value for test point $Z_{n+1}$ relative to calibration data $Z_{1:n}$ as the fraction of the $n+1$ scores, $v_{1:(n+1)}$, that are at least as large as $v_{n+1}$, with ties broken uniformly at random:
\begin{align}
    p_{n+1} := \frac{|\{i : v_i > v_{n+1}\}| + u_{n+1}|\{i : v_i = v_{n+1}\}|}{n+1},\label{eq:unweighted_p_vals}
\end{align}
where $U_{n+1}\stackrel{iid}{\sim} \text{Unif}[0, 1]$; that is, $u_{n+1}$ is obtained from an independent standard uniform distribution.\footnote{Conservative conformal $p$-values set $u_{n+1}:=1$; the random variable $P_{n+1}$ corresponding to $p_{n+1}$ is called a $p$-variable, but we will often refer to both as $p$-values for more common terminology.}  The CP set $\widehat{C}_{[n], \alpha}(X_{n+1})$ can then be understood as the subset of labels $y\in \mathcal{Y}$ that would not result in too ``extreme'' of a test-point $p$-value, $p_{n+1}(X_{n+1}, y)$ (this is explicit notation for $p_{n+1}$ to emphasize its dependence on the candidate label $y$):
\begin{align}\label{eq:cp_set_def_via_p_val}
    & \widehat{C}_{[n], \alpha}(X_{n+1}) := \big\{y \in \mathcal{Y} : p_{n+1}(X_{n+1}, y) > \alpha \big\}.
\end{align} 

\subsection{Exchangeability Underlies Standard CP Validity}

For the standard CP set in Eq. \eqref{eq:cp_set_def_via_p_val}, the coverage guarantee in Eq. \eqref{eq:coverage_condition} holds assuming that the calibration data $Z_{1:n}$ and test point $Z_{n+1}$ are all \textit{exchangeable}---intuitively, this means that the data distribution is invariant over time (e.g., independent and identically distributed or ``IID'' data). Formally, exchangeability means that the joint distribution, $F_Z$, is invariant to reorderings or permutations $\sigma$: that is, $F_Z(z_{\sigma(1)}, ..., z_{\sigma(n+1)}) = F_Z(z_{1}, ..., z_{n+1})$ for all $z_1, ..., z_{n+1}\in \mathcal{Z}$ and all $\sigma:[n+1]\rightarrow [n+1]$. 
With $\mathbf{F_Z^{\textbf{ex}}}$ denoting the set of exchangeable joint distributions, we can write the assumption of exchangeability as a (composite and nonparametric) null hypothesis:
\begin{align}
    \mathcal{H}_0^{\text{ex}} \ : \ Z_1, Z_2, ..., Z_{n+1} \sim \mathbf{F_Z^{\textbf{ex}}}.
    \label{eq:H_0_exchangeability}
\end{align}
Standard conformal $p$-values (Eq. \eqref{eq:unweighted_p_vals}) are valid or ``bona-fide'' $p$-values, in the usual statistical testing sense, for the exchangeability null hypothesis \citep{vovk2022algorithmic}.
 That is, assuming $\mathcal{H}_0^{\text{ex}}$, then $P_{n+1}\stackrel{iid}{\sim}\text{Unif}[0,1]$ \citep{vovk2021testing} and 
\begin{align}\label{eq:cp_p_val_validity}
    \mathbb{P}_{\mathcal{H}_0^{\text{ex}}}(P_{n+1}\leq \alpha) \leq \alpha \qquad \forall \ \alpha \in (0, 1). 
\end{align}
Thus, observing an extreme value of $p_{n+1}$ is evidence against $\mathcal{H}_0^{\text{ex}}$ (Eq. \eqref{eq:H_0_exchangeability}), or evidence for a distribution shift.



\subsection{Standard Conformal Test Martingales: Testing Exchangeability Online via Betting}
\label{subsec:standard_CTM}

Standard conformal test martingales (CTMs) \citep{volkhonskiy2017inductive, vovk2021testing, vovk2021retrain} continually aggregate information from a sequence of conformal $p$-values (Eq. \eqref{eq:unweighted_p_vals}) to perform \textit{online} testing of the exchangeability null, $\mathcal{H}_0^{\text{ex}}$. They can be constructed on top of any standard CP method, and thus on top of any conformalized AI/ML model, and used to monitor for deviations from $\mathcal{H}_0^{\text{ex}}$.
We will describe CTMs from a game-theoretic ``testing-by-betting'' interpretation \citep{shafer2019game, shafer2021testing}.

Under $\mathcal{H}_0^{\text{ex}}$, a stochastic process $M_0, M_1, ..., M_t, ...$ is a CTM if it is a nonnegative martingale---i.e., $M_t\geq 0$ for all $t$ and $\mathbb{E}_{\mathcal{H}_0^{\text{ex}}}[M_t \mid M_0, ..., M_{t-1}] = M_{t-1}$---constructed from a corresponding sequence of conformal $p$-values $p_1, ..., p_t, ...$ via an appropriately defined ``betting process.'' 
In the game-theoretic interpretation, a bettor has initial wealth $M_0$ (usually set to $M_0=1$, for simplicity). Then, at each time $t=0, 1, ...$, she may use wealth $M_t$ to bet on the value of $p_{t+1}$ to be observed next; once $p_{t+1}$ is revealed, she receives her reward and/or pays her losses, resulting in a new total wealth $M_{t+1}$. The bettor bets \textit{against} the null hypothesis $\mathcal{H}_0^{\text{ex}}$: if $\mathcal{H}_0^{\text{ex}}$ is true, then $P_t \stackrel{iid}{\sim} \text{Unif}[0,1]$ and she cannot expect to outperform random guessing; so, if the bettor grows her wealth by a large factor $M_t/M_0>>1$, this is evidence that the bettor ``knows'' an alternative hypothesis more accurate than $\mathcal{H}_0^{\text{ex}}$
\citep{vovk2021testing}. $M_t/M_0$ can be taken as an ``anytime-valid'' evidence against $\mathcal{H}_0^{\text{ex}}$. In particular, $M_t/M_0$ is also an ``$e$-process'' for $\mathcal{H}_0^{\text{ex}}$ (see \citet{vovk2021values, ramdas2023game}).


More formally, define a \textit{betting function} as a function $h:[0, 1]\rightarrow [0, \infty]$ that integrates to one, i.e., $\int_0^1h(u)\text{d}u=1$. We follow \citet{vovk2021testing} and assume the betting function
\begin{align}
    h_{\epsilon}(p) := 1 + \epsilon (p - 0.5),
    \label{eq:betting_function}
\end{align}
where $\epsilon\in \mathbf{E}:= \{-1, 0, 1\}$ ($\mathbf{E}$ is selected somewhat arbitrarily). The intuition is that $\epsilon<0$ corresponds to betting on smaller $p$-values, $\epsilon>0$ corresponds to betting on larger $p$-values, and $\epsilon=0$ represents not betting. A CTM can be constructed by selecting a betting strategy $h_{\epsilon_t}$ at each time $t$ that may depend only on the past $p$-value observations $p_1, ..., p_{t-1}$; this paper uses the ``composite jumper martingale'' strategy described in \citet{vovk2022algorithmic}. 
Then, a conformal (test) martingale $M_t : [0, 1]^* \rightarrow [0, \infty], t=0, 1, ...$, is constructed by accumulating the ``wealth'' attained by the bets on the previous conformal $p$-values, $p_1, p_2, ..., p_{t-1}$. That is, for a single bet $\epsilon_i\in \mathbf{E}$ at each time $i$, the product $\prod_{i=1}^th_{\epsilon_i}(p_i)$ is a CTM that can be interpreted as the factor by which the bettor has multiplied her wealth by time $t$. More broadly, CTMs may accommodate more flexible betting strategies (while only depending on past $p$-values), $\nu$, that may distribute bets across $\mathbf{E}$ at each time, as
\begin{align}
    M_t := \int \Big(\prod_{i=1}^th_{\epsilon_i}(p_i)\Big)\nu\big(\text{d}(\epsilon_0, \epsilon_1, ...)\big), \ \forall \ t.
    \label{eq:ctm_wealth_def}
\end{align}
For example, in our implementations we take $\nu$ to be the composite jumper Markov chain described by \citet{vovk2022algorithmic} to determine how the ``wealth'' values are spread across the betting space $\mathbf{E}$ at each timepoint. 


A standard CTM (Eq. \eqref{eq:ctm_wealth_def}) can be used to test the exchangeability null $\mathcal{H}_0^{\text{ex}}$ (Eq. \eqref{eq:H_0_exchangeability}) online either by using $M_t/M_0$ as an anytime-valid evidence metric, or by raising an alarm when $M_t/M_0$ exceeds some user-defined threshold $c$. That is, by Ville's inequality for martingales \citep{ville1939etude}, CTMs achieve anytime-valid control over the false alarm rate (i.e., over the probability of an alarm despite $\mathcal{H}_0^{\text{ex}}$ being true),
\begin{align}
    \mathbb{P}_{\mathcal{H}_0^{\text{ex}}}\big(\exists \ t : M_t / M_0 \geq c\big) \leq 1/c,
    \label{eq:ctm_anytime_valid_guarantee}
\end{align}
which is sometimes referred to as \textit{strong validity} due to its control over \textit{ever} raising a false alarm \citep{vovk2021testing}.  


\subsection{Weighted Conformal Prediction for Adapting to Distribution Shifts}
\label{subsec:weighted_cp}

Whereas Standard CP computes valid predictive confidence sets assuming exchangeable data, 
\textit{weighted} conformal prediction (WCP) 
(e.g., \citet{tibshirani2019conformal, podkopaev2021distribution, barber2023conformal, prinster2024conformal, barber2025unifying})
generalizes standard CP to attain valid coverage (Eq. \eqref{eq:coverage_condition}) even under various distribution shifts. WCP methods are thus an approach to \textit{adapting} to distribution shift, that is by computing CP sets on a reweighted version of the empirical calibration set's scores, which in effect can modulate the size of the prediction sets to maintain coverage. 

As we will see in the next section, WCP methods are associated with weighted-conformal $p$-values, a special-case of which was introduced in \citet{jin2023model} for standard covariate shifts, based on \citet{tibshirani2019conformal}. Concurrently with our paper, \citet{barber2025unifying} also leverage weighted-conformal $p$-values to unify various theories of conformal prediction. However, there are several key differences: regarding motivation, that work discusses connections to (one-shot) hypothesis testing as a means of framing conformal prediction, while in our paper (sequential) hypothesis testing is the primary focus; the conformal $p$-values in that paper are deterministic and attain conservative validity, whereas we incorporate randomness to achieve exact validity; lastly, whereas that paper focuses on a batch-data setting where all the data are observed simultaneously, we focus on an online monitoring setting where each $p$-value is determined only by the current and past data, and we importantly present new theory for when a sequence of WCP $p$-values can be distributed uniformly and \textit{independently}, which is important for monitoring with WCTMs.

\section{Theory and Methods: Weighted-\\
Conformal $p$-Values and Test Martingales}

Our main theory and method contribution is to introduce weighted-conformal test martingales (WCTMs), constructed from a generalized version of weighted-conformal $p$-values, which expand the scope of their standard conformal analogs. 
WCTMs enable more customizable and informative alarm criteria than standard CTMs, and our specific variants enable adaptation to mild shifts, while in response to severe shifts they raise alarms and enable root-cause analysis.

\label{sec:w_conformal_p_vals}

\subsection{Generalized Weighted-Conformal $p$-Values}

In this section we present a general version of weighted-conformal $p$-values.\footnote{Here, ``weighted'' refers to weights on the score distribution prior to computing a $p$-value, not to weighting the $p$-value itself.}
Sequences of these generalized weighted-conformal $p$-values will lay a theoretical foundation for online testing of a broad range of null hypotheses beyond exchangeability.
To begin, observe that standard conformal $p$-values (Eq. \eqref{eq:unweighted_p_vals}) can be equivalently written as 
\begin{align*}
    p_{n+1} = \sum_{i=1}^{n+1}\frac{1}{n+1}\Big[\mathbbm{1}\{v_i > v_{n+1}\} + u_{n+1}\mathbbm{1}\{v_i = v_{n+1}\}\Big],
\end{align*}
where the purpose of this notation is to isolate where the comparison of $v_{n+1}$ to each $i$-th score is given uniform weight $\tfrac{1}{n+1}$. For some arbitrary weight vector $\tilde{w} = (\tilde{w}_1, ..., \tilde{w}_{n+1})\in [0,1]^{n+1}$ where $\sum_{i=1}^{n+1}\tilde{w}_i=1$, it is then straightforward to define \textit{weighted-conformal $p$-values} as 
\begin{align}
    p_{n+1}^{\tilde{w}} = \sum_{i=1}^{n+1}\tilde{w}_i\Big[\mathbbm{1}\{v_i > v_{n+1}\} + u_{n+1}\mathbbm{1}\{v_i = v_{n+1}\}\Big].
    \label{eq:weighted_p_vals_arbitrary_def}
\end{align}
Note that existing (split or full) weighted CP methods (e.g., \citet{tibshirani2019conformal, podkopaev2021distribution, prinster2024conformal} and certain methods in \citet{barber2023conformal}) can also be defined by weighted-conformal $p$-values by plugging $p_{n+1}^{\tilde{w}}(X_{n+1}, y)$ in for $p_{n+1}(X_{n+1}, y)$ in Eq. \eqref{eq:cp_set_def_via_p_val}, where 
$\tilde{w}$ is an appropriately chosen weight vector.\footnote{Conservatively valid weighted CP methods would need to further set $u_{n+1}=1$ in Eq. \eqref{eq:weighted_p_vals_arbitrary_def} to reduce appropriately.} 

To understand the meaning of the weights $\tilde{w}$ for our hypothesis testing purpose, we draw from the general view of weighted CP described in \citet{prinster2024conformal}, which expounded analysis from \citet{tibshirani2019conformal}: For setup, let 
$E_z^{(n+1)}$ denote the event of observing the \textit{bag} (i.e., multiset) of data, 
$\{Z_1, ..., Z_{n+1}\} = \{z_1, ..., z_{n+1}\}$, meaning that we observe the empirical distribution of the data, but we do \textit{not} know their order or whether $Z_{i}=z_i$, and so on. Then, the oracle weights $\tilde{w}^o$ would be given by entries 
\begin{align}\label{eq:general_weights}
    \tilde{w}_i^o = & \ \mathbb{P}\{Z_{n+1} = z_i \mid E_z^{(n+1)}\} \\
    = & \ \frac{\sum_{\sigma:\sigma(n+1)=i}f(z_{\sigma(1)}, ..., z_{\sigma(n+1)})}{\sum_{\sigma}f(z_{\sigma(1)}, ..., z_{\sigma(n+1)})},\nonumber
\end{align}
where $f$ is the joint probability density function (PDF) and $\sigma$ is a permutation of  $[n+1]$. In words, the oracle weight $\tilde{w}_i^o$ we would ideally use for $\tilde{w}_i$ is the probability that the test point's random variable $Z_{n+1}$ took on the observed value $z_i$, conditioned on the bag of data, $E_z^{(n+1)}$. For further exposition, see \citet{prinster2024conformal}. While for arbitrary $f$, computing Eq. \eqref{eq:general_weights} is intractable due to requiring knowledge of $f$ and the factorial complexity, we next examine how reducing  $\tilde{w}^o$ can be useful in practical hypothesis testing.

\subsection{Expanded Hypothesis Testing with Weighted-Conformal $p$-Values}
\label{subsec:w_cpvals_hypothesis_testing}

\textbf{Exact One-Shot Testing for Any Null Distribution} 
We now look at how to use the weighted-conformal $p$-values introduced in Eq. \eqref{eq:weighted_p_vals_arbitrary_def} for testing null hypotheses beyond exchangeability, first focused on a one-shot setting.
Let $\mathcal{H}_0(\hat{f})$ denote any conjunction of assumptions on the true $f$ or approximations of $f$ that we assume hold under the null hypothesis, 
and let $\tilde{w}:= \tilde{w}(\mathcal{H}_0(\hat{f}))$ denote the weight vector computed by reducing Eq. \eqref{eq:general_weights} with these conditions only. That is, $\mathcal{H}_0(\hat{f})$ must denote jointly sufficient conditions for the computed weights to equal the oracle ones: $\mathcal{H}_0(\hat{f})\implies\tilde{w}=\tilde{w}^o$. For example, if we assume exchangeability as our null hypothesis ($\mathcal{H}_0(\hat{f})=\mathcal{H}_0^{\text{(ex)}}$), then $f(z_{\sigma(1)}, ..., z_{\sigma(n+1)})=f(z_{1}, ..., z_{n+1})$ and Eq. \eqref{eq:general_weights} reduces to $\tilde{w}_i=\frac{1}{n+1}=\tilde{w}_i^o$ for all $i\in[n+1]$, recovering standard conformal $p$-values.
More generally, $\mathcal{H}_0(\hat{f})$ may denote other invariance assumptions on $f$ or density-ratio estimates assumed to be correct (see Sec. \ref{subsec:cov_concept_shift} and App. \ref{app:deriving_wctm_algos} for a worked example for our main setting). Then, Lemma \ref{main:lemma1} says that $P^{\tilde{w}}_{n+1}$ is a valid and exact $p$-value for $\mathcal{H}_0(\hat{f})$.
\begin{lemma}\label{main:lemma1}
    For any $n\in\mathbb{N}$ and any level $\alpha \in [0, 1]$,
    \begin{align}\label{eq:H0(hat(f))_one_shot}
        \mathbb{P}_{\mathcal{H}_0(\hat{f})}\{P_{n+1}^{\tilde{w}} \leq \alpha \mid E_{z}^{(n+1)} \} = \alpha.
    \end{align}
\end{lemma}
\textbf{A Condition for Independent Online WCP $p$-Values}
Sequential hypothesis testing with (W)CTMs, however, relies not only on the validity of individual conformal $p$-values, but moreover that a \textit{sequence} of such $p$-values each carry new information; in particular, that they are \textit{independent} under the null. 
We now state a sufficient condition for independence, which we will hereon include in $\mathcal{H}_0(\hat{f})$.\footnote{This condition was added in this paper's v5; see Remark \ref{remark:version_change}.}
\begin{definition}[Bag Sufficiency]
    A sequence of random variables, $Z_1, Z_2, ...$, is bag-sufficient if, for every $t\geq 1$,
    \begin{align}\label{eq:bag_sufficiency}
    Z_t \indep Z_1, ..., Z_{t-1} \mid E_z^{(t-1)}.
    \end{align}
\end{definition}
That is, bag sufficiency states that every $Z_t$ is independent of the \textit{ordering} of past observations $Z_1, ..., Z_{t-1}$, conditional on the \textit{bag of unordered past values} (i.e., multiset or empirical distribution), $E_z^{(t-1)}$. This condition thus accommodates any sequential distribution shift where the shift may depend on past observations but not their ordering. 
We now state a result for independent online weighted conformal $p$-values.
\begin{theorem}[Independence and exact validity, for $\mathcal{H}_0(\hat{f})$, of online WCP $p$-values]
\textit{For any $T\in \mathbb{N}$, let $\mathcal{H}_0(\hat{f})$ denote the conjunction of bag sufficiency (Eq. \eqref{eq:bag_sufficiency}) and any assumptions on the joint PDF, $f$, such that $\mathcal{H}_0(\hat{f})$ is sufficient for the oracle weights (Eq. \eqref{eq:general_weights}) to reduce exactly to those computed: $\mathcal{H}_0(\hat{f})\implies \tilde{w}^o=\tilde{w}(\mathcal{H}_0(\hat{f}))$.
For user-defined significance levels $\alpha_1, ..., \alpha_T\in(0, 1)^T$, 
if $\mathcal{H}_0(\hat{f})$ holds, then $P^{\tilde{w}}_1, P^{\tilde{w}}_2, ...$ are IID uniform on [0,1]: 
}
\begin{align}\label{eq:H0(hat(f))_iid_bernoulli}
    \mathbb{P}_{\mathcal{H}_0(\hat{f})}\big\{P_1^{\tilde{w}} \leq \alpha_1, ..., P_T^{\tilde{w}} \leq \alpha_T \big\}= \alpha_1 \cdots \alpha_T.
\end{align}
    \label{thm:general_weighted_p_validity}
\end{theorem}
\vspace{-0.5cm}
\begin{remark}[Connection to Online Compression Models]\label{remark:ocm}
    Sequences of online weighted-conformal $p$-values are closely related to \citet{vovk2003well}'s online compression models (OCMs): the oracle weights $\tilde{w}^o$ (Eq. \eqref{eq:general_weights}) coincide with \citet{vovk2003well}'s ``backward kernels,'' with the bag $E_z^{(t)}$ as an OCM summarizing statistic. Appendix ~\ref{sec:related_work} relates our notation to that of OCMs and discusses the relation further.
\end{remark}
\begin{remark}[Change from earlier versions]\label{remark:version_change}
Versions v1--v4 of this paper stated Theorem~\ref{thm:general_weighted_p_validity} without bag sufficiency. That original statement is incorrect: Proposition~\ref{prop:counterexample} presents a three-observation counterexample in which each $p$-value is exactly marginally uniform yet $P_3^{\tilde{w}} \not\indep P_2^{\tilde{w}}$. Some restriction is therefore necessary. Bag sufficiency holds trivially under the null $H_0^{(\text{cs})}$ of Section~\ref{subsec:cov_concept_shift}, so the practical claims and experiments reported here are unaffected.
\end{remark}
\textbf{Proof Sketch:} We defer a full proof to Appendix \ref{sec:proof}. The proof builds on those for standard CTMs in \citet{vovk2002line} and \citet{vovk2003testing}, which leverage the idea of ``reversing time,'' while drawing on ideas from \citet{tibshirani2019conformal} and \citet{prinster2024conformal}. We imagine that the sequence of data observations $(z_1, ..., z_T)$ is generated in two steps: first, the unordered bag of data observations, $\{z_1, ..., z_T\}$, is generated from some probability distribution; then---here generalizing \citet{vovk2002line} and \citet{vovk2003testing} by weighting permutations according to their likelihood---from all possible permutations $\sigma$ of the values $\{z_1, ..., z_T\}$, each possible sequence $(z_{\sigma(1)}, ..., z_{\sigma(T)})$ is chosen with probability proportional to $f(z_{\sigma(1)}, ..., z_{\sigma(T)})$, where $f:=f_Z$ is the probability-density function\footnote{More generally, the Radon–Nikodym derivative.} for the distribution $F_Z$. Roughly (ignoring borderline effects), the second step ensures the following. Conditionally on knowing $\{z_1, ..., z_T\}$ (and thus unconditionally), $P_{T}^{\tilde{w}^o}$ has a standard uniform distribution. When $Z_T=z_{\sigma(T)}$ is observed, this settles the value of $P_{T}^{\tilde{w}^o}=p_{\sigma(T)}^{\tilde{w}^o}$. Then, bag sufficiency ensures that knowing $\{z_1, ..., z_T\}$ and $Z_T=z_{\sigma(T)}$ is equivalent, as far as $P_{T-1}^{\tilde{w}^o}, ..., P_{1}^{\tilde{w}^o}$ are concerned, to knowing $\{z_1, ..., z_{T-1}\}$ (after relabeling indices), so we can conclude $P_{T-1}^{\tilde{w}^o}$ also has a standard uniform distribution, and so on.

\subsection{Main Practical Testing Objective: Testing for \{Concept Shift or Unanticipated Covariate Shift\}}
\label{subsec:cov_concept_shift}

We now provide a worked example for how weighted-conformal $p$-values can be used to test a common assumption in the robust ML literature: namely, the \textit{covariate shift} \citep{shimodaira2000improving, sugiyama2007covariate} assumption, where the marginal input distribution $F_X$ may shift to some other distribution $G_X$ at test time, but the conditional label distribution $F_{Y \mid X}$ is assumed to remain invariant. 
In the standard covariate shift setting considered in \citet{tibshirani2019conformal}, the oracle weights $\tilde{w}_i^o$ in Eq. \eqref{eq:general_weights} are proportional to density-ratio weights $g_X(X_i)/f_X(X_i)$. A common goal is to adapt to covariate shift by learning an approximate density-ratio function $\widehat{w}(x)\approx g_X(x)/f_X(x)$ for reweighting the data (e.g., \citet{sugiyama2007covariate, tibshirani2019conformal, yang2024doubly, zhang2024adapting}). 

However, whereas robustness results often rely on the density-ratio estimator being reasonably accurate, i.e., $\widehat{w}(x)\approx g_X(x)/f_X(x)$, there is little work on \textit{testing} for whether this reliability criterion is achieved in practice.
Let $\mathbf{G_{X}^{(\widehat{w})}}$ denote the set of distributions for $X$ with density-ratio function $\widehat{w}(x)$ with respect to the true $F_X\in \mathbf{F_X}$. 
Then, we aim to test the following joint 
null hypothesis:
\begin{align}\label{eq:cs_null}
    \mathbf{\mathcal{H}_0^{\text{(cs)}}}  :=\begin{cases}
        (X_{i}, Y_{i}) \qquad \ \stackrel{iid}{\sim} \mathbf{F_{Y\mid X}} \times \mathbf{F_{X}} , \  i \in [n + t_{\text{ad}}-1] \\
        (X_{n+t}, Y_{n+t}) \stackrel{iid}{\sim} \mathbf{F_{Y\mid X}} \times \mathbf{G_{X}^{(\widehat{w})}}, \ t\geq t_{\text{ad}},
    \end{cases}
\end{align}
which can be roughly read as assuming that $Y\mid X$ is invariant and that $\widehat{w}(x)=g_X(x)/f_X(x)$. 
Thus, observing an extreme value of $p^{\tilde{w}}_{n+1}$ would convey evidence against $\mathcal{H}_0^{\text{(cs)}}$, meaning that there has been a concept shift in $Y\mid X$ or that the density-ratio adaptation $\widehat{w}(x)$ is inaccurate. 

\subsection{Weighted-Conformal Test Martingales for Continual Monitoring}
\label{subsec:wctms}

With Theorem \ref{thm:general_weighted_p_validity} at hand, weighted-conformal test martingales (WCTMs) can be constructed from a sequence of weighted-conformal $p$-values to enable continual, anytime-valid monitoring of a customizable null hypothesis $\mathcal{H}_0(\hat{f})$. 
Just as in the construction of standard CTMs (Section \ref{subsec:standard_CTM}), WCTMs require a betting function $h$ (we assume the betting function in Eq. \eqref{eq:betting_function}) and a strategy $\nu$ for placing bets $\epsilon_t$ at each time $t$ that can only depend on past $p$-value observations (we use the composite jumper strategy from \citet{vovk2022algorithmic}). Then, a WCTM is constructed by feeding a sequence of weighted-conformal $p$-values $p_1^{\tilde{w}}, ..., p_t^{\tilde{w}}, ...$ into Eq. \eqref{eq:ctm_wealth_def} (i.e., by setting $p_i:=p_i^{\tilde{w}}$ in Eq. \eqref{eq:ctm_wealth_def}).
Due to the $P_t^{\tilde{w}}$ being IID uniform on [0,1] under $\mathcal{H}_0(\hat{f})$ (Theorem \ref{thm:general_weighted_p_validity}), the stochastic process $\tilde{M}_0, \tilde{M}_1, ..., \tilde{M}_t$ is a nonnegative test martingale for $\mathcal{H}_0(\hat{f})$, and by Ville's inequality \citep{ville1939etude} it achieves the anytime-valid control over false alarms in the following proposition (proof in Appendix \ref{sec:proof}).
\begin{proposition}(WCTM anytime-valid false-alarm control)
\textit{Let $\tilde{M}_0, \tilde{M}_1, ..., \tilde{M}_t, ...$ be a WCTM constructed from the sequence of weighted-conformal $p$-values $p_1^{\tilde{w}}, p_2^{\tilde{w}}, ..., p_t^{\tilde{w}}, ...$. Then, assuming $\mathcal{H}_0(\hat{f})$, }
\begin{align}\label{eq:wctm_ville}
    \mathbb{P}_{\mathcal{H}_0(\hat{f})}\big(\exists \ t : \tilde{M}_t / \tilde{M}_0 \geq c \big) \leq 1/c.
\end{align}
\label{thm:wctm_anytime_valid_guarantee}
\end{proposition}
\vspace{-0.5cm}

\subsection{Scheduled or Multistage WCTM-Based Monitoring}
The anytime-valid guarantee in Eq. \eqref{eq:wctm_ville} controls the probability of raising a false alarm at \textit{any} time in an infinite sequence of observations, but this can be overly strong, especially when one has a set time horizon in mind for monitoring. For such ``scheduled'' or ``multistage'' monitoring, we improve the efficiency (speed in detecting true shifts) by following \citet{vovk2021retrain, vovk2021testing} and augmenting WCTMs with standard changepoint detection metrics such as CUSUM \citep{page1954continuous} and Shiryaev-Roberts \citep{roberts1966comparison, shiryaev1963optimum} while achieving a weaker form of validity. Consider the Shiryaev-Roberts procedure applied to WCTMs $\tilde{M}_t$, whose $k$-th stage alarm time is
\begin{align}
    \tau_k := \min\big\{t > \tau_{k-1} : \sum_{i=\
    \tau_{k-1}}^{t-1}\frac{\tilde{M}_t}{\tilde{M}_i}\geq c\big\}, \ k \in \mathbb{N}.
    \label{eq:SR_procedure}
\end{align}
This Shiryaev-Roberts procedure based on WCTMs controls the average run length (ARL) under $\mathcal{H}_0(\hat{f})$, where the procedure is expected to be reset after $c$ timesteps. 
\begin{proposition}(WCTM-based Shiryaev-Roberts ARL control) Let $\tau_1$ denote the Shiryaev-Roberts stopping time (Eq. \eqref{eq:SR_procedure}) based on a WCTM
$\tilde{M}_0, \tilde{M}_1, ..., \tilde{M}_t, ...$ constructed for the null hypothesis $\mathcal{H}_0(\hat{f})$. Then, assuming $\mathcal{H}_0(\hat{f})$, 
    \begin{align}
    \mathbb{E}_{\mathcal{H}_0(\hat{f})}\big[\tau_1\big]\geq c.
    \end{align}
    \label{thm:SR-WCTM_ARL_control}
\end{proposition}
\vspace{-0.8cm}

\begin{figure*}[!htb]
\centering
    \scriptsize{\textbf{Medical Expenditure Panel Survey (MEPS)}}
    \\
    \begin{subfigure}{}
        \includegraphics[width=0.18\textwidth]{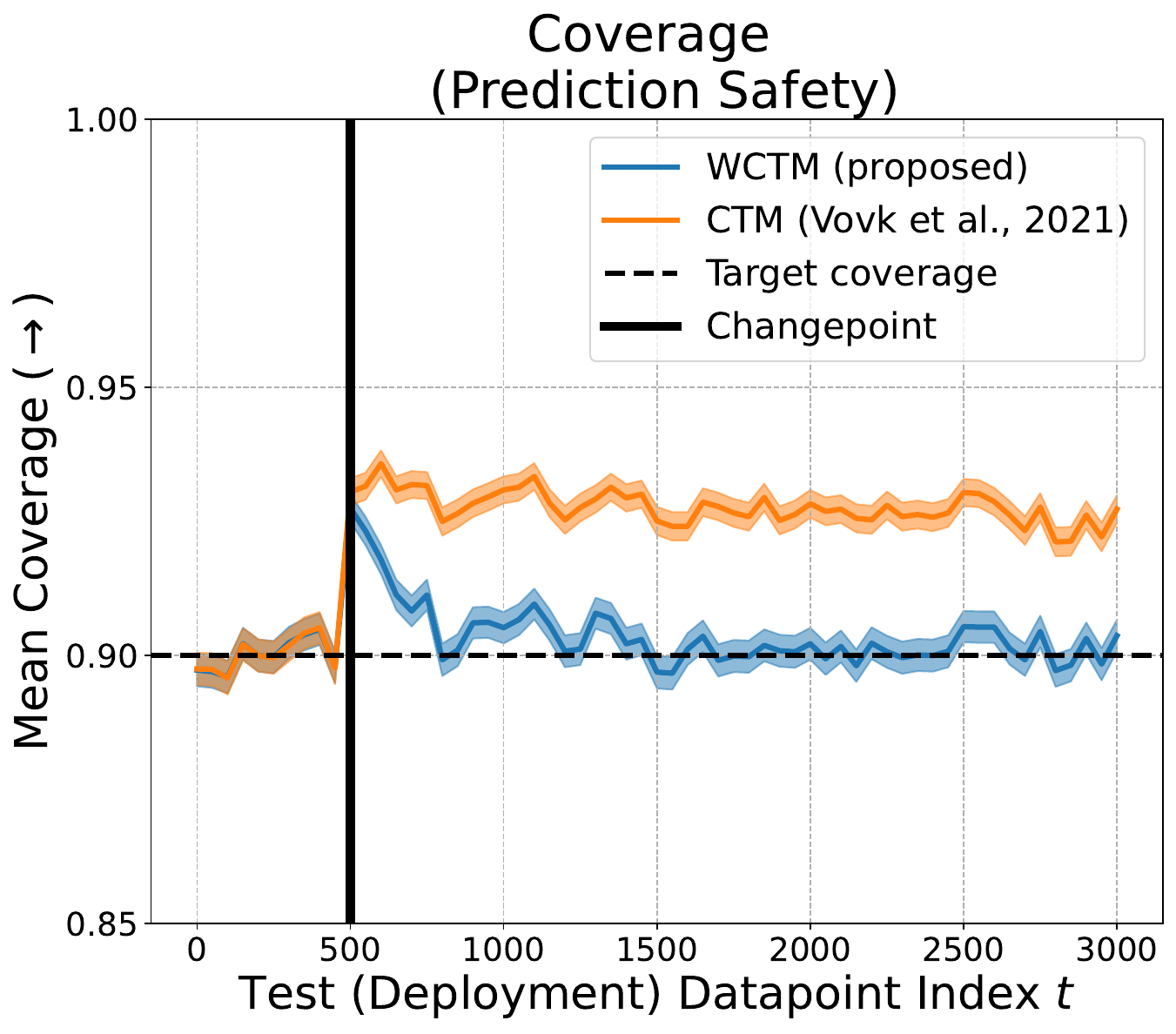}
    \end{subfigure}
    \hfill
    \begin{subfigure}{}
        \includegraphics[width=0.18\textwidth]{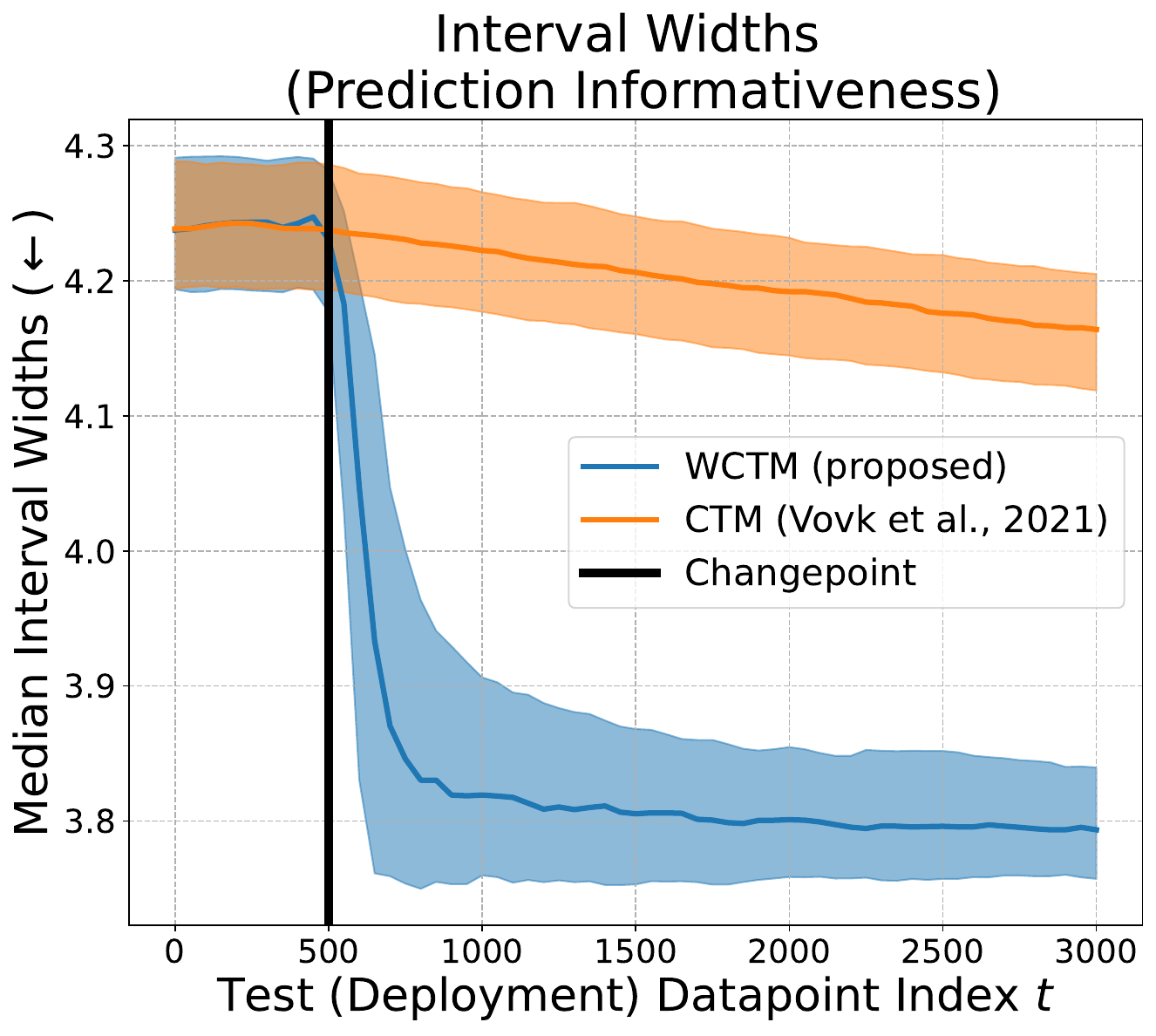}
    \end{subfigure}
    \hfill
    \begin{subfigure}{}
        \includegraphics[width=0.18\textwidth]{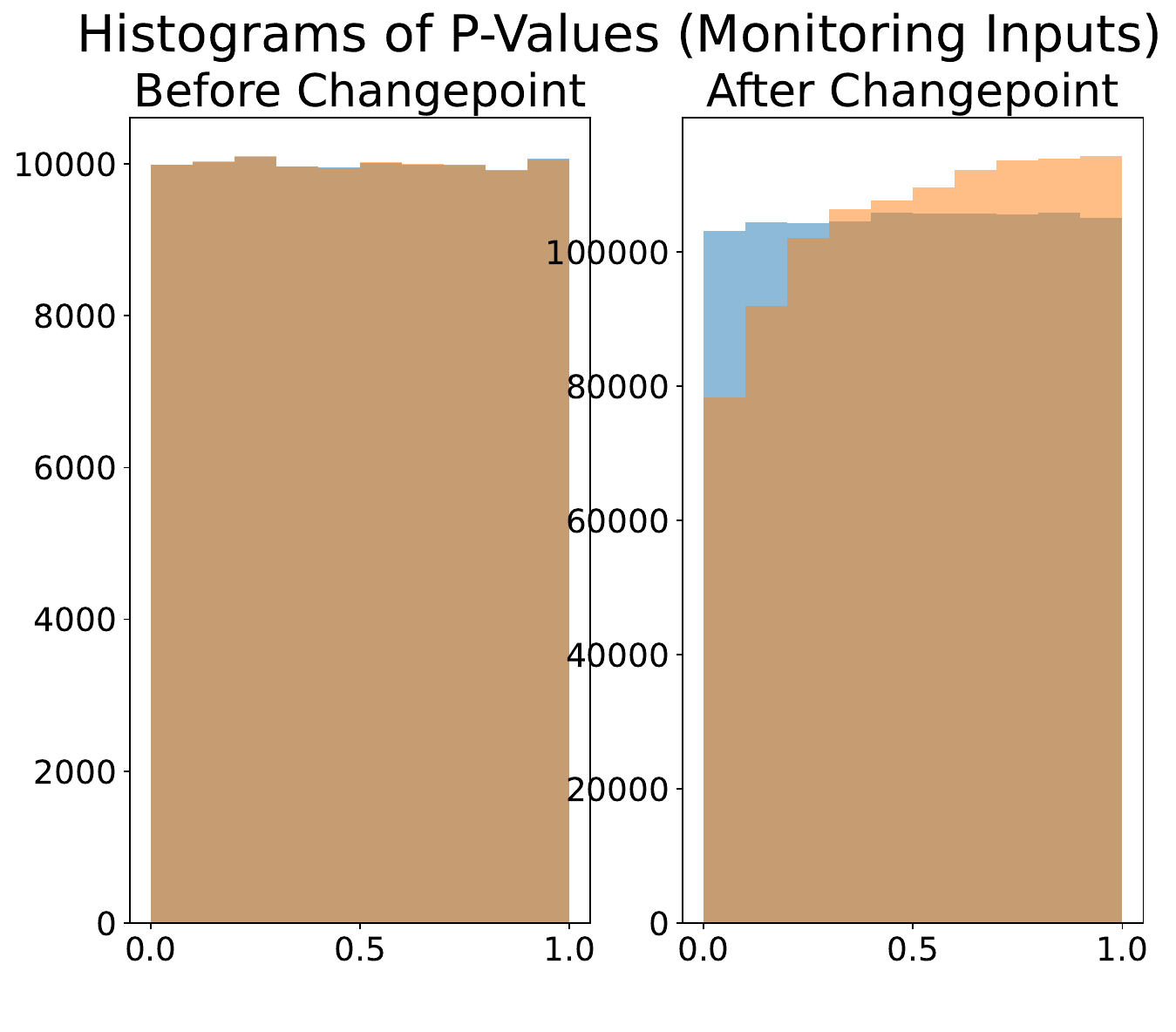}
    \end{subfigure}
    \hfill
    \begin{subfigure}{}
        \includegraphics[width=0.18\textwidth]{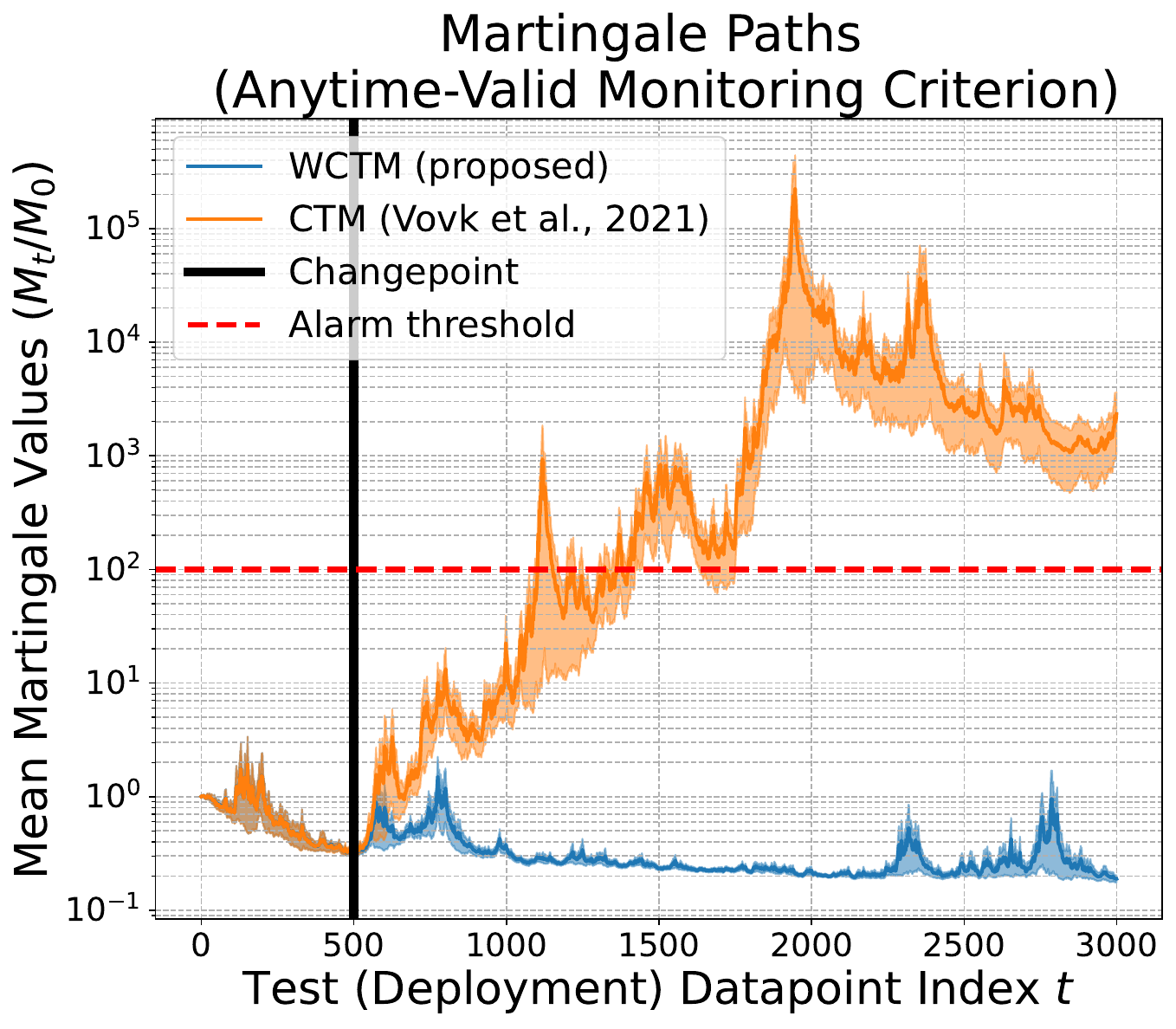}
    \end{subfigure}
    \hfill
    \begin{subfigure}{}
        \includegraphics[width=0.18\textwidth]{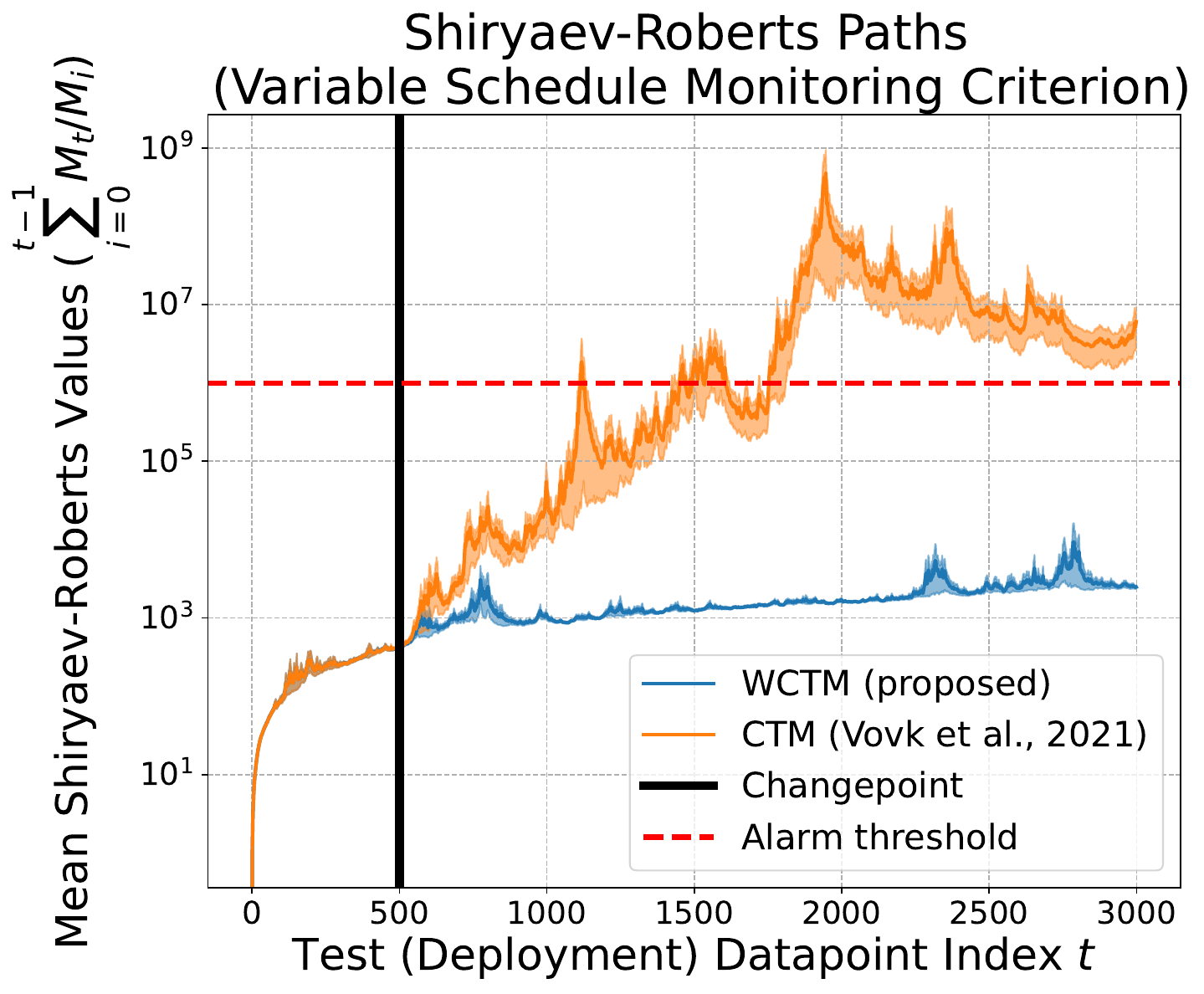}
    \end{subfigure}
    \\
    \vspace{-1em}
\caption{\small Results on the Medical Expenditure Panel Survey (MEPS) dataset for evaluating adaptation to benign covariate shift; all values are averaged over 200 random seeds. Training and calibration sets were sampled uniformly at random (with 1/3 of the total data used for training and calibration each), while post-changepoint test-set datapoints were bias-sampled from the remaining holdout data with probability proportional to $\exp(\lambda \cdot h(x))$. The shift-magnitude scalar was set to $\lambda=5.0$ and  
the function $h$ was selected to simulate realistic shifts as described in Appendix \ref{app:simulating_shifts}. Error regions represent standard errors for coverage, martingale paths, and Shiryaev-Roberts paths, and interquartile range for interval widths. Whereas standard CTMs (orange) raise unnecessary alarms with their anytime-valid and scheduled monitoring criteria, WCTMs avoid doing so by adapting online to the shift. That is, WCTMs (blue) maintain target coverage, adapt by increasing interval sharpness, and avoid unneeded alarms. See Appendix \ref{subsec:additional_adaptation_expts} for additional datasets.}
\label{fig:unnecessary_alarms}
\end{figure*}

\subsection{WCTM Implementation in WATCH: Dynamic Online Adaptation and Root-Cause Analysis}
\label{subsec:wctm_practical_implementation}

All our implementations and experiments in this paper focus on WCTMs that continuously adapt to mild covariate shifts (e.g., Fig \ref{fig1:synthetic_examples}a) while raising alarms in response to either extreme covariate shifts (e.g., Fig \ref{fig1:synthetic_examples}b) or concept shifts in $Y|X$ (e.g., Fig \ref{fig1:synthetic_examples}c). These goals correspond to monitoring the $\mathcal{H}_0^{(\text{cs})}$ null hypothesis given in Eq. \eqref{eq:cs_null}. Our adaptation procedure performs online density ratio estimation with an online probabilistic classifier (e.g., a multilayer perceptron or logistic regression model similar to in \citet{zhang2024adapting}) to update the estimation of $\widehat{w}^{(t)}(x)$ at each timestep $t$.

\textbf{Online Adaptation with Dynamic Initialization}~ For the initial adaptation stage, a key question is \textit{when} to begin the adaptation procedure, or when to begin estimating $\widehat{w}^{(t)}(x)$. Our methods make this determination automatically and dynamically by running a secondary standard CTM that is restricted to only monitoring for changepoints in the marginal $X$ distribution via a nearest-neighbor nonconformity score \citep{vovk2021retrain}. At deployment, the main WCTM method is initially a standard CTM (with uniform weights); then, once the secondary  ``$X$-CTM'' method exceeds a pre-determined adaptation threshold (see gray paths in Fig \ref{fig1:synthetic_examples}) at some time $t_{\text{ad}}$, this triggers the adaptation of the main WCTM monitoring method (blue paths in Fig \ref{fig1:synthetic_examples}) and corresponding weighted CP intervals. For efficiency purposes, after the adaptation time $t_{\text{ad}}$, the CP calibration set $[n+t_{\text{ad}}]$ is treated as fixed (no longer adding test points to it online), meaning the weighted-conformal $p$-values are computed summing only over $[n+t_{\text{ad}}-1]\cup \{n+t\}$,
\begin{align}
    p_{n+t}^{\tilde{w}} = \sum_{i\in[n+t_{\text{ad}}-1]\cup \{n+t\}}\tilde{w}_i^{(t)}\Bigg[\begin{matrix}\mathbbm{1}\{v_i > v_{n+t}\} \quad \quad \\
     + u_{n+t}\mathbbm{1}\{v_i = v_{n+t}\}\end{matrix}\Bigg],
    \label{eq:weighted_p_vals_fixed_cal}
\end{align}
\vspace{-0.3cm}
where here $\tilde{w}_i^{(t)}\propto \widehat{w}^{(t)}(X_i)$. See Appendix \ref{app:deriving_wctm_algos} for details.

\textbf{Root-Cause Analysis with Parallelized WCTM and $X$-CTM}~ The parallel implementation of both the primary WCTM and the secondary $X$-CTM method furthermore enables root-cause analysis, that is to diagnose whether performance degradation was due to a harmful covariate shift (e.g., Fig \ref{fig1:synthetic_examples}b) or a fundamental concept shift (e.g., Fig \ref{fig1:synthetic_examples}c). That is, if both the WCTM and the $X$-CTM have detected changepoints, then a harmful shift can be diagnosed as an extreme covariate shift (e.g., Fig \ref{fig1:synthetic_examples}b); on the other hand, if the $X$-CTM does not detect a changepoint, this suggests that no covariate shift is present, and the WCTM alarm was instead
due to a concept shift in $Y|X$ (Figure \ref{fig1:synthetic_examples}). 

\begin{figure*}[!htb]
\centering
    \begin{subfigure}{}
        \includegraphics[width=.23\textwidth]{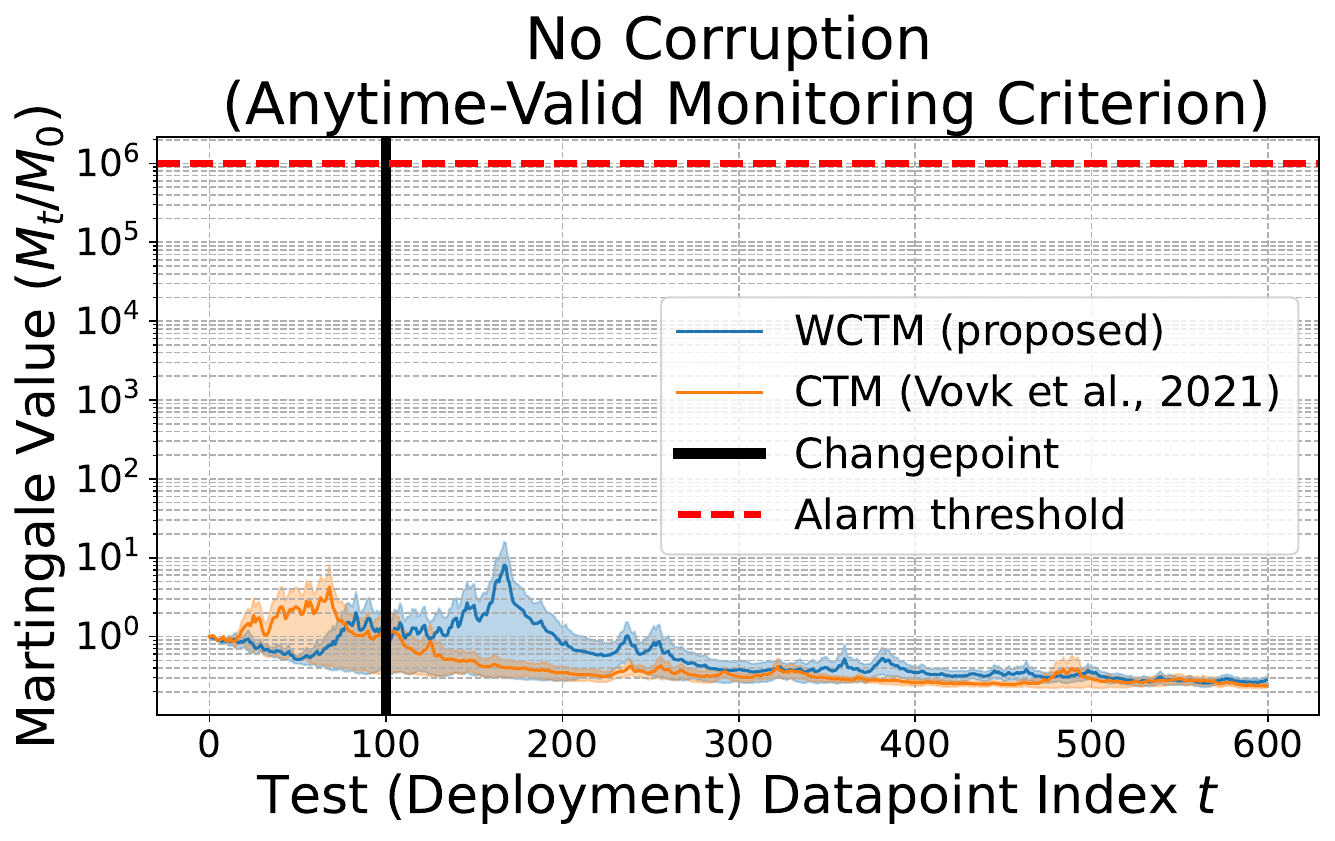}
    \end{subfigure}
    \hfill
    \begin{subfigure}{}
        \includegraphics[width=.23\textwidth]{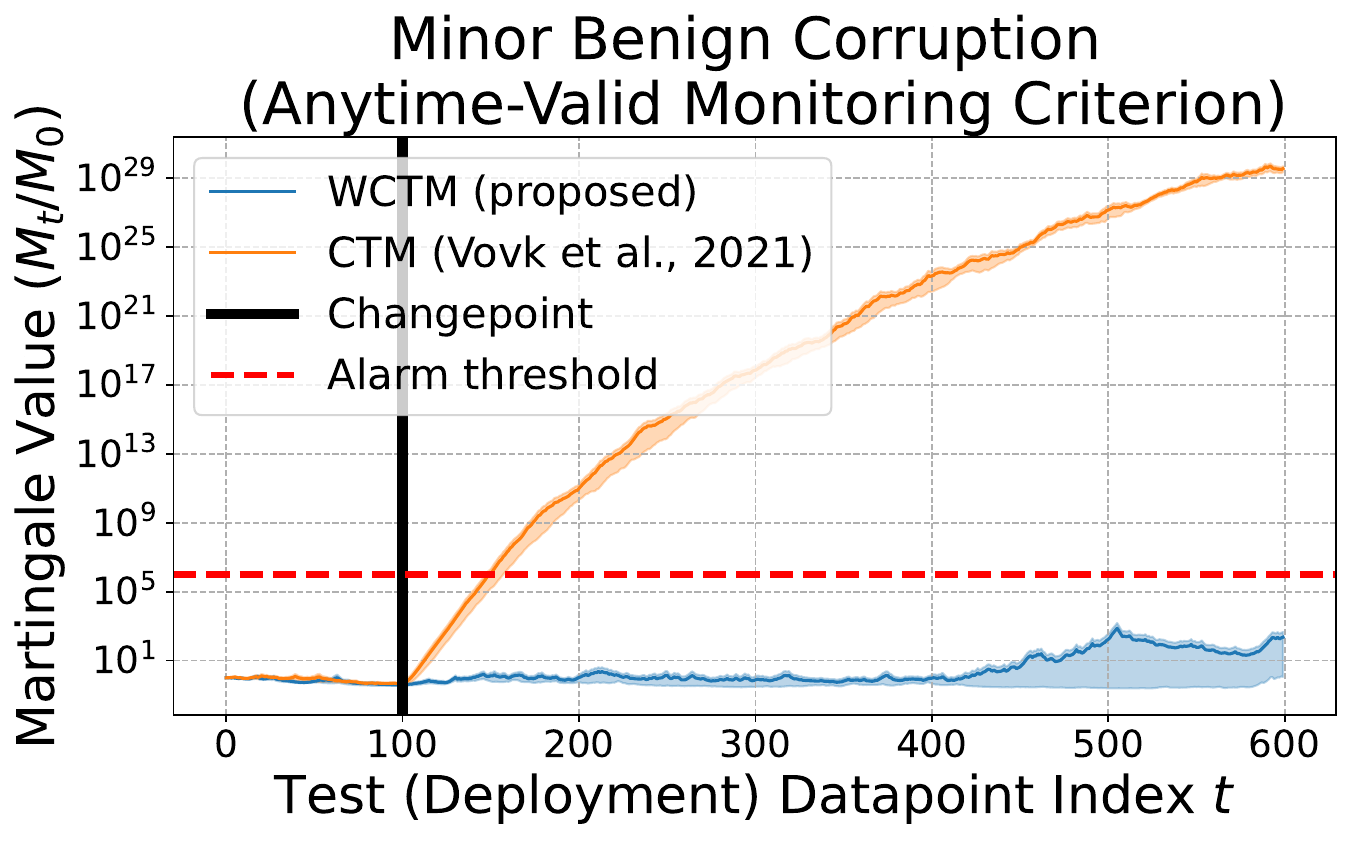}
    \end{subfigure}
    \hfill
    \begin{subfigure}{}
        \includegraphics[width=.23\textwidth]{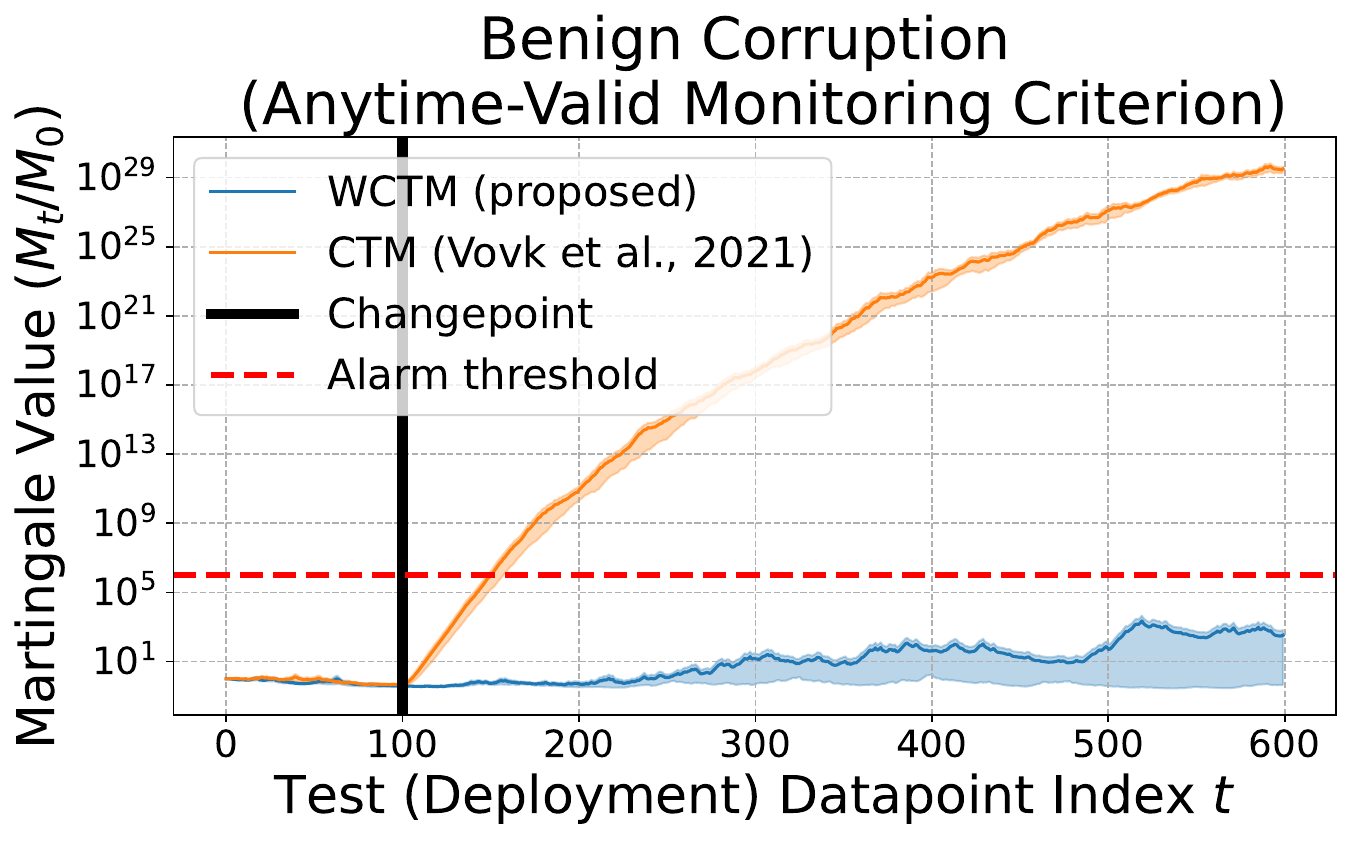}
    \end{subfigure}
    \hfill
    \begin{subfigure}{}
        \includegraphics[width=.23\textwidth]{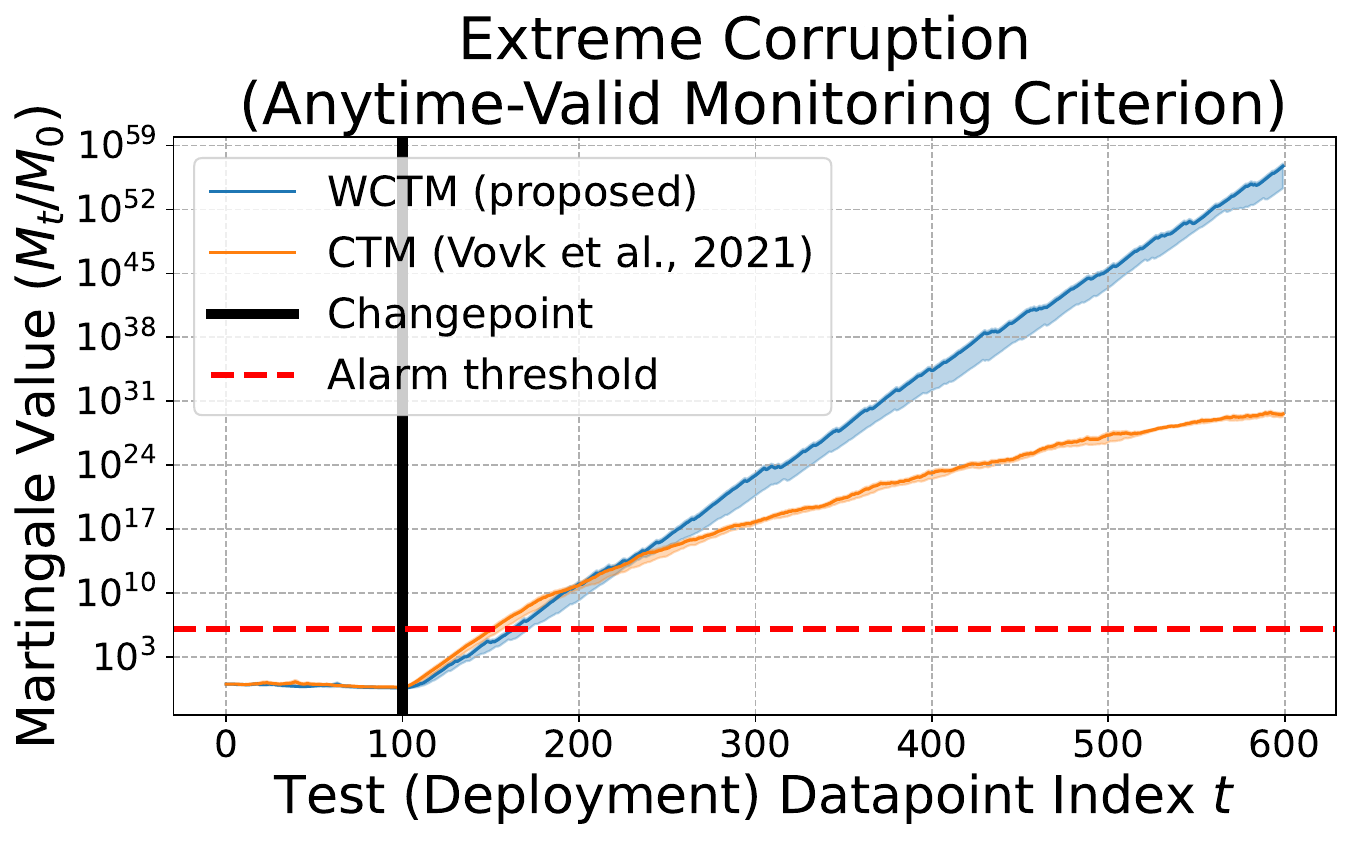}
    \end{subfigure}
    \\
    \vspace{-1em}
\caption{Example martingale trajectories of WCTM and CTM on CIFAR-10-C with increasing corruption levels. WCTMs adapt to milder shifts to avoid unneeded alarms while still detecting severe shifts. Averaged over 5 random experiments.}
\label{fig:image_benign}
\end{figure*}

\begin{figure*}[t!]

\centering
\renewcommand\arraystretch{0.8}
\captionof{table}{\small Average detection delay (ADD) computed over 200 random seeds for each monitoring method and dataset (error bars are standard errors). Anytime-valid monitoring methods are WCTMs (proposed), CTMs \citep{vovk2021testing}, and sequential testing from \citet{podkopaev2021tracking} (PR-ST); multistage monitoring methods are Shiryaev-Roberts (SR) procedure applied to WCTMs (proposed), SR applied to CTMs \citep{vovk2021testing}, changepoint detection on the miscoverage risk from \citet{podkopaev2021tracking} in both online (PR-CD-online) and minibatched (PR-CD-50) variants. Results corresponding to our method are in {\color{NavyBlue}blue}. The best ADD are \textbf{bolded}.}
\scalebox{0.9}{
\begin{tabular}{ll |cc |ccc }
\toprule
 & \textit{Monitoring Criterion} ($\rightarrow$)& \multicolumn{2}{c}{\textit{Anytime-Valid  Criterion}}   & \multicolumn{3}{c}{\textit{Scheduled, Multistage Criterion}} \\
\cmidrule(lr){3-4} \cmidrule(lr){5-7} 
 & \textit{Monitoring Method} ($\rightarrow$) & (W)CTM  & PR-ST & SR via (W)CTM & PR-CD-online & PR-CD-50 \\
\cmidrule(lr){3-4} \cmidrule(lr){5-7} \textit{Dataset} ($\downarrow$)
 & \textit{CP Type} ($\downarrow$) & \textbf{ADD ($\pm$ SE)} &   \textbf{ADD ($\pm$ SE)} &   \textbf{ADD ($\pm$ SE)} &   \textbf{ADD ($\pm$ SE)} &   \textbf{ADD ($\pm$ SE)}  \\
\cmidrule(lr){1-7} 
\multirow{2}{*}{\textbf{MEPS}} 
  & Weighted & \color{NavyBlue} \textbf{115.8 ($\pm 3.8$)} &  536.6 ($\pm 23.9$) &   \color{NavyBlue}\textbf{102.1 ($\pm 3.2$)} &  111.3 $(\pm 3.8)$  &   196.5 $(\pm 5.4)$  \\
  & Standard  & \textbf{114.8 ($\pm 3.7$)} &  475.2 ($\pm 18.8$)  &   \textbf{101.8 ($\pm 3.2$)} & 111.0 ($\pm 3.9$) &  194.5 ($\pm 5.3$)   \\
\cmidrule(lr){1-7} 
\multirow{2}{*}{\textbf{Bike Sharing}} 
  & Weighted & \color{NavyBlue} \textbf{190.0 ($\pm 8.6$)} &  602.5 $(\pm 24.7)$  &  \color{NavyBlue}\textbf{155.6 $(\pm 6.6)$} &  \textbf{154.3} $(\pm 5.9)$ &  241.2 $(\pm 6.5)$ \\
  & Standard  & \textbf{192.4 $(\pm 8.9)$} &  607.7 $(\pm 24.7)$  &  \textbf{157.4 $(\pm 6.6)$} & \textbf{157.4 $(\pm 5.9)$} & 244.3 $(\pm 6.6)$  \\
\cmidrule(lr){1-7}
\multirow{2}{*}{\textbf{Superconduct}} 
  & Weighted & \color{NavyBlue} \textbf{173.2 $(\pm 6.2)$}  & 802.9 $(\pm 30.5)$ & \color{NavyBlue}\textbf{142.8 $(\pm 5.0)$} & 192.9 $(\pm 6.0)$ & 289.1 $(\pm 6.1)$ \\
  & Standard  & \textbf{169.7 $(\pm 5.8)$} & 813.1 $(\pm 	29.9)$ &  \textbf{142.3 $(\pm 5.0)$} & 187.6 $(\pm 5.8)$  & 287.9 $(\pm 6.1)$   \\
\bottomrule
\label{tab:results}
\end{tabular}}
\vspace{-1em}
\end{figure*}

\section{Experiments}\label{sec:exp}

We conduct a comprehensive empirical analysis of the WATCH framework on real-world datasets with various distribution shifts. Our results show that WATCH adapts effectively to benign shifts (Section \ref{subsec:unnecessary_alarm}) and triggers alarms when it fails to adapt (Section \ref{subsec:benign_to_extreme_cs}), while also quickly detecting harmful shifts with little delay (Section \ref{subsec:harmful}). Details on the datasets, models, and additional results, can be found in Appendix \ref{app:expt_details}. Code to reproduce all experiments is available at the following repository: \url{https://github.com/aaronhan223/watch}. 

\textbf{Baselines}~ We compare the proposed WCTM methods that constitute WATCH against standard CTMs \citep{vovk2021retrain} in all experiments. To evaluate detection speed on true harmful shifts (sec \ref{subsec:harmful}), we also compare against methods for directly performing sequential hypothesis testing and changepoint detection on the set-prediction miscoverage risk, as proposed by \citet{podkopaev2021tracking}.

In all experiments and for all baselines, the underlying ML predictor being monitored was a neural network. On the tabular data, we used the scikit-learn \citep{pedregosa2011scikit} MLPRegressor (with L-BFGS solver and logistic activation); for the image data, we used a 3-layer MLP with ReLU activations on the MNIST datasets and a ResNet-32 \citep{he2016deep} on CIFAR-10 datasets.
For weight estimation, we use a 3-layer MLP with ReLU activations to distinguish between source and target distributions.

\subsection{WCTMs Adapt to Mild and Benign Shifts to Avoid Unnecessary Alarms} \label{subsec:unnecessary_alarm}



\textbf{WCTMs Adapt to Benign Covariate Shifts}~ Figure \ref{fig:unnecessary_alarms} compares the average performance of WCTMs to standard CTMs on a benign shift setting with the Medical Expenditure Panel Survey (MEPS) dataset (see Appendix \ref{subsec:additional_adaptation_expts} for similar results on additional datasets). Results are across 3000 sequentially-observed test datapoints, where the true changepoint shift is induced after the 500th test point. From the conformal coverage plots in the first column, it is clear that the shift is benign in that coverage for the corresponding (standard or weighted) CP methods---a metric of prediction safety---does not degrade below the target level (0.9) after the changepoint. In fact, for the baseline method (orange), coverage \textit{increases}, which could be considered a ``beneficial'' shift; nonetheless, the standard CTM raises unnecessary alarms, for both its anytime-valid (fourth column) and scheduled (fifth column) monitoring criteria. In contrast, the proposed WCTM (blue) avoids these unnecessary alarms for both monitoring criteria. That is, the WCTM \textit{adapts} to the benign covariate shift by decreasing its interval widths (second column)---indicating more informative predictions---while maintaining target coverage. The relatively uniform distribution of the postchangepoint weighted $p$-values (third column, blue) is empirical evidence validating Theorem \ref{thm:general_weighted_p_validity}; meanwhile, the martingale (fourth column) and Shiryaev-Roberts (fifth column) WCTM paths avoiding alarms supports Propositions \ref{thm:wctm_anytime_valid_guarantee} and \ref{thm:SR-WCTM_ARL_control} respectively. 



\subsection{From Mild to Extreme Covariate Shifts: WCTMs Raise Alarms if Unable to Adapt}
\label{subsec:benign_to_extreme_cs}

We demonstrate such a property  WCTM using image corruption experiments. 
Shift severity was controlled by increasing the proportion of corrupted samples in the test set relative to training and calibration.
Figure \ref{fig:image_benign} illustrates the behavior of WCTM in comparison to CTM on CIFAR-10 under varying levels of brightness corruption. When the target distribution is completely clean (left), neither method triggers a false alarm. We then introduce level-1 brightness corruption while retaining over $50\%$ clean samples (center left), creating a very mild, benign shift. In this scenario, CTM quickly raises an unnecessary alarm, whereas WCTM successfully avoids it. Next, we increase the corruption to level-3 and reduce the clean-sample ratio to $40\%$ (center right). Although WCTM reacts to this shift, it still does not trigger an alarm. Finally, at corruption level-5, with the clean-sample ratio reduced to $30\%$ (right), WCTM does raise the alarm. See Appendix \ref{app:more_image_expts} for further image data experiments and \ref{subsec:mild_mod_ex} for further analysis of what characterizes shift severity.


\begin{figure*}[htbp!]
\centering
    \includegraphics[width=0.9\textwidth]{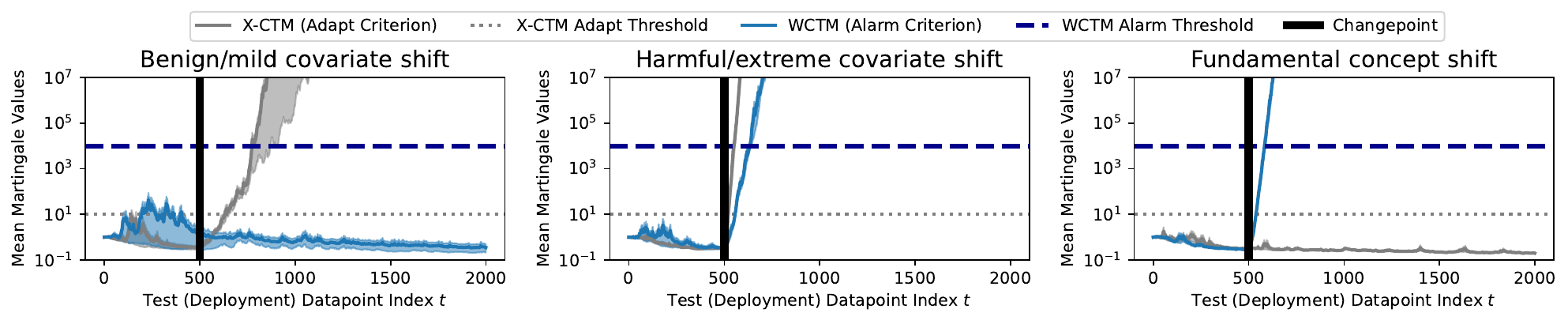}
\caption{Results for root-cause analysis with a WCTM (blue) and a secondary $X$-CTM (gray) on the MEPS dataset, averaged over 200 random seeds. The WCTM and $X$-CTM together are able to diagnose different types of shifts.}
\label{fig4:MEPS_RCA}
\vspace{-1em}
\end{figure*}

\subsection{WCTMs Detect Harmful Concept Shifts Faster than Sequentially Tracking Loss Metrics}
\label{subsec:harmful}

Because concept shifts fundamentally change the $Y|X$ relationship, they are generally harmful, and monitoring methods should detect them as quickly as possible. Our experiments compare the average detection delay (ADD) of WCTM methods relative to comparable CTM methods, as well as against methods for sequentially tracking the loss metrics directly with comparable false-alarm control. That is, the sequential testing procedures proposed in \citet{podkopaev2021tracking} (PR-ST) have anytime-valid control comparable to (W)CTMs, while the changepoint detection procedure in \citet{podkopaev2021tracking} (PR-CD) has average-run-length control comparable to running the Shiryaev-Roberts procedure \eqref{eq:SR_procedure} on top of (W)CTMs. Both PR-ST and PR-CD methods can be used to monitor the miscoverage risks of either weighted CP or standard CP methods, though they are less data efficient (requiring an extra holdout set for computing concentration inequalities). We compared to the betting-based 
$e$-process in \citet{podkopaev2021tracking} (with $\epsilon_{\text{tol}}=0$) because those methods were reported to perform the best and are more comparable to  (W)CTMs, which are also betting-based. We compared WCTMs to sequential testing variant with standardized anytime-false alarm rate of 0.01; for SR-WCTMs we compared vs changepoint detection variant with common average-run-length (under the null) of 20,000.

Table \ref{tab:results} reports the average detection delay (ADD) results for monitoring methods grouped by false-alarm control type, with proposed methods in blue. Among the anytime-valid monitoring methods, WCTMs and CTMs achieve comparably fast ADD; however, relative to comparable PR-ST methods, (W)CTMs are over \textit{three times} faster. Meanwhile, among the stagewise monitoring methods, WCTM and CTM-based SR procedures are comparable, but with significantly faster ADD than PR-CD run either online or in minibatches. PR-CD-online has lower ADD than in minibatches, but its wall-clock runtime is far slower (Appendix \ref{app:runtimes}) due to the method's $\mathcal{O}(t^2)$ time complexity, whereas the proposed WCTM and SR-WCTM methods are $\mathcal{O}(t)$.

\subsection{WCTMs Diagnose Root Cause of Harmful Shifts}

\label{sec:rca}

The left plot corresponds to a benign covariate shift setting (same setting as in Figure \ref{fig:unnecessary_alarms}), and WATCH diagnoses it as such with the $X$-CTM achieving a large value while the WCTM adapts to the shift, avoiding an alarm. The middle plot corresponds to an extreme covariate shift setting, and WATCH diagnoses it by the $X$-CTM identifying covariate shift and the WCTM raising an alarm due to potential harm. Lastly, the right plot corresponds to a harmful concept shift, and WATCH identifies it as such by the WCTM raising an alarm, but with the $X$-CTM maintaining lower values, suggesting that no covariate shift has occurred.

\section{Summary and Future Directions}
\label{sec:possible_directions}

In this paper we introduce weighted-conformal test martingales (WCTMs), which are constructed from sequences of weighted-conformal $p$-values. We presented theory for how online weighted-conformal $p$-values can be used in hypothesis testing, and we demonstrated how WCTMs enable continual monitoring of deployed AI/ML models.
The proposed approach, WATCH, achieves three main goals: (1) \textit{adaptation} to benign shifts to avoid unnecessary alarms and improve the utility of CP prediction sets for end-users; (2) \textit{fast detection} of harmful shifts; and (3) \textit{root-cause analysis} to identify the cause of any performance degradation as a harmful covariate shift or a fundamental concept shift. 

Our contribution opens up several promising directions for future work. 
Future theory directions include developing specific WCTM algorithms for testing other nonparametric null hypotheses, exploring connections to conditional permutation tests (e.g., \citet{berrett2020conditional}), analyzing WCTMs' efficiency with respect to certain alternative hypotheses, developing robustness results for approximate weights, and finer-grained root-cause analysis. For applications, implementations of WATCH to monitor the risks of AI systems in real practical settings such as healthcare and extensions to monitor foundation models and AI agents would be valuable. We hope that WATCH encourages greater attention on post-deployment monitoring for safer AI systems, with a focus on adaptive monitoring, fast detection of harmful shifts, and diagnostics to inform recovery.

\section*{Impact Statement}
This paper presents research aimed at propelling advancements in the broad domain of machine learning. The implications of our findings are wide-ranging, with potential high-stakes applications in sectors including healthcare, autonomous driving, and e-commerce. Based on our current understanding, this research does not warrant an ethics review, and a detailed discussion of the potential societal impacts is not required at the current stage.

\section*{Acknowledgements}

D.P. was funded by a PhD fellowship from the Amazon Initiative for Interactive AI (AI2AI); D.P., X.H., A.L., and S.S. were partially funded by the Gordon and Betty Moore Foundation grant \#12128; D.P., X.H., and S.S. were partially funded by National Science Foundation (NSF) grant \#1840088. A.L. is also supported by an Amazon Research Award. The authors would like to thank Yaniv Romano for helpful discussions on a draft of this paper; Jwala Dhamala and Jie Ding for helpful discussions throughout the JHU + AI2AI collaboration; and the anonymous ICML reviewers for informative feedback.

\section*{Use of AI Assistance}
The authors used Claude Opus 5 (Anthropic) in preparing the v5 revision of this paper. It was used to review the paper's theory, which led to identifying the gap in the original statement of Theorem~\ref{thm:general_weighted_p_validity}; to help draft the counterexample of Proposition~\ref{prop:counterexample}, the bag-sufficiency condition, and the corresponding minimal corrections to the proof; and to clarify the connection to the online compression models of \citet{vovk2003well}. All mathematical claims were checked
by the authors, who maintain full responsibility for all content.

\bibliographystyle{apalike} 

\bibliography{bibliography.bib}

@book{vovk2005algorithmic,
  title={Algorithmic learning in a random world},
  author={Vovk, Vladimir and Gammerman, Alexander and Shafer, Glenn},
  volume={29},
  year={2005},
  publisher={Springer}
}

@article{tibshirani2019conformal,
  title={Conformal prediction under covariate shift},
  author={Tibshirani, Ryan J and Foygel Barber, Rina and Candes, Emmanuel and Ramdas, Aaditya},
  journal={Advances in neural information processing systems},
  volume={32},
  year={2019}
}

@article{romano2020classification,
  title={Classification with valid and adaptive coverage},
  author={Romano, Yaniv and Sesia, Matteo and Candes, Emmanuel},
  journal={Advances in neural information processing systems},
  volume={33},
  pages={3581--3591},
  year={2020}
}

@article{berrett2020conditional,
  title={The conditional permutation test for independence while controlling for confounders},
  author={Berrett, Thomas B and Wang, Yi and Barber, Rina Foygel and Samworth, Richard J},
  journal={Journal of the Royal Statistical Society Series B: Statistical Methodology},
  volume={82},
  number={1},
  pages={175--197},
  year={2020},
  publisher={Oxford University Press}
}

@inproceedings{he2016deep,
  title={Deep residual learning for image recognition},
  author={He, Kaiming and Zhang, Xiangyu and Ren, Shaoqing and Sun, Jian},
  booktitle={Proceedings of the IEEE conference on computer vision and pattern recognition},
  pages={770--778},
  year={2016}
}

@article{mu2019mnist,
  title={Mnist-c: A robustness benchmark for computer vision},
  author={Mu, Norman and Gilmer, Justin},
  journal={arXiv preprint arXiv:1906.02337},
  year={2019}
}

@inproceedings{
  hendrycks2018benchmarking,
  title={Benchmarking Neural Network Robustness to Common Corruptions and Perturbations},
  author={Dan Hendrycks and Thomas Dietterich},
  booktitle={International Conference on Learning Representations},
  year={2019},
  url={https://openreview.net/forum?id=HJz6tiCqYm},
}

@article{amodei2016concrete,
  title={Concrete problems in AI safety},
  author={Amodei, Dario and Olah, Chris and Steinhardt, Jacob and Christiano, Paul and Schulman, John and Man{\'e}, Dan},
  journal={arXiv preprint arXiv:1606.06565},
  year={2016}
}

@article{prinster2024conformal,
  title={Conformal Validity Guarantees Exist for Any Data Distribution (and How to Find Them)},
  author={Prinster, Drew and Stanton, Samuel Don and Liu, Anqi and Saria, Suchi},
  journal={Forty-first International Conference on Machine Learning},
  year={2024}
}

@article{nair2023randomization,
  title={Randomization tests for adaptively collected data},
  author={Nair, Yash and Janson, Lucas},
  journal={arXiv preprint arXiv:2301.05365},
  year={2023}
}

@article{barber2023conformal,
  title={Conformal prediction beyond exchangeability},
  author={Barber, Rina Foygel and Candes, Emmanuel J and Ramdas, Aaditya and Tibshirani, Ryan J},
  journal={The Annals of Statistics},
  volume={51},
  number={2},
  pages={816--845},
  year={2023},
  publisher={Institute of Mathematical Statistics}
}

@inproceedings{vovk2021retrain,
  title={Retrain or not retrain: Conformal test martingales for change-point detection},
  author={Vovk, Vladimir and Petej, Ivan and Nouretdinov, Ilia and Ahlberg, Ernst and Carlsson, Lars and Gammerman, Alex},
  booktitle={Conformal and Probabilistic Prediction and Applications},
  pages={191--210},
  year={2021},
  organization={PMLR}
}

@article{vovk2021testing,
  title={Testing randomness online},
  author={Vovk, Vladimir},
  journal={Statistical Science},
  volume={36},
  number={4},
  pages={595--611},
  year={2021},
  publisher={Institute of Mathematical Statistics}
}

@article{podkopaev2021tracking,
  title={Tracking the risk of a deployed model and detecting harmful distribution shifts},
  author={Podkopaev, Aleksandr and Ramdas, Aaditya},
  journal={arXiv preprint arXiv:2110.06177},
  year={2021}
}

@article{bates2023testing,
  title={Testing for outliers with conformal p-values},
  author={Bates, Stephen and Cand{\`e}s, Emmanuel and Lei, Lihua and Romano, Yaniv and Sesia, Matteo},
  journal={The Annals of Statistics},
  volume={51},
  number={1},
  pages={149--178},
  year={2023},
  publisher={Institute of Mathematical Statistics}
}

@article{ramdas2022testing,
  title={Testing exchangeability: Fork-convexity, supermartingales and e-processes},
  author={Ramdas, Aaditya and Ruf, Johannes and Larsson, Martin and Koolen, Wouter M},
  journal={International Journal of Approximate Reasoning},
  volume={141},
  pages={83--109},
  year={2022},
  publisher={Elsevier}
}

@inproceedings{eliades2023conformal,
  title={A conformal martingales ensemble approach for addressing concept drift},
  author={Eliades, Charalambos and Papadopoulos, Harris},
  booktitle={Conformal and Probabilistic Prediction with Applications},
  pages={328--346},
  year={2023},
  organization={PMLR}
}

@article{bar2024protected,
  title={Protected Test-Time Adaptation via Online Entropy Matching: A Betting Approach},
  author={Bar, Yarin and Shaer, Shalev and Romano, Yaniv},
  journal={arXiv preprint arXiv:2408.07511},
  year={2024}
}

@book{shafer2019game,
  title={Game-theoretic foundations for probability and finance},
  author={Shafer, Glenn and Vovk, Vladimir},
  volume={455},
  year={2019},
  publisher={John Wiley \& Sons}
}

@article{shafer2021testing,
  title={Testing by betting: A strategy for statistical and scientific communication},
  author={Shafer, Glenn},
  journal={Journal of the Royal Statistical Society Series A: Statistics in Society},
  volume={184},
  number={2},
  pages={407--431},
  year={2021},
  publisher={Oxford University Press}
}

@article{fedorova2012plug,
  title={Plug-in martingales for testing exchangeability on-line},
  author={Fedorova, Valentina and Gammerman, Alex and Nouretdinov, Ilia and Vovk, Vladimir},
  journal={arXiv preprint arXiv:1204.3251},
  year={2012}
}

@article{ramdas2023game,
  title={Game-theoretic statistics and safe anytime-valid inference},
  author={Ramdas, Aaditya and Gr{\"u}nwald, Peter and Vovk, Vladimir and Shafer, Glenn},
  journal={Statistical Science},
  volume={38},
  number={4},
  pages={576--601},
  year={2023},
  publisher={Institute of Mathematical Statistics}
}

@book{vovk2022algorithmic,
  title={Algorithmic Learning in a Random World},
  author={Vovk, Vladimir and Gammerman, Alexander and Shafer, Glenn},
  year={2022},
  publisher={Springer Nature}
}

@book{ville1939etude,
  title={Etude critique de la notion de collectif},
  author={Ville, Jean},
  year={1939},
  publisher={Gauthier-Villars Paris}
}

@article{page1954continuous,
  title={Continuous inspection schemes},
  author={Page, Ewan S},
  journal={Biometrika},
  volume={41},
  number={1/2},
  pages={100--115},
  year={1954},
  publisher={JSTOR}
}

@article{roberts1966comparison,
  title={A comparison of some control chart procedures},
  author={Roberts, SW},
  journal={Technometrics},
  volume={8},
  number={3},
  pages={411--430},
  year={1966},
  publisher={Taylor \& Francis}
}

@article{shiryaev1963optimum,
  title={On optimum methods in quickest detection problems},
  author={Shiryaev, Albert N},
  journal={Theory of Probability \& Its Applications},
  volume={8},
  number={1},
  pages={22--46},
  year={1963},
  publisher={SIAM}
}

@inproceedings{podkopaev2021distribution,
  title={Distribution-free uncertainty quantification for classification under label shift},
  author={Podkopaev, Aleksandr and Ramdas, Aaditya},
  booktitle={Uncertainty in artificial intelligence},
  pages={844--853},
  year={2021},
  organization={PMLR}
}

@article{sugiyama2007covariate,
  title={Covariate shift adaptation by importance weighted cross validation.},
  author={Sugiyama, Masashi and Krauledat, Matthias and M{\"u}ller, Klaus-Robert},
  journal={Journal of Machine Learning Research},
  volume={8},
  number={5},
  year={2007}
}

@article{shimodaira2000improving,
  title={Improving predictive inference under covariate shift by weighting the log-likelihood function},
  author={Shimodaira, Hidetoshi},
  journal={Journal of statistical planning and inference},
  volume={90},
  number={2},
  pages={227--244},
  year={2000},
  publisher={Elsevier}
}

@article{zhang2024adapting,
  title={Adapting to continuous covariate shift via online density ratio estimation},
  author={Zhang, Yu-Jie and Zhang, Zhen-Yu and Zhao, Peng and Sugiyama, Masashi},
  journal={Advances in Neural Information Processing Systems},
  volume={36},
  year={2024}
}

@article{yang2024doubly,
  title={Doubly robust calibration of prediction sets under covariate shift},
  author={Yang, Yachong and Kuchibhotla, Arun Kumar and Tchetgen Tchetgen, Eric},
  journal={Journal of the Royal Statistical Society Series B: Statistical Methodology},
  pages={qkae009},
  year={2024},
  publisher={Oxford University Press US}
}

@article{cohen2009medical,
  title={The medical expenditure panel survey: a national information resource to support healthcare cost research and inform policy and practice},
  author={Cohen, Joel W and Cohen, Steven B and Banthin, Jessica S},
  journal={Medical care},
  volume={47},
  number={7\_Supplement\_1},
  pages={S44--S50},
  year={2009},
  publisher={LWW}
}

@misc{superconductivty_data_464,
  author       = {Hamidieh, Kam},
  title        = {{Superconductivty Data}},
  year         = {2018},
  howpublished = {UCI Machine Learning Repository},
  note         = {{DOI}: https://doi.org/10.24432/C53P47}
}

@misc{bike_sharing_275,
  author       = {Fanaee-T, Hadi},
  title        = {{Bike Sharing}},
  year         = {2013},
  howpublished = {UCI Machine Learning Repository},
  note         = {{DOI}: https://doi.org/10.24432/C5W894}
}

@article{pedregosa2011scikit,
  title={Scikit-learn: Machine learning in Python},
  author={Pedregosa, Fabian and Varoquaux, Ga{\"e}l and Gramfort, Alexandre and Michel, Vincent and Thirion, Bertrand and Grisel, Olivier and Blondel, Mathieu and Prettenhofer, Peter and Weiss, Ron and Dubourg, Vincent and others},
  journal={the Journal of machine Learning research},
  volume={12},
  pages={2825--2830},
  year={2011},
  publisher={JMLR. org}
}

@article{feng2022clinical,
  title={Clinical artificial intelligence quality improvement: towards continual monitoring and updating of AI algorithms in healthcare},
  author={Feng, Jean and Phillips, Rachael V and Malenica, Ivana and Bishara, Andrew and Hubbard, Alan E and Celi, Leo A and Pirracchio, Romain},
  journal={NPJ digital medicine},
  volume={5},
  number={1},
  pages={66},
  year={2022},
  publisher={Nature Publishing Group UK London}
}

@article{adams2022prospective,
  title={Prospective, multi-site study of patient outcomes after implementation of the TREWS machine learning-based early warning system for sepsis},
  author={Adams, Roy and Henry, Katharine E and Sridharan, Anirudh and Soleimani, Hossein and Zhan, Andong and Rawat, Nishi and Johnson, Lauren and Hager, David N and Cosgrove, Sara E and Markowski, Andrew and others},
  journal={Nature medicine},
  volume={28},
  number={7},
  pages={1455--1460},
  year={2022},
  publisher={Nature Publishing Group US New York}
}

@article{finlayson2021clinician,
  title={The clinician and dataset shift in artificial intelligence},
  author={Finlayson, Samuel G and Subbaswamy, Adarsh and Singh, Karandeep and Bowers, John and Kupke, Annabel and Zittrain, Jonathan and Kohane, Isaac S and Saria, Suchi},
  journal={New England Journal of Medicine},
  volume={385},
  number={3},
  pages={283--286},
  year={2021},
  publisher={Mass Medical Soc}
}

@inproceedings{volkhonskiy2017inductive,
  title={Inductive conformal martingales for change-point detection},
  author={Volkhonskiy, Denis and Burnaev, Evgeny and Nouretdinov, Ilia and Gammerman, Alexander and Vovk, Vladimir},
  booktitle={Conformal and Probabilistic Prediction and Applications},
  pages={132--153},
  year={2017},
  organization={PMLR}
}

@incollection{papadopoulos2008inductive,
  title={Inductive conformal prediction: Theory and application to neural networks},
  author={Papadopoulos, Harris},
  booktitle={Tools in artificial intelligence},
  year={2008},
  publisher={Citeseer}
}

@article{foygel2021limits,
  title={The limits of distribution-free conditional predictive inference},
  author={Foygel Barber, Rina and Candes, Emmanuel J and Ramdas, Aaditya and Tibshirani, Ryan J},
  journal={Information and Inference: A Journal of the IMA},
  volume={10},
  number={2},
  pages={455--482},
  year={2021},
  publisher={Oxford University Press}
}

@article{feng2025not,
  title={Not All Clinical AI Monitoring Systems Are Created Equal: Review and Recommendations},
  author={Feng, Jean and Xia, Fan and Singh, Karandeep and Pirracchio, Romain},
  journal={NEJM AI},
  volume={2},
  number={2},
  pages={AIra2400657},
  year={2025},
  publisher={Massachusetts Medical Society}
}

@article{vovk2021values,
  title={E-values: Calibration, combination and applications},
  author={Vovk, Vladimir and Wang, Ruodu},
  journal={The Annals of Statistics},
  volume={49},
  number={3},
  pages={1736--1754},
  year={2021},
  publisher={Institute of Mathematical Statistics}
}

@inproceedings{xu2021conformal,
  title={Conformal prediction interval for dynamic time-series},
  author={Xu, Chen and Xie, Yao},
  booktitle={International Conference on Machine Learning},
  pages={11559--11569},
  year={2021},
  organization={PMLR}
}

@article{prinster2022jaws,
  title={JAWS: Auditing Predictive Uncertainty Under Covariate Shift},
  author={Prinster, Drew and Liu, Anqi and Saria, Suchi},
  journal={Advances in Neural Information Processing Systems},
  volume={35},
  pages={35907--35920},
  year={2022}
}

@inproceedings{prinster2023jaws,
  title={JAWS-X: addressing efficiency bottlenecks of conformal prediction under standard and feedback covariate shift},
  author={Prinster, Drew and Saria, Suchi and Liu, Anqi},
  booktitle={International Conference on Machine Learning},
  pages={28167--28190},
  year={2023},
  organization={PMLR}
}

@article{jin2023model,
  title={Model-free selective inference under covariate shift via weighted conformal p-values},
  author={Jin, Ying and Cand{\`e}s, Emmanuel J},
  journal={arXiv preprint arXiv:2307.09291},
  year={2023}
}

@article{feldman2024robust,
  title={Robust Conformal Prediction Using Privileged Information},
  author={Feldman, Shai and Romano, Yaniv},
  journal={arXiv preprint arXiv:2406.05405},
  year={2024}
}

@inproceedings{vovk2003testing,
  title={Testing exchangeability on-line},
  author={Vovk, Vladimir and Nouretdinov, Ilia and Gammerman, Alexander},
  booktitle={Proceedings of the 20th international conference on machine learning (ICML-03)},
  pages={768--775},
  year={2003}
}

@inproceedings{saha2024testing,
  title={Testing exchangeability by pairwise betting},
  author={Saha, Aytijhya and Ramdas, Aaditya},
  booktitle={International Conference on Artificial Intelligence and Statistics},
  pages={4915--4923},
  year={2024},
  organization={PMLR}
}

@inproceedings{vovk2002line,
  title={On-line confidence machines are well-calibrated},
  author={Vovk, Vladimir},
  booktitle={The 43rd Annual IEEE Symposium on Foundations of Computer Science, 2002. Proceedings.},
  pages={187--196},
  year={2002},
  organization={IEEE}
}

@book{shiryaev2016probability,
  title={Probability-1},
  author={Shiryaev, Albert N},
  volume={95},
  year={2016},
  publisher={Springer}
}

@article{fischer2025sequential,
  title={Sequential Monte Carlo testing by betting},
  author={Fischer, Lasse and Ramdas, Aaditya},
  journal={Journal of the Royal Statistical Society Series B: Statistical Methodology},
  pages={qkaf014},
  year={2025},
  publisher={Oxford University Press UK}
}

@incollection{wald1945sequential,
  title={Sequential tests of statistical hypotheses},
  author={Wald, Abraham},
  booktitle={Breakthroughs in statistics: Foundations and basic theory},
  pages={117-186},
  year={1945},
  publisher={Springer}
}

@book{tartakovsky2014sequential,
  title={Sequential analysis: Hypothesis testing and changepoint detection},
  author={Tartakovsky, Alexander and Nikiforov, Igor and Basseville, Michele},
  year={2014},
  publisher={CRC press}
}

@incollection{veeravalli2014quickest,
  title={Quickest change detection},
  author={Veeravalli, Venugopal V and Banerjee, Taposh},
  booktitle={Academic press library in signal processing},
  volume={3},
  pages={209--255},
  year={2014},
  publisher={Elsevier}
}

@article{xie2021sequential,
  title={Sequential (quickest) change detection: Classical results and new directions},
  author={Xie, Liyan and Zou, Shaofeng and Xie, Yao and Veeravalli, Venugopal V},
  journal={IEEE Journal on Selected Areas in Information Theory},
  volume={2},
  number={2},
  pages={494--514},
  year={2021},
  publisher={IEEE}
}

@inproceedings{ho2005martingale,
  title={A martingale framework for concept change detection in time-varying data streams},
  author={Ho, Shen-Shyang},
  booktitle={Proceedings of the 22nd international conference on Machine learning},
  pages={321--327},
  year={2005}
}

@inproceedings{ho2019martingale,
  title={A martingale-based approach for flight behavior anomaly detection},
  author={Ho, Shen-Shyang and Schofield, Matthew and Sun, Bo and Snouffer, Jason and Kirschner, Jean},
  booktitle={2019 20th IEEE International Conference on Mobile Data Management (MDM)},
  pages={43--52},
  year={2019},
  organization={IEEE}
}

@inproceedings{papadopoulos2002inductive,
  title={Inductive confidence machines for regression},
  author={Papadopoulos, Harris and Proedrou, Kostas and Vovk, Volodya and Gammerman, Alex},
  booktitle={Machine learning: ECML 2002: 13th European conference on machine learning Helsinki, Finland, August 19--23, 2002 proceedings 13},
  pages={345--356},
  year={2002},
  organization={Springer}
}

@inproceedings{eliades2022betting,
  title={A betting function for addressing concept drift with conformal martingales},
  author={Eliades, Charalambos and Papadopoulos, Harris},
  booktitle={Conformal and Probabilistic Prediction with Applications},
  pages={219--238},
  year={2022},
  organization={PMLR}
}

@article{vovk2020testing,
  title={Testing for concept shift online},
  author={Vovk, Vladimir},
  journal={arXiv preprint arXiv:2012.14246},
  year={2020}
}

@article{angelopoulos2024theoretical,
  title={Theoretical foundations of conformal prediction},
  author={Angelopoulos, Anastasios N and Barber, Rina Foygel and Bates, Stephen},
  journal={arXiv preprint arXiv:2411.11824},
  year={2024}
}

@article{waudby2024estimating,
  title={Estimating means of bounded random variables by betting},
  author={Waudby-Smith, Ian and Ramdas, Aaditya},
  journal={Journal of the Royal Statistical Society Series B: Statistical Methodology},
  volume={86},
  number={1},
  pages={1--27},
  year={2024},
  publisher={Oxford University Press US}
}

@article{i2024sequential,
  title={Sequential Harmful Shift Detection Without Labels},
  author={I Amoukou, Salim and Bewley, Tom and Mishra, Saumitra and Lecue, Freddy and Magazzeni, Daniele and Veloso, Manuela},
  journal={Advances in Neural Information Processing Systems},
  volume={37},
  pages={129279--129302},
  year={2024}
}

@article{hindy2024diagnostic,
  title={Diagnostic Runtime Monitoring with Martingales},
  author={Hindy, Ali and Luo, Rachel and Banerjee, Somrita and Kuck, Jonathan and Schmerling, Edward and Pavone, Marco},
  journal={arXiv preprint arXiv:2407.21748},
  year={2024}
}

@article{wasserman2020universal,
  title={Universal inference},
  author={Wasserman, Larry and Ramdas, Aaditya and Balakrishnan, Sivaraman},
  journal={Proceedings of the National Academy of Sciences},
  volume={117},
  number={29},
  pages={16880--16890},
  year={2020},
  publisher={National Academy of Sciences}
}

@article{shin2022detectors,
  title={E-detectors: A nonparametric framework for sequential change detection},
  author={Shin, Jaehyeok and Ramdas, Aaditya and Rinaldo, Alessandro},
  journal={arXiv preprint arXiv:2203.03532},
  year={2022}
}

@article{shekhar2023reducing,
  title={Reducing sequential change detection to sequential estimation},
  author={Shekhar, Shubhanshu and Ramdas, Aaditya},
  journal={arXiv preprint arXiv:2309.09111},
  year={2023}
}

@inproceedings{shekhar2023sequential,
  title={Sequential changepoint detection via backward confidence sequences},
  author={Shekhar, Shubhanshu and Ramdas, Aaditya},
  booktitle={International Conference on Machine Learning},
  pages={30908--30930},
  year={2023},
  organization={PMLR}
}

@article{podkopaev2023sequential,
  title={Sequential predictive two-sample and independence testing},
  author={Podkopaev, Aleksandr and Ramdas, Aaditya},
  journal={Advances in neural information processing systems},
  volume={36},
  pages={53275--53307},
  year={2023}
}

@article{shekhar2023nonparametric,
  title={Nonparametric two-sample testing by betting},
  author={Shekhar, Shubhanshu and Ramdas, Aaditya},
  journal={IEEE Transactions on Information Theory},
  volume={70},
  number={2},
  pages={1178--1203},
  year={2023},
  publisher={IEEE}
}

@article{barber2025unifying,
  title={Unifying Different Theories of Conformal Prediction},
  author={Barber, Rina Foygel and Tibshirani, Ryan J},
  journal={arXiv preprint arXiv:2504.02292},
  year={2025}
}

@article{fannjiang2022conformal,
  title={Conformal prediction under feedback covariate shift for biomolecular design},
  author={Fannjiang, Clara and Bates, Stephen and Angelopoulos, Anastasios N and Listgarten, Jennifer and Jordan, Michael I},
  journal={Proceedings of the National Academy of Sciences},
  volume={119},
  number={43},
  pages={e2204569119},
  year={2022},
  publisher={National Acad Sciences}
}

@inproceedings{stanton2023bayesian,
  title={Bayesian Optimization with Conformal Prediction Sets},
  author={Stanton, Samuel and Maddox, Wesley and Wilson, Andrew Gordon},
  booktitle={International Conference on Artificial Intelligence and Statistics},
  pages={959--986},
  year={2023},
  organization={PMLR}
}

@article{farinhas2023non,
  title={Non-Exchangeable Conformal Risk Control},
  author={Farinhas, Ant{\'o}nio and Zerva, Chrysoula and Ulmer, Dennis and Martins, Andr{\'e} FT},
  journal={arXiv preprint arXiv:2310.01262},
  year={2023}
}

@article{teneggi2024testing,
  title={Testing Semantic Importance via Betting},
  author={Teneggi, Jacopo and Sulam, Jeremias},
  journal={Advances in Neural Information Processing Systems},
  volume={37},
  pages={76450--76499},
  year={2024}
}

@inproceedings{shaer2023model,
  title={Model-X sequential testing for conditional independence via testing by betting},
  author={Shaer, Shalev and Maman, Gal and Romano, Yaniv},
  booktitle={International Conference on Artificial Intelligence and Statistics},
  pages={2054--2086},
  year={2023},
  organization={PMLR}
}

@article{bashari2023derandomized,
  title={Derandomized novelty detection with FDR control via conformal e-values},
  author={Bashari, Meshi and Epstein, Amir and Romano, Yaniv and Sesia, Matteo},
  journal={Advances in Neural Information Processing Systems},
  volume={36},
  pages={65585--65596},
  year={2023}
}

@article{gauthier2025values,
  title={E-Values Expand the Scope of Conformal Prediction},
  author={Gauthier, Etienne and Bach, Francis and Jordan, Michael I},
  journal={arXiv preprint arXiv:2503.13050},
  year={2025}
}

@article{lee2025selection,
  title={Selection from Hierarchical Data with Conformal e-values},
  author={Lee, Yonghoon and Ren, Zhimei},
  journal={arXiv preprint arXiv:2501.02514},
  year={2025}
}

@article{vovk2023confidence,
  title={Confidence and discoveries with e-values},
  author={Vovk, Vladimir and Wang, Ruodu},
  journal={Statistical Science},
  volume={38},
  number={2},
  pages={329--354},
  year={2023},
  publisher={Institute of Mathematical Statistics}
}

@inproceedings{vovk2003well,
  title={Well-calibrated predictions from online compression models},
  author={Vovk, Vladimir},
  booktitle={Algorithmic Learning Theory: 14th International Conference, ALT 2003, Sapporo, Japan, October 17-19, 2003. Proceedings 14},
  pages={268--282},
  year={2003},
  organization={Springer}
}

@inproceedings{vovk2023power,
  title={The power of forgetting in statistical hypothesis testing},
  author={Vovk, Vladimir},
  booktitle={Conformal and Probabilistic Prediction with Applications},
  pages={347--366},
  year={2023},
  organization={PMLR}
}

\newpage
\appendix
\onecolumn

\begin{center}
\textbf{\Large Appendix for \\ \vspace{.2em} 
``WATCH: Adaptive Monitoring for AI Deployments \\via Weighted-Conformal Martingales''}
\end{center}

\section{Related Works}
\label{sec:related_work}

This work is motivated by developing methods for monitoring AI/ML deployments that perform three key functions: (1) online adaptation to mild or benign data shifts; (2) rapid detection of extreme or harmful shifts that necessitate updates; and (3) identifying the root-cause of degradation to inform appropriate recovery. 
Although these monitoring goals could be viewed from many perspectives, here we primarily focus on methods in anytime-valid inference, sequential testing of nonparametric null hypotheses, and especially conformal test martingales, which most closely relate to our own.

\textbf{Sequential testing for changepoints in the data distribution:} Sequential hypothesis testing to detect changes in the data distribution is an old and widely-studied problem, dating at least to Wald's sequential probability ratio test \citep{wald1945sequential}. However, classic sequential changepoint detection methods often required a prespecified stopping time and were primarily designed for testing simple, often parametrically-specified null hypotheses---for an overview of classic parametric methods, we defer to a relevant textbook \citet{tartakovsky2014sequential} and review articles \citep{veeravalli2014quickest, xie2021sequential}. In contrast, our weighted-conformal test martingale (WCTM) methods belong to a more recent literature called sequential anytime-valid inference (SAVI), which allow
for arbitrary stopping times, and specifically our work is situated in the recent SAVI literature on testing composite and nonparametric null hypotheses. We refer readers to  \citet{ramdas2023game} for a recent review of the SAVI literature, as well as to the textbook by \citet{shafer2019game} on the closely related topic of testing-by-betting.

\textbf{Standard conformal test martingales:} Within the SAVI literature, our work is closely related to conformal test martingales (CTMs), which are martingales constructed from a sequence of standard conformal $p$-values for continually testing the assumption that the data are exchangeable or independent and identically distributed (IID). The WCTM methods we introduce generalize standard CTMs for testing a broader range of null hypotheses, including those where one wishes to accommodate or adapt to certain anticipated changes in the data distribution (e.g., adapting to mild covariate shifts). Standard CTMs were initially introduced by \citet{vovk2003testing}, while drawing on theory for the calibration of online conformal prediction methods developed in \citet{vovk2002line}. Since then, various works have further developed standard CTMs, such as by introducing new betting and score functions for more efficient changepoint detection, ensembling CTMs, and demonstrating their performance on real-world applications \citep{ho2005martingale, fedorova2012plug, volkhonskiy2017inductive, ho2019martingale, vovk2021retrain, vovk2021testing, eliades2022betting, eliades2023conformal}. CTMs are also discussed in textbooks on conformal prediction \citep{vovk2005algorithmic, angelopoulos2024theoretical}, and they are closely related to online compression models \citep{vovk2003well}, as discussed at the end of this section.

In particular, \citet{volkhonskiy2017inductive} developed inductive CTMs (i.e., CTMs based on inductive or split conformal \citep{papadopoulos2002inductive, papadopoulos2008inductive}) with score and betting functions specifically tailored to fast changepoint detection; \citet{vovk2021retrain} proposed an approach to using standard CTMs to monitor for when a deployed AI/ML system should be retrained; and \citet{vovk2021testing} provided a comprehensive and detailed review of CTMs as methods for testing the IID or exchangeability assumptions. \citet{bar2024protected} implement methods based on CTMs for test time-adaptation of a classifier's point prediction and \citet{hindy2024diagnostic} leverage multiple CTMs over different feature spaces to aid in diagnostic runtime monitoring.
Several works have developed or used CTMs for testing for concept shift (i.e., shift in $Y\mid X$) \citep{ho2005martingale, eliades2022betting, eliades2023conformal, vovk2020testing}, but all of these have focused on a classification setting with a limited number of classes (e.g., to implement label-conditional CTMs, or using standard CTMs that lack ability to disambiguate between covariate and concept shifts). In contrast, our specific WCTMs implemented in this paper are able to test for concept shift in both regression and classification settings while also being able to help disambiguate between concept shifts and extreme covariate shifts with the aid of an additional $X$-CTM.

\textbf{Confidence sequences, $e$-processes, and other multiple-hypothesis testing methods:} Our work also relates more broadly to other SAVI and testing-by-betting methods outside of CTMs such as confidence sequences and $e$-processes---we refer to the review paper of \citet{ramdas2023game} for full exposition of these topics. In particular, in the main paper we empirically compared our WCTM methods against the betting-based sequential testing and changepoint detection methods in \citet{podkopaev2021tracking} (the theory for which was developed in \citet{waudby2024estimating}) because this paper is the closest to ours in its motivation of monitoring deployed AI/ML deployments among (non-CTM) SAVI papers. As we note in the main paper, for monitoring the risk of a set-valued (e.g., conformalized) AI/ML predictor, our WCTM methods are generally more data-efficient (not requiring separate datasets for conformalization and testing, as \citet{podkopaev2021tracking} does); our WCTM-based Shiryaev-Roberts procedure is more computationally efficient than the comparable changepoint detection method in \citet{podkopaev2021tracking} (i.e., $\mathcal{O}(t)$ versus $\mathcal{O}(t^2)$); and, empirically we found that our methods often detect harmful concept shifts faster than comparable methods in \citet{podkopaev2021tracking}. More recently, \citet{i2024sequential} built on \citet{podkopaev2021tracking} by developing similar methods for when ground-truth labels do not become  available and need to be estimated.

A key target of the SAVI literature on testing composite or nonparametric null hypotheses has been developing new approaches for testing the IID or exchangeability assumptions. Other than CTMs, the other main nontrivial approach to sequentially testing exchangeability was developed in \citet{ramdas2022testing} based on the ideas of universal inference \citep{wasserman2020universal}; more recently, other approaches have emerged based on pairwise betting \citep{saha2024testing} and sequential Monte Carlo testing \citep{fischer2025sequential}, the latter of which can be viewed as a special case of CTMs \citep{vovk2021testing} streamlined for a particular alternative hypothesis. Otherwise, there are various and proliferating other methods for sequential nonparametric changepoint detection (e.g., \citet{shin2022detectors, shekhar2023nonparametric, shekhar2023reducing, shekhar2023sequential, podkopaev2023sequential}). SAVI and testing-by betting methods are also being leveraged for a wide variety of applications including interpretability \citep{teneggi2024testing}, conditional independence testing \citep{shaer2023model}, applications in finance \citep{shafer2019game}, and more.

Testing-by-betting, which is fundamental to SAVI \citep{ramdas2023game}, can be understood as one approach to multiple hypothesis testing that is especially advantageous in sequential settings. Other related works that leverage conformal prediction for hypothesis testing while accounting for multiple-testing corrections, but in batch settings, include \citet{bates2023testing, bashari2023derandomized, vovk2023confidence, gauthier2025values, lee2025selection}.

\textbf{Weighted conformal prediction and adapting to distribution shifts:} Weighted conformal prediction (e.g., \citet{tibshirani2019conformal, podkopaev2021distribution, xu2021conformal, fannjiang2022conformal, prinster2022jaws, prinster2023jaws, stanton2023bayesian, barber2023conformal, farinhas2023non, nair2023randomization, yang2024doubly, feldman2024robust, barber2025unifying}) is broadly an approach to proactively adapting the validity of conformal predictive sets to distribution shift by reweighting nonconformity scores using either knowledge or estimates of the distribution shift. Any weighted CP prediction set is associated with a weighted-conformal $p$-value, as described in the main paper; accordingly, a WCTM can be constructed on top of any weighted CP method deployed on a sequence of data observations to continually monitor the assumptions or approximations underlying that WCP method's implementation. There are of course many other approaches to adapting to distribution shifts at test time; for example, one work that is similar to ours with regard to this motivation, but that does not fall under weighted CP, is \citet{bar2024protected}, which leverages standard CTMs to guide the test-time adaptation of a classifier's point prediction by entropy matching.

\textbf{Online compression models:} The online form of conformal prediction that CTMs are based on is one form of online compression model (OCM), as defined and discussed in \citet{vovk2003well, vovk2005algorithmic, vovk2023power}.
In particular, \citet{vovk2003well} generalizes conformal prediction beyond exchangeability by replacing the assumption
``the data are exchangeable'' with ``the data agree with a given online compression model'': a
summarizing statistic that can be updated online, together with a known backward kernel giving the
law of the previous summary and current example given the current summary. Conformal validity holds
in any such model \citep[Thm.~1]{vovk2003well}. The \emph{exchangeability model} takes the bag as
summarizing statistic and a backward kernel uniform over orderings, recovering standard conformal
prediction; however, \citet{vovk2003well} also develops OCMs for other models, such as 
a type of Markov model. 
The oracle-weighted $p$-values studied here
correspond to the bag summary with a \emph{general} backward kernel, which is what permits the
marginal law of $Z_t$ to drift with $t$; the bag sufficiency condition is precisely
the requirement that the bag be a valid summarizing or sufficient statistic. 
However, the study of OCMs has so far focused on a relatively select set of models and has largely been disconnected from the literature on adapting to distribution shift via weighted conformal prediction; for instance, online testing of the covariate shift null hypothesis has, to our knowledge, not previously been studied with OCMs or sequential testing broadly.
We hope that our contributions can help bridge these gaps and open the door to new sequential testing settings. 

To aid in comparing terminology between our paper and OCMs, Table \ref{tab:ocm-dictionary} compares terms in our paper to those of \citet{vovk2003well}.
The connection noted in Remark~\ref{remark:ocm} suggests that Theorem~\ref{thm:general_weighted_p_validity} may be related to Theorem 1 of
\citet{vovk2003well}, which establishes conformal validity in any OCM; we have not attempted a formal reduction, and we leave the precise relationship to future work. We provide a self-contained proof in contemporary notation, both for accessibility and to encourage greater appreciation of Vovk's OCM contributions.

\begin{table*}[t]
\centering
\small
\begin{tabular}{@{}p{0.40\textwidth}p{0.54\textwidth}@{}}
\toprule
\textbf{\citet{vovk2003well}} & \textbf{Present paper} \\
\midrule
example $z_n$
  & observation $Z_t = (X_t, Y_t)$ \\
$n$th statistic $t_n(z_1,\dots,z_n)$; summary $\sigma_n$
  & sufficient statistic of the past; here the bag $E_z^{(t)}$ \\
summary space $\Sigma$, empty summary $\square$
  & the space of bags; empty bag $\emptyset$ \\
forward function $F_n : \Sigma \times \mathcal{Z} \to \Sigma$
  & the update $\{Z_{1:t}\} = \{Z_{1:t-1}\} \cup \{Z_t\}$ \\
backward kernel $B_n(\cdot \mid \sigma_n)$
  & conditional law of $(\{Z_{1:t-1}\}, Z_t)$ given $\{Z_{1:t}\}$, i.e.\ \textbf{the oracle weights $\tilde w^{o}$} \\
``$P$ agrees with the OCM''
  & $\mathcal{H}_0(\hat f)$ \\
exchangeability model (backward kernel uniform over orderings)
  & standard, unweighted conformal prediction \\
individual strangeness measure $A_n(\sigma, z)$
  & nonconformity score $\hat S$ \\
randomised confidence transducer
  & the smoothed conformal $p$-value process $(P_t)$ \\
transducer is \emph{valid}
  & $P_1, P_2, \dots \stackrel{iid}{\sim} \text{Unif}[0,1]$ \\
\bottomrule
\end{tabular}
\caption{Correspondence between the online compression model framework of
\citet{vovk2003well} and the notation of this paper. The key connection is that Vovk's backward kernel is our oracle
weight vector.}
\label{tab:ocm-dictionary}
\end{table*}


\vspace{3cm}

\section{Proof of Theorems}
\label{sec:proof}

First recall the definition of bag sufficiency, where $E_z^{(t-1)}$ denotes the event that $\{Z_1, ..., Z_{t-1}\}=\{z_1, ..., z_{t-1}\}$.

\begin{definition}[Bag Sufficiency]
    A sequence of random variables, $Z_1, Z_2, ...$, is bag-sufficient if, for every $t\geq 1$,
    \begin{align}\label{appeq:bag_sufficiency}
    Z_t \indep Z_1, ..., Z_{t-1} \mid E_z^{(t-1)}.
    \end{align}
\end{definition}

\subsection{Proof for Theorem \ref{thm:general_weighted_p_validity} (Independence and Exact Validity of Online Weighted-Conformal $p$-Values Under Bag Sufficiency)}
\label{subsec:p_val_validity_pf}

The proof can be viewed as a generalization of the proofs for Theorem 1 in \citet{vovk2003testing} and Theorem 2 in \citet{vovk2002line}, while drawing on analysis from \citet{tibshirani2019conformal} and exposition and discussion from \citet{prinster2024conformal}. 
The key difference relative to \citet{vovk2002line} and \citet{vovk2003testing} is that, whereas those papers use the assumption of exchangeability to place equal weight on every permutation of the data observations, here we avoid this assumption by first proving a general result for any (potentially non-exchangeable) joint distribution that satisfies bag sufficiency. We then describe how this implies that the validity of more specific and tractable methods is premised on the assumptions or approximations used for practical implementation.

The basic idea for the proof begins with setup from \citet{vovk2002line}, for ``reversing time.'' In particular, we imagine that the sequence of data observations $(z_1, ..., z_T)$ is generated in two steps: first, the unordered bag or multiset of data observations, $\{z_1, ..., z_T\}$, is generated from some probability distribution (that is, the image of $F_Z$ under the mapping $(z_1, z_2, ...) \rightarrow \{z_1, ..., z_T\}$); then---here generalizing \citet{vovk2002line} and \citet{vovk2003testing} by weighting permutations according to their likelihood---from all possible permutations $\sigma$ of the values $\{z_1, ..., z_T\}$, each possible sequence $(z_{\sigma(1)}, ..., z_{\sigma(T)})$ is chosen with probability proportional to $f(z_{\sigma(1)}, ..., z_{\sigma(T)})$, where $f:=f_Z$ is the probability-density function\footnote{More generally, the Radon–Nikodym derivative.} for the distribution $F_Z$. Roughly (ignoring borderline effects), the second step ensures the following. Conditionally on knowing $\{z_1, ..., z_T\}$ (and thus unconditionally), $P_{T}^{\tilde{w}^o}$ has a standard uniform distribution. When $Z_T=z_{\sigma(T)}$ is observed, this settles the value of $P_{T}^{\tilde{w}^o}=p_{\sigma(T)}^{\tilde{w}^o}$. Then, bag sufficiency ensures that knowing $\{z_1, ..., z_T\}$ and $Z_T=z_{\sigma(T)}$ is equivalent, as far as $P_{T-1}^{\tilde{w}^o}, ..., P_{1}^{\tilde{w}^o}$ are concerned, to knowing $\{z_1, ..., z_{T-1}\}$ (after relabeling indices), so we can conclude $P_{T-1}^{\tilde{w}^o}$ also has a standard uniform distribution, and so on.


\begin{lemma}[Restatement of Lemma \ref{main:lemma1}]\label{app:lemma1}
    For any trial $t$ and any confidence level $\alpha \in [0, 1]$,
    \begin{align}\label{eq:lemma1}
        \mathbb{P}\{P_{t}^{\tilde{w}^o} \leq \alpha \mid E_{z}^{(t)} \} = \alpha.
    \end{align}
\end{lemma}

\begin{proof}[Proof of Lemma \ref{app:lemma1}]
    We begin by conditioning on the event $\{Z_1, ..., Z_t\}=\{z_1, ..., z_t\}$, which we denote as $E_{z}^{(t)}$, and we consider drawing any particular ordering or permutation $\sigma$ of the data values with probability according to $f$, that is with probability proportional to $f(z_{\sigma(1)}, ..., z_{\sigma(t)})$. 



    Note that for any $i\in [t]$, if a permutation is drawn such that $\sigma(t)=i$, this means that
    $Z_t=z_{\sigma(t)}=z_i$, and hence---since $\widehat{\mathcal{S}}$ is a function of the
    data---that $V_t=v_{\sigma(t)}=v_i$. Thus, given the bag of data $E_z^{(t)}$, recall that for
    each $i\in [t]$, the probability of drawing such a permutation is given by the ``oracle weights''
    \begin{align}\label{eq:general_weights_app}
    \tilde{w}^o_i := \mathbb{P}\{Z_{t} = z_i \mid E_{z}^{(t)}\} = \frac{\sum_{\sigma:\sigma(t)=i}f(z_{\sigma(1)}, ..., z_{\sigma(t)})}{\sum_{\sigma}f(z_{\sigma(1)}, ..., z_{\sigma(t)})},
    \end{align}
    which we assume to be well-defined, which will generally be the case in practice, where at least
    $f(z_1, ..., z_T)>0$ for the true (identity permutation) ordering of the data observations
    $(z_1, ..., z_T)$.

    Pushing \eqref{eq:general_weights_app} forward through $\widehat{\mathcal{S}}$,\footnote{Note that only the forward implication $Z_t = z_i \Rightarrow V_t = v_i$ is used here. If two bag elements share a score, $v_i = v_j$ for $i\neq j$, then the corresponding intervals $[p_i^{\tilde{w}^o-}, p_i^{\tilde{w}^o+})$ defined below simply coincide, and their common length $\sum_{k:v_k=v_i}\tilde{w}^o_k$ is the conditional probability that a bag element with that score was the one observed at trial $t$; the tiling of $[0,1)$ and the argument that follows are unaffected.} the conditional
    distribution of the test-point score given the bag of data values $E_{z}^{(t)}$ is
    \begin{align*}
        V_t \mid E_{z}^{(t)} \sim \sum_{i=1}^t\tilde{w}^o_i\cdot \delta_{v_i}.
    \end{align*}

    For any $i\in[t]$, define a conservative WCP $p$-value, $p_i^{\tilde{w}^o+}$, and an anticonservative WCP $p$-value, $p_i^{\tilde{w}^o-}$, as
    \begin{align*}
        p_i^{\tilde{w}^o+} & := \sum_{j=1}^t\tilde{w}^o_j\cdot \mathbbm{1}\{v_j\geq v_i\} \\
        p_i^{\tilde{w}^o-} & := \sum_{j=1}^t\tilde{w}^o_j\cdot \mathbbm{1}\{v_j> v_i\}.
    \end{align*}
    It is worth noting that $p_t^{\tilde{w}^o+}$ is a valid $p$-value for $f$: $\mathbb{P}\{P_t^{\tilde{w}^o+}\leq \alpha \mid E_{z}^{(t)}\} \leq \alpha \implies \mathbb{P}\{P_t^{\tilde{w}^o+}\leq \alpha\} \leq \alpha$. However, the lemma claims \textit{exact} validity for $p_t^{\tilde{w}^o}$ conditional on $E_{z}^{(t)}$, which we will now proceed to show.

    Observe that for all $i\in [t]$, $p_i^{\tilde{w}^o-}\leq p_i^{\tilde{w}^o+}$ and 
    \begin{align*}
        p_i^{\tilde{w}^o+}-p_i^{\tilde{w}^o-} = \sum_{j=1}^t\tilde{w}^o_j\cdot \mathbbm{1}\{v_j = v_i\}.
    \end{align*}
    Moreover, observe that as in the proof for Lemma 1 in \citet{vovk2002line}, the semi-closed intervals $[p_i^{\tilde{w}^o-}, p_i^{\tilde{w}^o+})$ either coincide or are disjoint, and $\cup_{i=1}^t[p_i^{\tilde{w}^o-}, p_i^{\tilde{w}^o+})=[0,1)$. 

    Similarly as in the proof for Lemma 1 in \citet{vovk2002line}, for an $\alpha\in [0, 1]$, let us partition the indices as follows, where we say that an index $i$ is
    \begin{itemize}
        \item ``strange'' if $p_i^{\tilde{w}^o+}\leq \alpha$,
        \item ``ordinary'' if $p_i^{\tilde{w}^o-} > \alpha$,
        \item and ``borderline'' if $p_i^{\tilde{w}^o-} \leq \alpha < p_i^{\tilde{w}^o+}$.
    \end{itemize}

    Let $i'$ denote the index of any borderline example, and denote $p^{\tilde{w}^o+}:=p_{i'}^{\tilde{w}^o+}$ and $p^{\tilde{w}^o-}:=p_{i'}^{\tilde{w}^o-}$. Then, the probability (conditional on $E_{z}^{(t)}$, drawing each permutation $\sigma$ with probabilities according to $f$) that the last index $\sigma(t)$ is strange is $p^{\tilde{w}^o-}$; the probability that $\sigma(t)$ is ordinary is $1-p^{\tilde{w}^o+}$; and, the probability that $\sigma(t)$ is borderline is $p^{\tilde{w}^o+}-p^{\tilde{w}^o-}$. Moreover, observe that if $\sigma(t)$ is strange, then $p_{\sigma(t)}^{\tilde{w}^o}\leq \alpha$ (by definition) and if $\sigma(t)$ is borderline then the event that $p_{\sigma(t)}^{\tilde{w}^o}\leq \alpha$ is determined by the independent uniform $u_t$, and thus the probability of this event is $\frac{\alpha-p^{\tilde{w}^o-}}{p^{\tilde{w}^o+}-p^{\tilde{w}^o-}}$. That is,
    
    \begin{align*}
        \mathbb{P}\{P_{t}^{\tilde{w}^o} \leq \alpha \mid E_{z}^{(t)} \} = & \ \mathbb{P}\{P_{t}^{\tilde{w}^o} \leq \alpha \mid E_{z}^{(t)} , \ \sigma(t) \text{ is strange} \}\cdot\mathbb{P}\{\sigma(t) \text{ is strange} \mid E_{z}^{(t)}  \} \nonumber \\
        & + \mathbb{P}\{P_{t}^{\tilde{w}^o} \leq \alpha \mid E_{z}^{(t)} , \ \sigma(t) \text{ is ordinary} \}\cdot\mathbb{P}\{\sigma(t) \text{ is ordinary} \mid E_{z}^{(t)}  \} \nonumber \\
        & + \mathbb{P}\{P_{t}^{\tilde{w}^o} \leq \alpha \mid E_{z}^{(t)} , \ \sigma(t) \text{ is borderline} \}\cdot\mathbb{P}\{\sigma(t) \text{ is borderline} \mid E_{z}^{(t)}  \} \\
        = & \ p^{\tilde{w}^o-} + 0 + (p^{\tilde{w}^o+}-p^{\tilde{w}^o-})\cdot\frac{\alpha-p^{\tilde{w}^o-}}{p^{\tilde{w}^o+}-p^{\tilde{w}^o-}} \\
        = & \ \alpha.
    \end{align*}
\end{proof}

With Lemma \ref{app:lemma1} in hand, we can now proceed with the proof for the main theorem. Temporarily fix a positive integer $T$; following the strategy of \citet{vovk2003testing}, we will first prove by induction that for any $t=1, ..., T$ and any $\alpha_1, ..., \alpha_t\in[0,1]^t$, that
\begin{align}\label{eq:conditional_iid_bernoulli}
    \mathbb{P}\{P_t^{\tilde{w}^o} \leq \alpha_t, ..., P_1^{\tilde{w}^o} \leq \alpha_1 \mid E_{z}^{(t)}\} = \alpha_t \cdots \alpha_1.
\end{align}

For $t=1$, Eq. \eqref{eq:conditional_iid_bernoulli} immediately follows from Lemma \ref{app:lemma1}. For $t>1$, by the law of total probability over $\sigma(t)$ (i.e., over the last index value after drawing a permutation $\sigma$ from $E_z^{(t)}$), the fundamental bridge between probability and expectation, and properties of the indicator function, we have
\begin{align*}
    \mathbb{P}&\{P_t^{\tilde{w}^o} \leq \alpha_t, ..., P_1^{\tilde{w}^o} \leq \alpha_1 \mid E_{z}^{(t)}\} \\
    & = \sum_{\sigma(t)=1}^t\mathbb{P}\{P_t^{\tilde{w}^o} \leq \alpha_t, ..., P_1^{\tilde{w}^o} \leq \alpha_1 \mid E_{z}^{(t)}, Z_t = z_{\sigma(t)}\}\cdot \mathbb{P}\{Z_t = z_{\sigma(t)} \mid E_{z}^{(t)}\}\\
    & = \sum_{\sigma(t)=1}^t\mathbb{E}\big[\mathbbm{1}\{P_t^{\tilde{w}^o} \leq \alpha_t, ..., P_1^{\tilde{w}^o} \leq \alpha_1\} \mid E_{z}^{(t)}, Z_t = z_{\sigma(t)}\big]\cdot \mathbb{P}\{Z_t = z_{\sigma(t)} \mid E_{z}^{(t)}\} \\
    & = \sum_{\sigma(t)=1}^t\mathbb{E}\big[\mathbbm{1}\{P_t^{\tilde{w}^o} \leq \alpha_t\}\cdot\mathbbm{1}\{P_{t-1}^{\tilde{w}^o} \leq \alpha_{t-1}, ..., P_1^{\tilde{w}^o} \leq \alpha_1\} \mid E_{z}^{(t)}, Z_t = z_{\sigma(t)}\big]\cdot \mathbb{P}\{Z_t = z_{\sigma(t)} \mid E_{z}^{(t)}\}
\end{align*}

Next, although conditioning on $E_{z}^{(t)}$ and $Z_t = z_{\sigma(t)}$ does not determine
$\mathbbm{1}\{P_t^{\tilde{w}^o} \leq \alpha_t\}$ due to the independent uniform $U_t$, we can proceed using conditional independence. Since $Z_t = z_{\sigma(t)}$ implies
$V_t = v_{\sigma(t)}$, we have
\begin{align*}
    P_{t}^{\tilde{w}^o} = p_{\sigma(t)}^{\tilde{w}^o-} + u_t\big(p_{\sigma(t)}^{\tilde{w}^o+}-p_{\sigma(t)}^{\tilde{w}^o-}\big),
\end{align*}
which remains random through the smoothing variable $u_t$ whenever $\sigma(t)$ is borderline in the
sense of Lemma \ref{app:lemma1}. 
Conditional on $E_{z}^{(t)}$ and $Z_t = z_{\sigma(t)}$, the first indicator
$\mathbbm{1}\{P_t^{\tilde{w}^o} \leq \alpha_t\}$ is a function of $u_t$ alone, whereas
$P_{t-1}^{\tilde{w}^o}, ..., P_{1}^{\tilde{w}^o}$ are functions of the ordering in which the
remaining $t-1$ values were observed, together with $u_1, ..., u_{t-1}$. Because
$u_1, u_2, ...$ are IID uniform on $[0,1]$ and independent of the data, $u_t$ is independent of both,
so the conditional expectation factorizes and we obtain
\begin{align*}
    \mathbb{P}& \{P_t^{\tilde{w}^o} \leq \alpha_t, ..., P_1^{\tilde{w}^o} \leq \alpha_1 \mid E_{z}^{(t)}\} \\
    & = \sum_{\sigma(t)=1}^t\mathbb{P}\{P_t^{\tilde{w}^o} \leq \alpha_t \mid E_{z}^{(t)}, Z_t = z_{\sigma(t)}\}\cdot\mathbb{E}\big[\mathbbm{1}\{P_{t-1}^{\tilde{w}^o} \leq \alpha_{t-1}, ..., P_1^{\tilde{w}^o} \leq \alpha_1\} \mid E_{z}^{(t)}, Z_t = z_{\sigma(t)}\big]\cdot \mathbb{P}\{Z_t = z_{\sigma(t)} \mid E_{z}^{(t)}\}.
\end{align*}

Now, observe that conditioning on $E_{z}^{(t)}$ and $Z_t = z_{\sigma(t)}$ implies that $\{Z_1, ..., Z_{t-1}\}=\{z_1, ..., z_t\}\backslash \{z_{\sigma(t)}\}$. That is, observing the bag of $t$ data values (i.e., observing $E_{z}^{(t)}$) and the value taken on by the $t$-th random variable ($Z_t = z_{\sigma(t)}$) implies that each of $Z_1, ..., Z_{t-1}$ takes a value in $\{z_1, ..., z_t\}\backslash \{z_{\sigma(t)}\}$. 

Without loss of generality, we can relabel the indices on the data values as
$\{z_1, ..., z_{t-1}\}\leftarrow \{z_1, ..., z_{t}\}\backslash \{z_{\sigma(t)}\}$, and so we denote
this event $\{Z_1, ..., Z_{t-1}\}=\{z_1, ..., z_{t-1}\}$ as $E_{z}^{(t-1)}$. Each of
$P_{t-1}^{\tilde{w}^o}, ..., P_{1}^{\tilde{w}^o}$ is a function of the ordering in which these $t-1$
values were observed (together with $u_1, ..., u_{t-1}$), so to replace the conditioning on
$\{E_{z}^{(t)}, Z_t = z_{\sigma(t)}\}$ by conditioning on $E_{z}^{(t-1)}$ it suffices to show that
the conditional law of that ordering is the same under both. Letting $\tau$ denote a permutation of
$[t-1]$ (distinguished from the permutation $\sigma$ of $[t]$), the chain rule gives
\begin{align*}
    f(z_{\tau(1)}, ..., z_{\tau(t-1)}, z_{\sigma(t)}) = f_{Z_1, ..., Z_{t-1}}(z_{\tau(1)}, ..., z_{\tau(t-1)})\cdot f_{Z_t \mid Z_1, ..., Z_{t-1}}(z_{\sigma(t)}\mid z_{\tau(1)}, ..., z_{\tau(t-1)}),
\end{align*}
and by bag sufficiency \eqref{appeq:bag_sufficiency} the second factor depends on
$z_{\tau(1)}, ..., z_{\tau(t-1)}$ only through $E_{z}^{(t-1)}$, so it takes the same value for every
$\tau$ and cancels in the normalization:
\begin{align*}
    \mathbb{P}\{\tau \mid E_{z}^{(t)}, Z_t = z_{\sigma(t)}\}
    = \frac{f(z_{\tau(1)}, ..., z_{\tau(t-1)}, z_{\sigma(t)})}{\sum_{\tau'}f(z_{\tau'(1)}, ..., z_{\tau'(t-1)}, z_{\sigma(t)})}
    = \frac{f_{Z_1, ..., Z_{t-1}}(z_{\tau(1)}, ..., z_{\tau(t-1)})}{\sum_{\tau'}f_{Z_1, ..., Z_{t-1}}(z_{\tau'(1)}, ..., z_{\tau'(t-1)})}
    = \mathbb{P}\{\tau \mid E_{z}^{(t-1)}\}.
\end{align*}
That is, under bag sufficiency $E_{z}^{(t-1)}$ is a sufficient statistic for
$P_{t-1}^{\tilde{w}^o}, ..., P_{1}^{\tilde{w}^o}$, and we may substitute the conditioning on
$E_{z}^{(t)}$ and $Z_t=z_{\sigma(t)}$ with conditioning on $E_{z}^{(t-1)}$:

\begin{align*}
    \mathbb{P}& \{P_t^{\tilde{w}^o} \leq \alpha_t, ..., P_1^{\tilde{w}^o} \leq \alpha_1 \mid E_{z}^{(t)}\}\\
    & = \sum_{\sigma(t)=1}^t\mathbb{P}\{P_t^{\tilde{w}^o} \leq \alpha_t \mid E_{z}^{(t)}, Z_t = z_{\sigma(t)}\}\cdot\mathbb{P}\{P_{t-1}^{\tilde{w}^o} \leq \alpha_{t-1}, ..., P_1^{\tilde{w}^o} \leq \alpha_1 \mid E_{z}^{(t-1)}\}\cdot \mathbb{P}\{Z_t = z_{\sigma(t)} \mid E_{z}^{(t)}\}.
\end{align*}

Using the inductive assumption, then the law of total probability over $\sigma(t)$, and finally
Lemma \ref{app:lemma1}, this becomes
\begin{align*}
    \mathbb{P}\{P_t^{\tilde{w}^o} \leq \alpha_t, ..., P_1^{\tilde{w}^o} \leq \alpha_1 \mid E_{z}^{(t)}\}& = \sum_{\sigma(t)=1}^t\mathbb{P}\{P_t^{\tilde{w}^o}\leq \alpha_t \mid E_{z}^{(t)}, Z_t = z_{\sigma(t)}\}\cdot\alpha_{t-1}\cdots\alpha_{1}\cdot \mathbb{P}\{Z_t = z_{\sigma(t)} \mid E_{z}^{(t)}\}\nonumber \\
    & = \mathbb{P}\big\{P_t^{\tilde{w}^o}\leq \alpha_t \mid E_{z}^{(t)}\big\}\cdot\alpha_{t-1}\cdots\alpha_{1} \nonumber \\
    & = \alpha_{t}\cdots\alpha_{1},
\end{align*}
which proves Eq. \eqref{eq:conditional_iid_bernoulli}. Note that Eq. \eqref{eq:conditional_iid_bernoulli} is a conditional result for any $t=1, ..., T$; marginalizing over the event $E_{z}^{(t)}$ and taking $t=T$ implies
\begin{align}\label{eq:unconditional_iid_bernoulli}
    \mathbb{P}\{P_T^{\tilde{w}^o} \leq \alpha_T, ..., P_1^{\tilde{w}^o} \leq \alpha_1 \} = \alpha_T \cdots \alpha_1.
\end{align}
We have proven that $P_1^{\tilde{w}^o},P_2^{\tilde{w}^o}, ..., P_T^{\tilde{w}^o}\stackrel{iid}{\sim}\text{Unif}[0, 1]^T$; this implies the analogous result for the infinite sequence, that is, $P_1^{\tilde{w}^o},P_2^{\tilde{w}^o}, ...\stackrel{iid}{\sim}\text{Unif}[0, 1]^{\infty}$ \citep{shiryaev2016probability, vovk2003testing}.

Note that Eq. \eqref{eq:unconditional_iid_bernoulli} holds in theory for any joint PDF $f$ satisfying bag sufficiency, but it is an abstract statement because it relies on the oracle weights in Eq. \eqref{eq:general_weights_app}, which are intractable due to requiring knowledge of $f$ and factorial complexity. Thus, consider if the oracle weights in Eq. \eqref{eq:general_weights_app} are simplified using some assumptions or approximations on $f$ (e.g., conditional independence or invariance assumptions, density-ratio estimations) and the bag-sufficiency assumption. In particular, let 
$\mathcal{H}_0(\hat{f})$ denote the conjunction of these assumptions including bag sufficiency, and denote the resulting approximated weights as $\tilde{w}:=\tilde{w}(\mathcal{H}_0(\hat{f}))$, so that $\mathcal{H}_0(\hat{f})\implies \tilde{w} = \tilde{w}^o$.  Then, the $P_{t}^{\tilde{w}}$ are IID uniformly distributed, assuming $\mathcal{H}_0(\hat{f})$:
\begin{align}\label{eq:H0(hat(f))_iid_bernoulli_app}
    \mathbb{P}_{\mathcal{H}_0(\hat{f})}\big\{P_T^{\tilde{w}} \leq \alpha_T, ..., P_1^{\tilde{w}} \leq \alpha_1 \big\}= \alpha_T \cdots \alpha_1.
\end{align}

Eq. \eqref{eq:H0(hat(f))_iid_bernoulli_app} implies that $P_1^{\tilde{w}},P_2^{\tilde{w}}, ..., P_T^{\tilde{w}}\mid \mathcal{H}_0(\hat{f})\stackrel{iid}{\sim}\text{Unif}[0, 1]^T$. Similarly as before, this result  for a fixed number of points $T$ implies the corresponding result for infinite sequences \citep{shiryaev2016probability, vovk2003testing}, that is, $P_1^{\tilde{w}},P_2^{\tilde{w}},... \mid \mathcal{H}_0(\hat{f}) \stackrel{iid}{\sim}\text{Unif}[0, 1]^{\infty}$.

\subsection{Proofs for Proposition \ref{thm:wctm_anytime_valid_guarantee} (WCTM Anytime False-Alarm Control) and Proposition \ref{thm:SR-WCTM_ARL_control} (Average-Run-Length Control for WCTM-based Shiryaev-Roberts Procedure)}
\label{subsec:wctm_anytime_pf}

The proofs for Proposition \ref{thm:wctm_anytime_valid_guarantee} and Proposition \ref{thm:SR-WCTM_ARL_control} follow from Theorem \ref{thm:general_weighted_p_validity} and \citet{vovk2021testing}'s results for standard conformal test martingales. 
By Theorem \ref{thm:general_weighted_p_validity}, weighted conformal $p$-values have IID uniform distribution on [0,1] premised on some assumptions $\mathcal{H}_0(\hat{f})$. Thus, a sequence of weighted conformal $p$-values is a sequence of IID random variables drawn from \text{Unif}[0,1], assuming that the null hypothesis $\mathcal{H}_0(\hat{f})$ is true. Because the construction of a weighted conformal test martingale (e.g., Eq. \eqref{eq:ctm_wealth_def}) is only allowed to depend on the data via the sequence of weighted conformal $p$-values, we can draw from \citet{vovk2021testing}'s results that construct betting martingales that only assume a sequence of IID uniform [0,1] random variables. That is, by construction, weighted conformal test martingales are a betting martingale \citep{vovk2021testing} whose validity is premised on the assumptions $\mathcal{H}_0(\hat{f})$ (used to compute the weights). Then, Proposition \ref{thm:wctm_anytime_valid_guarantee} follows by Ville's inequality \citep{ville1939etude} and Proposition \ref{thm:SR-WCTM_ARL_control} follows from Proposition 4.2 in \citet{vovk2021testing}.

\subsection{Counterexample to Independence of Weighted Conformal $p$-Values for Arbitrary Distributions}
\label{app:counterexample}

Here we present a counterexample to online weighted conformal $p$-values satisfying independence for arbitrary joint data distributions. As throughout this paper, in a slight abuse of notation we will use set brackets $\{\}$ for both sets and bags/multisets (where the latter allows duplicate elements); also we denote $\{Z_1, Z_2, Z_3\}=\{Z_{1:3}\}$.

\begin{proposition}\label{prop:counterexample}
There exists a joint distribution $f$ on $\mathcal{Z}^3$ under which the computed weights coincide with the oracle
weights, $\tilde w = \tilde w^{o}$, and each of $P_1^{\tilde{w}},P_2^{\tilde{w}},P_3^{\tilde{w}}$ is exactly $\text{Unif}[0,1]$, but
$P_1^{\tilde{w}},P_2^{\tilde{w}},P_3^{\tilde{w}}$ are \emph{not} mutually independent. Therefore, bag sufficiency
\eqref{appeq:bag_sufficiency} cannot be dropped from $\mathcal{H}_0(\hat f)$: a generalization of Theorem~\ref{thm:general_weighted_p_validity} to arbitrary $f$ without further restriction (as stated in v1--v4 of this paper) is thus false, and the requirement is not an artifact of the
proof.
\end{proposition}

\begin{proof}
Let $\mathcal{Z}=\{0,1\}$ with $\hat S(z)=z$, so $V_t=Z_t$, and let $T=3$. Take
$Z_1,Z_2\stackrel{iid}{\sim}\text{Unif}\{0,1\}$ and
\begin{equation}\label{eq:ce-model}
  f_{3\mid 1:2}(1\mid z_1,z_2)\;=\;\mathbbm{1}\{(z_1,z_2)=(0,1)\}\;+\;\tfrac12\,\mathbbm{1}\{z_1=z_2\}.
\end{equation}
Since $f_{3\mid 1:2}(1\mid 0,1)=1\neq 0=f_{3\mid 1:2}(1\mid 1,0)$ while the two histories share the
bag $\{0,1\}$, bag sufficiency \eqref{appeq:bag_sufficiency} fails; as $(Z_1,Z_2)$ is
exchangeable, this is the only departure. The weights used below are the oracle weights
$\tilde w^{o}$, so the remaining content of $\mathcal{H}_0(\hat f)$, that the computed weights be
exact, $\tilde w = \tilde w^{o}$, holds by construction, and nothing that follows is an artifact
of weight estimation.

Six sequences have positive probability. At $t=2$ the bag $\{0,1\}$ receives equal mass from
$(0,1)$ and $(1,0)$, so each value carries oracle weight $\tfrac12$. At $t=3$ the bag
$\{0,0,1\}$ arises only from $(0,0,1)$ and $(1,0,0)$, of probabilities $\tfrac18$ and
$\tfrac14$ (the ordering $(0,1,0)$ has zero mass under \eqref{eq:ce-model}), so
\begin{equation}\label{eq:ce-weights}
  \mathbb{P}\{Z_3=1\mid \{Z_{1:3}\}=\{0,0,1\}\}=\tfrac13,
  \qquad
  \mathbb{P}\{Z_3=0\mid \{Z_{1:3}\}=\{0,1,1\}\}=\tfrac13 ,
\end{equation}
by symmetry. Substituting into $P_t=p^-_{i}+U_t\,(p^+_{i}-p^-_{i})$ gives Table~\ref{tab:ce}
($P_1^{\tilde{w}}=U_1$ always, the bag being a singleton).

\begin{table}[h]
\centering
\begin{tabular}{cccc}
\toprule
$(z_1,z_2,z_3)$ & $f_{1:3}$ & law of $P_2^{\tilde{w}}$ & law of $P_3^{\tilde{w}}$\\
\midrule
$(0,0,0)$ & $1/8$ & $\text{Unif}[0,1]$        & $\text{Unif}[0,1]$\\
$(0,0,1)$ & $1/8$ & $\text{Unif}[0,1]$        & $\text{Unif}[0,\tfrac13]$\\
$(0,1,1)$ & $1/4$ & $\text{Unif}[0,\tfrac12]$ & $\text{Unif}[0,\tfrac23]$\\
$(1,0,0)$ & $1/4$ & $\text{Unif}[\tfrac12,1]$ & $\text{Unif}[\tfrac13,1]$\\
$(1,1,0)$ & $1/8$ & $\text{Unif}[0,1]$        & $\text{Unif}[\tfrac23,1]$\\
$(1,1,1)$ & $1/8$ & $\text{Unif}[0,1]$        & $\text{Unif}[0,1]$\\
\bottomrule
\end{tabular}
\caption{The counterexample. Conditional on the sequence, $U_2$ and $U_3$ are independent.}
\label{tab:ce}
\end{table}

Each $P_t^{\tilde{w}}$ is exactly $\text{Unif}[0,1]$ by Lemma~\ref{app:lemma1}, which requires only
$\tilde w = \tilde w^{o}$ and no restriction on the serial dependence of $f$. For the joint law, read
the third column: $P_2^{\tilde{w}}<\tfrac12$ holds with probability $1$ under $(0,1,1)$ and probability $0$ under
$(1,0,0)$, and with probability $\tfrac12$ under each of the four homogeneous-bag rows. Conditioning
on $\{P_2^{\tilde{w}}<\tfrac12\}$ therefore leaves those four rows unchanged while doubling $(0,1,1)$ from
$\tfrac14$ to $\tfrac12$ and annihilating $(1,0,0)$. Averaging the fourth column against this
conditional law,
\begin{equation}\label{eq:ce-conclusion}
  \mathbb{P}\{P_3^{\tilde{w}}<\tfrac13 \mid P_2^{\tilde{w}}<\tfrac12\}=\tfrac{11}{24}\neq\tfrac13,
  \qquad
  \mathbb{E}[P_3^{\tilde{w}}\mid P_2^{\tilde{w}}<\tfrac12]=\tfrac{5}{12}\neq\tfrac12 .
\end{equation}
So $P_3^{\tilde{w}}\not\indep P_2^{\tilde{w}}$.\footnote{All quantities are exact rationals; the computation is reproduced by
the code accompanying this paper.}
\end{proof}

\begin{remark}[The wager this permits]
Because a betting strategy may depend on $P_1^{\tilde{w}},\dots,P_{t-1}^{\tilde{w}}$, the rule ``bet at the extreme value of
$\epsilon$ for which $h_\epsilon(p)=\tfrac32-p$ whenever $P_2^{\tilde{w}}<\tfrac12$, and abstain otherwise'' is
admissible. By \eqref{eq:ce-conclusion} it returns
$\mathbb{E}[h_\epsilon(P_3^{\tilde{w}})\mid P_2^{\tilde{w}}<\tfrac12]=\tfrac32-\tfrac5{12}=\tfrac{13}{12}$, an edge of
$8.3\%$ per round \emph{under the null}. Hence $\tilde M_t$ is a strict submartingale rather than a
martingale, Ville's inequality does not apply, and Proposition~\ref{thm:wctm_anytime_valid_guarantee} fails as well
without bag sufficiency.
\end{remark}

\begin{remark}[Why the $p$-value filtration does not help]
One might hope to retain the general statement by weakening the conclusion to
$P_t^{\tilde{w}} \mid \sigma(P_1^{\tilde{w}},\dots,P_{t-1}^{\tilde{w}}) \sim \text{Unif}[0,1]$ (where here $\sigma$ is a sigma algebra, not a permutation), which is all the betting construction requires.
Proposition~\ref{prop:counterexample} rules this out, and indeed such a conclusion holding for
every $t$ is equivalent to mutual independence, since the joint law telescopes. The reason is that
the coarsening discards nothing: with distinct scores the intervals $[p^-_i,p^+_i)$ partition
$[0,1)$ conditionally on $\{Z_{1:t}\}$, so $P_t$ identifies which element of the bag occupied slot
$t$. Observing $P_t$ is observing $Z_t$; when $f_{t\mid 1:t-1}$ depends on the order of the past,
$Z_t$ is informative about that order, and the earlier $p$-values encode precisely that order.
\end{remark}

\begin{remark}[Covariate-shift null satisfies bag sufficiency]
Under the covariate-shift null $H_0^{(\text{cs})}$ the observations are independent with
$f_{m\mid 1:m-1}$ determined by the index $m$ alone, so bag sufficiency
\eqref{appeq:bag_sufficiency} holds trivially and no such $f$ can arise. The construction above
necessarily lies outside the family monitored in this paper, and none of our experimental or
practical claims are affected. 
\end{remark}

\section{Deriving Specific WCTM Algorithms for Main Practical Testing Objective: Testing for \{Concept Shift or Unanticipated Covariate Shift\}}
\label{app:deriving_wctm_algos}

In this section we describe the derivation for the specific WCTM algorithms implemented in this paper for our primary testing objective, of detecting either concept shift (in $Y\mid X$) or unanticipated covariate shift (extreme shift in $X$). Because a WCTM is constructed on top of a corresponding weighted CP method by a betting process on the associated sequence of weighted CP $p$-values, this will initially parallel the procedure described in \citet{prinster2024conformal} for deriving weighted CP methods. The resulting weighted CP method (that the WCTM algorithms are built on top of) will be similar to that used in \citet{tibshirani2019conformal}, except whereas that paper focuses on CP for a shift between two batches of data, we focus on an online monitoring setting where a shift may occur at an unknown post-deployment time and density-ratio weights may be estimated online.

\textit{Step 1: List assumptions/null hypothesis of interest}

We begin by restating the null hypothesis (i.e., the nonparametric assumptions on what aspects of the data distribution are expected to stay invariant versus change) that is of primary practical interest in this paper. That is, as in Eq. \eqref{eq:cs_null} in the main paper, our null hypothesis $\mathcal{H}_0^{(\text{cs})}$ assumes that the conditional label distribution $Y\mid X$ remains invariant and that the marginal $X$ distribution can shift, but that at some post-deployment time $t_{\text{ad}}$ (the time at which to begin adaptation to covariate shift), the covariates $X$ shift in distribution in a manner described by the (estimated) density-ratio function $\widehat{w}(x)$. That is, we want to sequentially test the null hypothesis

\begin{align}\label{appeq:cs_null}
    \mathcal{H}_0^{\text{(cs)}}  :=\begin{cases}
        (X_{i}, Y_{i}) \qquad \ \stackrel{iid}{\sim} \mathbf{F_{Y\mid X}} \times \mathbf{F_{X}} , \ & i \in [n + t_{\text{ad}}-1] \\
        (X_{n+t}, Y_{n+t}) \stackrel{iid}{\sim} \mathbf{F_{Y\mid X}} \times \mathbf{G_{X}^{(\widehat{w})}}, & t\geq t_{\text{ad}},
    \end{cases}
\end{align}

where boldface is used to denote a set of distributions, and in particular $\mathbf{G_{X}^{(\widehat{w})}}$ denotes the set of distributions  for $X$ with density-ratio function $\widehat{w}(x)$ (with respect to the true but unknown source covariate distribution $F_{X}\in \mathbf{F_{X}}$):
\begin{align*}
    \mathbf{G_{X}^{(\widehat{w})}} := \Big\{G_{X} \ : \ \text{d}G_{X}(x)/\text{d}F_{X}(x) = \widehat{w}(x) \quad \text{if} \quad \text{d}F_{X}(x)>0\Big\}.
\end{align*}

The density-ratio function $\widehat{w}(x)$ can be estimated with offline data and/or continually updated as each test point is observed.

\textit{Step 2: Factorize joint density using null hypothesis assumptions} 

We now use our null hypothesis assumptions $\mathcal{H}_0^{\text{(cs)}}$ to factorize the joint probability density function. See \citet{prinster2024conformal}'s Appendix B.1 for a more detailed derivation for a more general setting.

\begin{align*}
    f(z_1, ..., z_{n+t}) & = f_{X}(x_1) \cdots f_{X}(x_{n+t_{\text{ad}}-1})\cdot g_{X}(x_{n+t_{\text{ad}}}) \cdots g_{X}(x_{n+t})\cdot\prod_{j=1}^{n+t}\Big[f_{Y\mid X}(y_j| x_{j} )\Big]  \nonumber \\
    & = \prod_{j=1}^{n+t_{\text{ad}}-1}\Big[f_{X}(x_j)\Big] \cdot \prod_{j=n+t_{\text{ad}}}^{n+t}\Big[g_{X}(x_j)\Big]\cdot\prod_{j=1}^{n+t}\Big[f_{Y\mid X}(y_j| x_{j} )\Big],
\end{align*}

multiplying by $1=\frac{\prod_{j=n+t_{\text{ad}}}^{n+t}f_{X}(x_j)}{\prod_{j=n+t_{\text{ad}}}^{n+t}f_{X}(x_j)}$, we have

\begin{align*}
    f(z_1, ..., z_{n+t}) 
    & = \prod_{j=1}^{n+t}\Big[f_{X}(x_j)\Big] \cdot \prod_{j=n+t_{\text{ad}}}^{n+t}\bigg[\frac{g_{X}(x_j)}{f_{X}(x_j)}\bigg]\cdot\prod_{j=1}^{n+t}\Big[f_{Y\mid X}(y_j| x_{j} )\Big],
\end{align*}
and approximating the density-ratio function $\frac{g_{X}(x_j)}{f_{X}(x_j)}$ with our estimate $\widehat{w}(x_j)$, this becomes 
\begin{align}\label{eq:factorized_f_cs}
    f(z_1, ..., z_{n+t}) 
    & = \prod_{j=n+t_{\text{ad}}}^{n+t}\widehat{w}(x_j) \cdot\prod_{j=1}^{n+t}\Big[f_{Y\mid X}(y_j| x_{j} )\cdot f_{X}(x_j)\Big] \nonumber \\
    & = \underbrace{\prod_{j=n+t_{\text{ad}}}^{n+t}\widehat{w}(x_j)}_{\text{Time-dependent factors}}\ \cdot \ \underbrace{\prod_{j=1}^{n+t}f_{Z}(z_j)}_{\text{Time-invariant factor}}.
\end{align}

\textit{Step 3: Compute or estimate weights and weighted-conformal $p$-values}

Recalling the definition of the CP weights from the main paper (Eq. \eqref{eq:general_weights}) and plugging in the factorization from Eq. \eqref{eq:factorized_f_cs}:

\begin{align} \label{appeq:mfcs_weights_exact}
    \mathbb{P}\{Z_{n+t} = z_i \mid E_z^{(t)}\} & = \frac{\sum_{\sigma:\sigma(n+t)=i}f(z_{\sigma(1)}, ..., z_{\sigma(n+t)})}{\sum_{\sigma}f(z_{\sigma(1)}, ..., z_{\sigma(n+t)})} \nonumber \\
    & = \frac{\sum_{\sigma:\sigma(n+t)=i}\prod_{j=n+t_{\text{ad}}}^{n+t}\widehat{w}(x_{\sigma(j)})\ \cdot \ \prod_{j=1}^{n+t}f_{Z}(z_{\sigma(j)})}{\sum_{\sigma}\prod_{j=n+t_{\text{ad}}}^{n+t}\widehat{w}(x_{\sigma(j)})\ \cdot \ \prod_{j=1}^{n+t}f_{Z}(z_{\sigma(j)})} \nonumber 
    \\
    & = \frac{\sum_{\sigma:\sigma(n+t)=i}\prod_{j=n+t_{\text{ad}}}^{n+t}\widehat{w}(x_{\sigma(j)})}{\sum_{\sigma}\prod_{j=n+t_{\text{ad}}}^{n+t}\widehat{w}(x_{\sigma(j)})},
\end{align}

where the last line follows because the factor $\prod_{j=1}^{n+t}f_{Z}(z_{\sigma(j)})$ is invariant to permutations and thus cancels in the ratio. However, the computational complexity for computing Eq. \eqref{appeq:mfcs_weights_exact} is still $n+t$ choose $t-t_{\text{ad}}+1$ (i.e., requiring $\frac{(n+t)!}{(t-t_{\text{ad}}+1)!\,(n+t_{\text{ad}}-1)!}$ computations), so to enable tractable estimation, we focus only on permuting the test point value with values assumed to be from the source distribution (prior to the adaptation beginning); that is, we restrict the sums to permutations $\sigma':[n+t]\rightarrow [n+t]$ such that $\sigma'(j)=j$ for all $j\in \{n+t_{\text{ad}}, ..., n+t-1\}$: 

\begin{align} \label{appeq:restricted_cs_weights}
    \mathbb{P}\{Z_{n+t} = z_i \mid E_z^{(t)}\} \approx \tilde{w}_i^{(t)} := & \frac{\sum_{\sigma':\sigma'(n+t)=i}\widehat{w}(x_{\sigma'(n+t)})}{\sum_{\sigma'}\widehat{w}(x_{\sigma'(n+t)})} \nonumber \\
    = & \frac{\widehat{w}(x_{i})}{\sum_{j\in [n+t_{\text{ad}}-1]\cup\{n+t\}}\widehat{w}(x_{j})}.
\end{align}

Then, the corresponding weighted-conformal $p$-value is 

\begin{align}\label{appeq:weighted_p_vals_fixed_cal}
    p_{n+t}^{\tilde{w}} = \sum_{i\in[n+t_{\text{ad}}-1]\cup \{n+t\}}\tilde{w}_i^{(t)}\Big[\mathbbm{1}\{v_i > v_{n+t}\} + u_{n+t}\mathbbm{1}\{v_i = v_{n+t}\}\Big]. 
\end{align}

Due to the modification made in Eq. \eqref{appeq:restricted_cs_weights}, Theorem \ref{thm:general_weighted_p_validity} only directly applies for the covariate shift null hypothesis (i.e., for testing $\mathcal{H}_0(\hat{f})=\mathcal{H}_0^{(\text{cs})}$ as in Eq. \eqref{appeq:cs_null}) for the timepoint $t=t_{\text{ad}}$. 
This is due to that, when a weighted-conformal $p$-value is not computed fully online,\footnote{That is, not computed relative to all the full observed data sequence with indices $[n+t]$, but only relative to a strict subset of $[n+t]$; for example, Eq. \eqref{appeq:weighted_p_vals_fixed_cal} computes the $p$-value relative to the data with indices $[n+t_{\text{ad}}-1]\cup \{n+t\}$, where $t_{\text{ad}}<t$.} the proof-by-induction approach in Section \ref{subsec:p_val_validity_pf} may not apply; in other words, for two test points $n+t$ and $n+t'$, their $p$-values $p_{n+t}^{\tilde{w}}$ and $p_{n+t'}^{\tilde{w}}$ may not be statistically independent due to their common dependency on the calibration data $[n+t_{\text{ad}}-1]$ in Eq. \eqref{appeq:weighted_p_vals_fixed_cal}. This gap can be resolved to enable testing of Eq. \eqref{appeq:cs_null} if the calibration data $[n+t_{\text{ad}}-1]$ are resampled for each test point, which restores independence trivially (due to each calibration set being entirely new). 
Alternatively, we can modify Eq. \eqref{appeq:cs_null} to also assume that we have enough initial source data (i.e., $n$ is large enough) so that the calibration scores' empirical CDF is equal to the true source score distribution, that is, $\widehat{F}_{V}^{[n+t_{\text{ad}}-1]} = F_{V}$, which is of course true in the limit $n\rightarrow \infty$. Accordingly, the $p$-value in Eq. \eqref{appeq:weighted_p_vals_fixed_cal} is effectively testing the modified null hypothesis of $\mathcal{H}_0^{\text{(cs)}}$ \textit{and} $\widehat{F}_{V}^{[n+t_{\text{ad}}-1]} = F_{V}$:

\begin{align}\label{appeq:cs_null_modified}
    \mathcal{H}_0^{\text{(cs)}} \land (\widehat{F}_{V}^{[n+t_{\text{ad}}-1]} = F_{V}) =\begin{cases}
        (X_{i}, Y_{i}) \qquad \ \stackrel{iid}{\sim} \mathbf{F_{Y\mid X}} \times \mathbf{F_{X}} , \ & i \in [n + t_{\text{ad}}-1] \\
        (X_{n+t}, Y_{n+t}) \stackrel{iid}{\sim} \mathbf{F_{Y\mid X}} \times \mathbf{G_{X}^{(\widehat{w})}}, & t\geq t_{\text{ad}} \\
        \widehat{F}_{V}^{[n+t_{\text{ad}}-1]} = F_{V}.
    \end{cases}
\end{align}
In practice, Eq. \eqref{appeq:cs_null_modified} can roughly be thought of as also testing whether the weighted CP calibration set is ``large enough'' for CP coverage to hold precisely \textit{conditional on} the calibration data. That is, letting $Z_{\text{cal}}$ denote the event that $\{Z_1, ..., Z_{n+t_{\text{ad}}-1}\}=\{z_1, ..., z_{n+t_{\text{ad}}-1}\}$, taking Eq. \eqref{appeq:cs_null_modified} as given would imply that, for all $\alpha\in (0, 1)$,
\begin{align}\label{appeq:cal_cond_cov}
    \mathbb{P}\{P_{n+t}^{\tilde{w}}\leq \alpha \mid Z_{\text{cal}}\} = \alpha \quad \iff \quad \mathbb{P}\big\{Y_{n+t} \in \widehat{C}_{[n+t_{\text{ad}}-1], \alpha}(X_{n+t}) \mid Z_{\text{cal}}\big\} = 1 - \alpha,
\end{align}
where $P_{n+t}^{\tilde{w}}$ is the random variable taking values in Eq. \eqref{appeq:weighted_p_vals_fixed_cal} and the probability is only over the draw of the test point $Z_{n+t}$. In other words,  Eq. \eqref{appeq:cs_null_modified} $\implies$ Eq. \eqref{appeq:cal_cond_cov}, so via the contrapositive, violations of calibration-set-conditional coverage (Eq. \eqref{appeq:cal_cond_cov}) would imply violations of $\mathcal{H}_0^{\text{(cs)}} \land (\widehat{F}_{V}^{[n+t_{\text{ad}}-1]} = F_{V})$ (Eq. \eqref{appeq:cs_null_modified}).

\textit{Step 4: Construct WCTM by a betting process on the weighted-conformal $p$-values}

Lastly, the practical WCTM algorithms implemented in this paper are constructed on top of a sequence of weighted-conformal $p$-values $p_{n+1}^{\tilde{w}}, p_{n+2}^{\tilde{w}}, ..., p_{n+t}^{\tilde{w}}, ...$, as defined in Eq. \eqref{appeq:weighted_p_vals_fixed_cal}. The resulting WCTM thus achieves the anytime-valid false alarm guarantee in Proposition \ref{thm:wctm_anytime_valid_guarantee} for the null hypothesis $\mathcal{H}_0(\hat{f}) = \mathcal{H}_0^{\text{(cs)}} \land (\widehat{F}_{V}^{[n+t_{\text{ad}}-1]} = F_{V})$, as defined in Eq. \eqref{appeq:cs_null_modified} (and analogously for the scheduled, Shiryaev-Roberts monitoring procedure based on WCTMs).

\clearpage

\section{Additional Experimental Results}\label{sec:additional}

\subsection{Additional Datasets for Adaptation to Benign Covariate Shifts}
\label{subsec:additional_adaptation_expts}

\begin{figure*}[!htb]
\centering
    \scriptsize{\textbf{Medical Expenditure Panel Survey (MEPS)}}
    \\
    \begin{subfigure}{}
        \includegraphics[width=0.18\textwidth]{fig/2025-05-03_coverage_meps-NN-covariate-shift_bias5.0-label_shiftNone-err_win50-cs_typeabs-nseeds200-Wfixed_cal_dyn_none.pdf}
    \end{subfigure}
    \hfill
    \begin{subfigure}{}
        \includegraphics[width=0.18\textwidth]{fig/2025-05-03_widths_meps-NN-covariate-shift_bias5.0-label_shiftNone-err_win50-cs_typeabs-nseeds200-Wfixed_cal_dyn_none.pdf}
    \end{subfigure}
    \hfill
    \begin{subfigure}{}
        \includegraphics[width=0.18\textwidth]{fig/2025-05-03_p_vals_hist_meps-NN-covariate-shift_bias5.0-label_shiftNone-err_win50-cs_typeabs-nseeds200-Wfixed_cal_dyn_none.pdf}
    \end{subfigure}
    \hfill
    \begin{subfigure}{}
        \includegraphics[width=0.18\textwidth]{fig/2025-05-03_Martingale_meps-NN-covariate-shift_bias5.0-label_shiftNone-err_win50-cs_typeabs-nseeds200-Wfixed_cal_dyn_none.pdf}
    \end{subfigure}
    \hfill
    \begin{subfigure}{}
        \includegraphics[width=0.18\textwidth]{fig/2025-05-03_Shiryaev-Roberts_meps-NN-covariate-shift_bias5.0-label_shiftNone-err_win50-cs_typeabs-nseeds200-Wfixed_cal_dyn_none.pdf}
    \end{subfigure}
    \\
    \scriptsize{\textbf{Superconductivity}}
    \\
    \begin{subfigure}{}
        \includegraphics[width=0.18\textwidth]{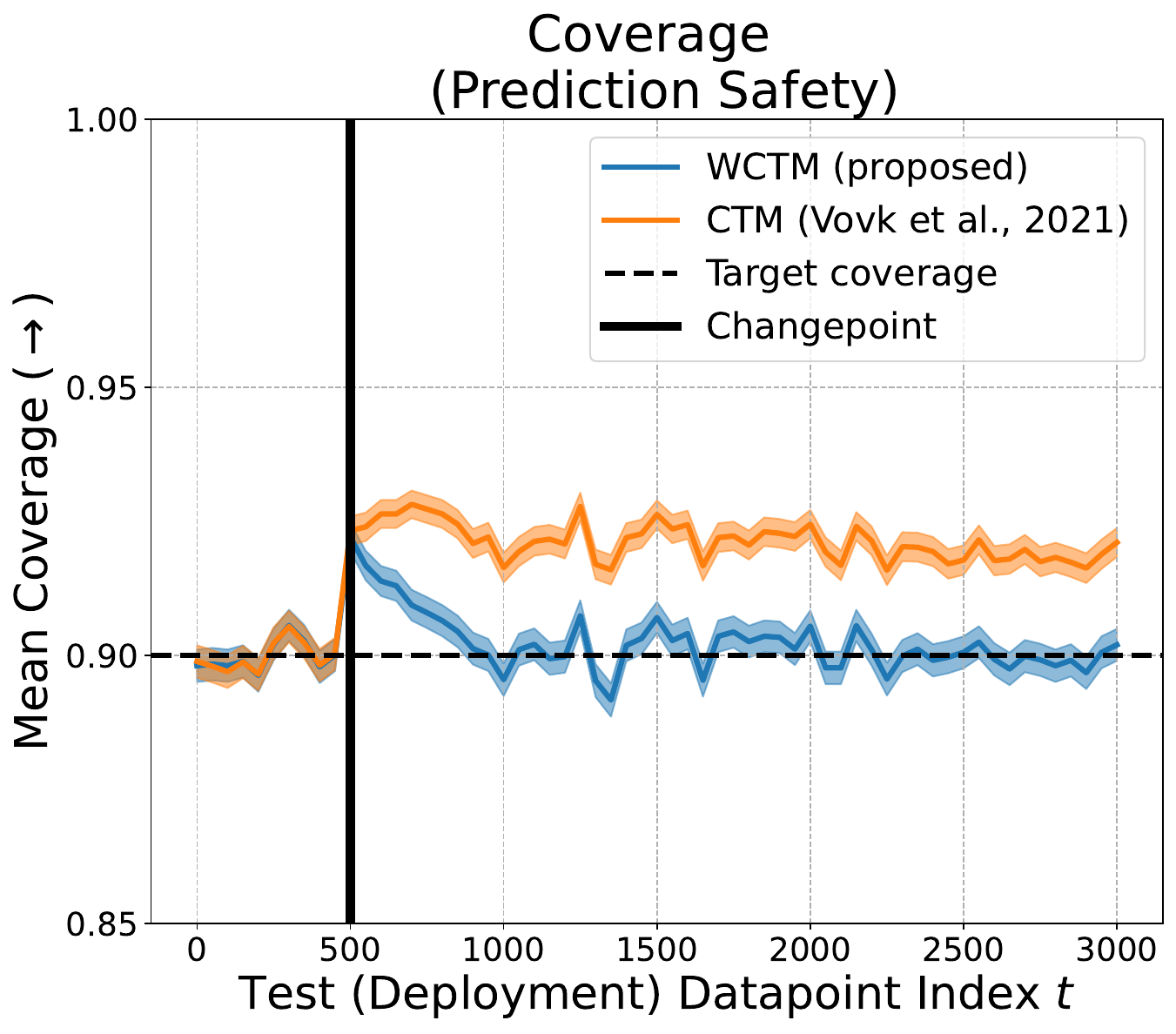}
    \end{subfigure}
    \hfill
    \begin{subfigure}{}
        \includegraphics[width=0.18\textwidth]{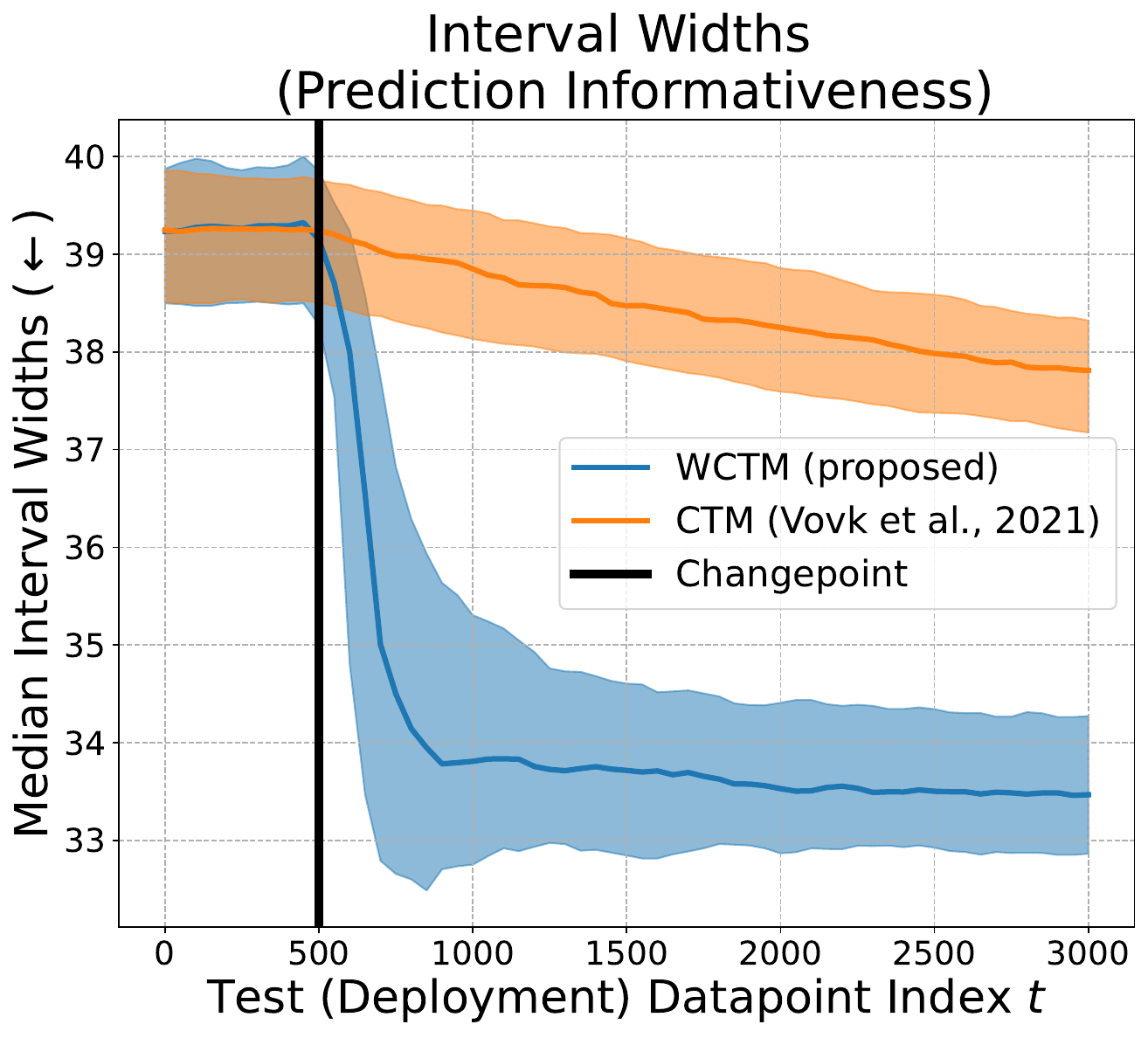}
    \end{subfigure}
    \hfill
    \begin{subfigure}{}
        \includegraphics[width=0.18\textwidth]{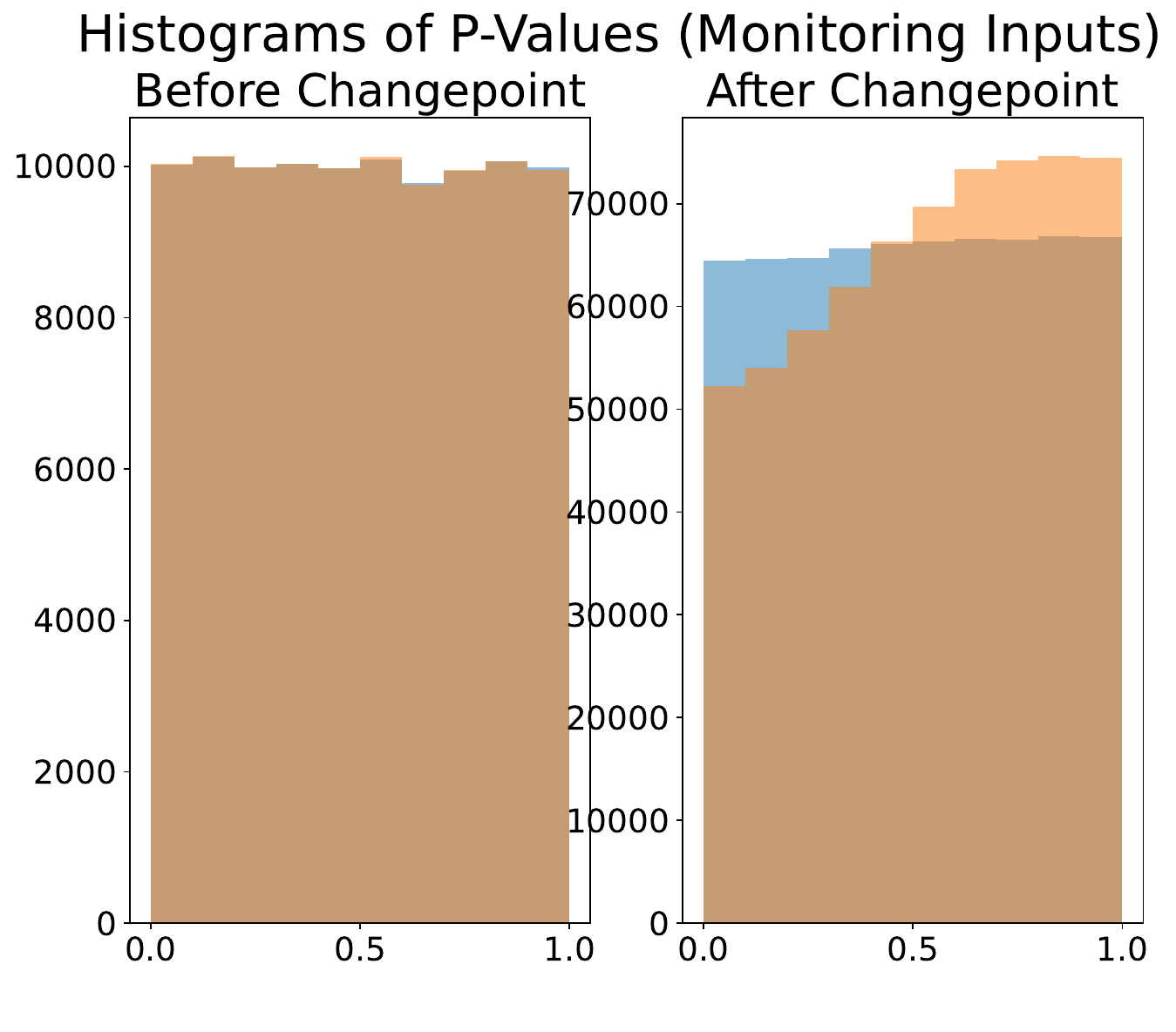}
    \end{subfigure}
    \hfill
    \begin{subfigure}{}
        \includegraphics[width=0.18\textwidth]{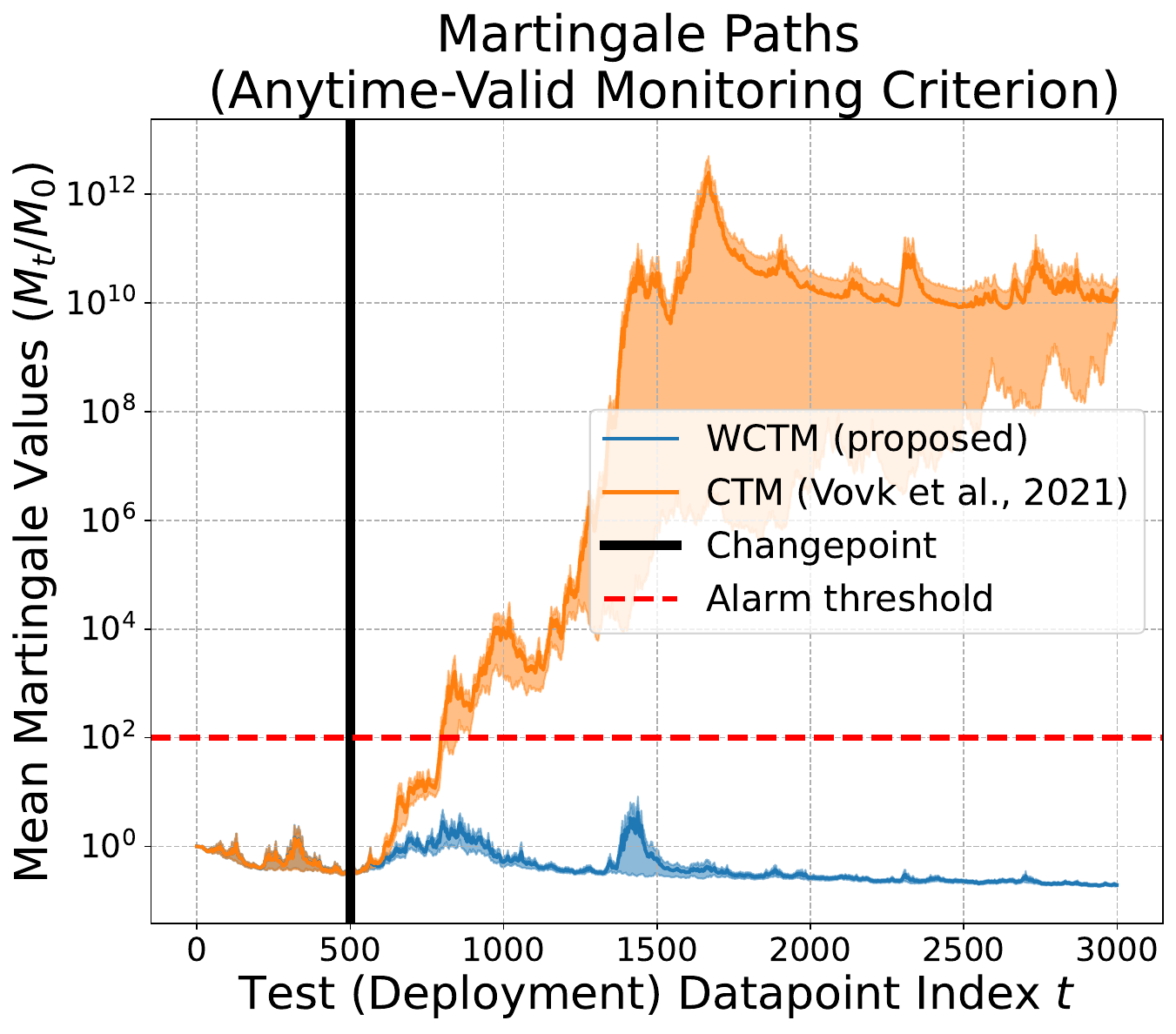}
    \end{subfigure}
    \hfill
    \begin{subfigure}{}
        \includegraphics[width=0.18\textwidth]{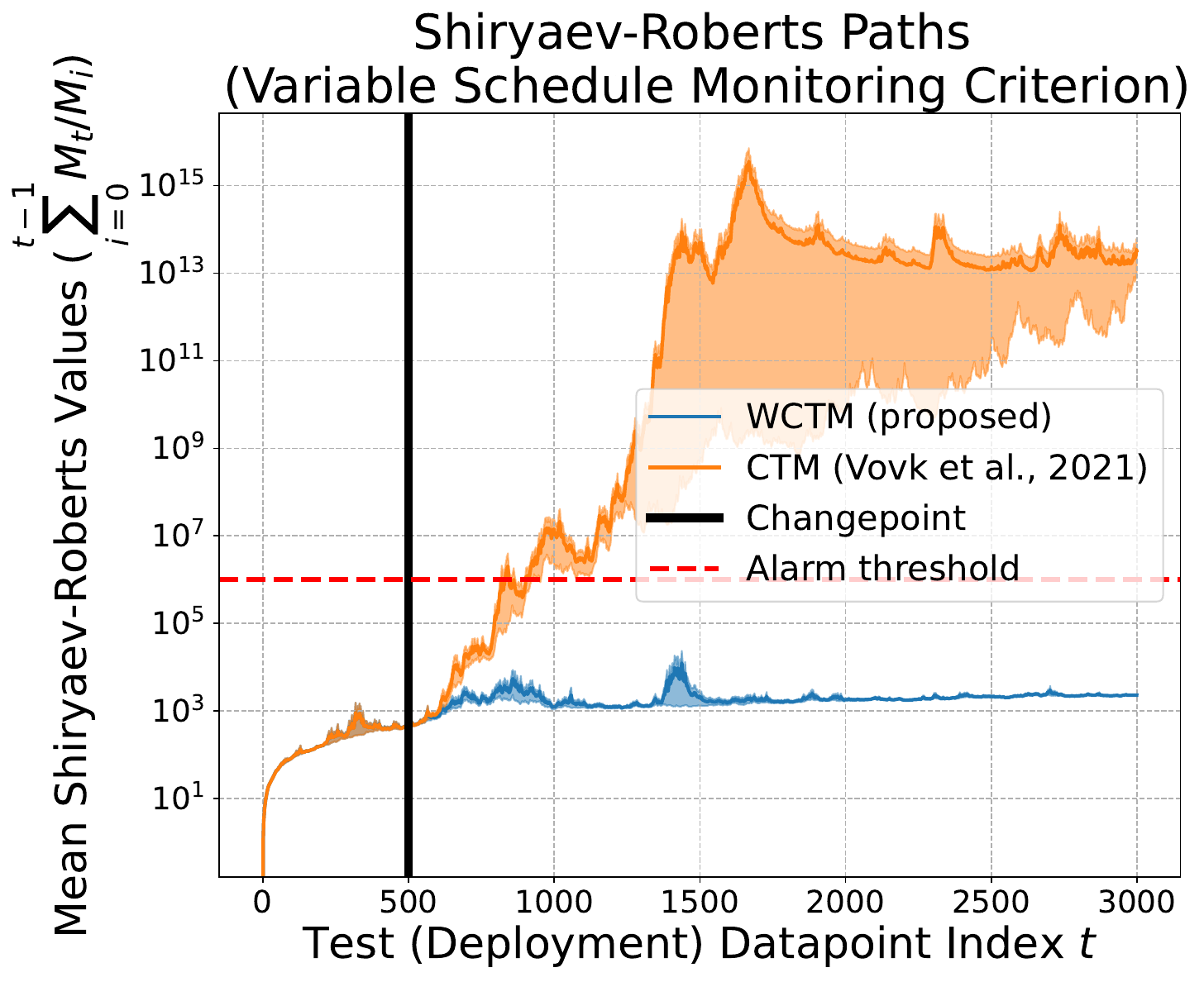}
    \end{subfigure}
    \\
    \scriptsize{\textbf{Bike Sharing}}
    \\
    \begin{subfigure}{}
        \includegraphics[width=0.18\textwidth]{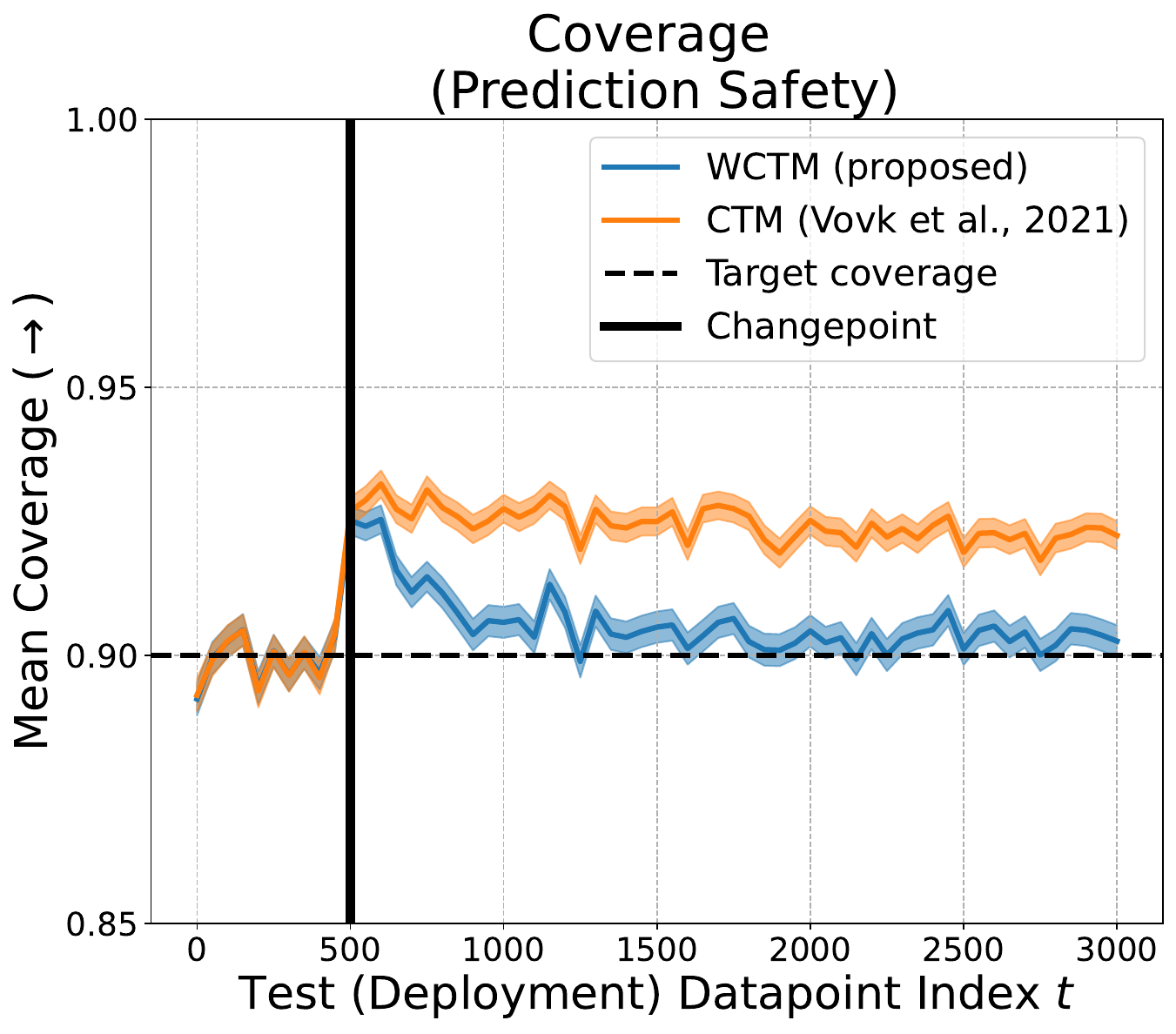}
    \end{subfigure}
    \hfill
    \begin{subfigure}{}
        \includegraphics[width=0.18\textwidth]{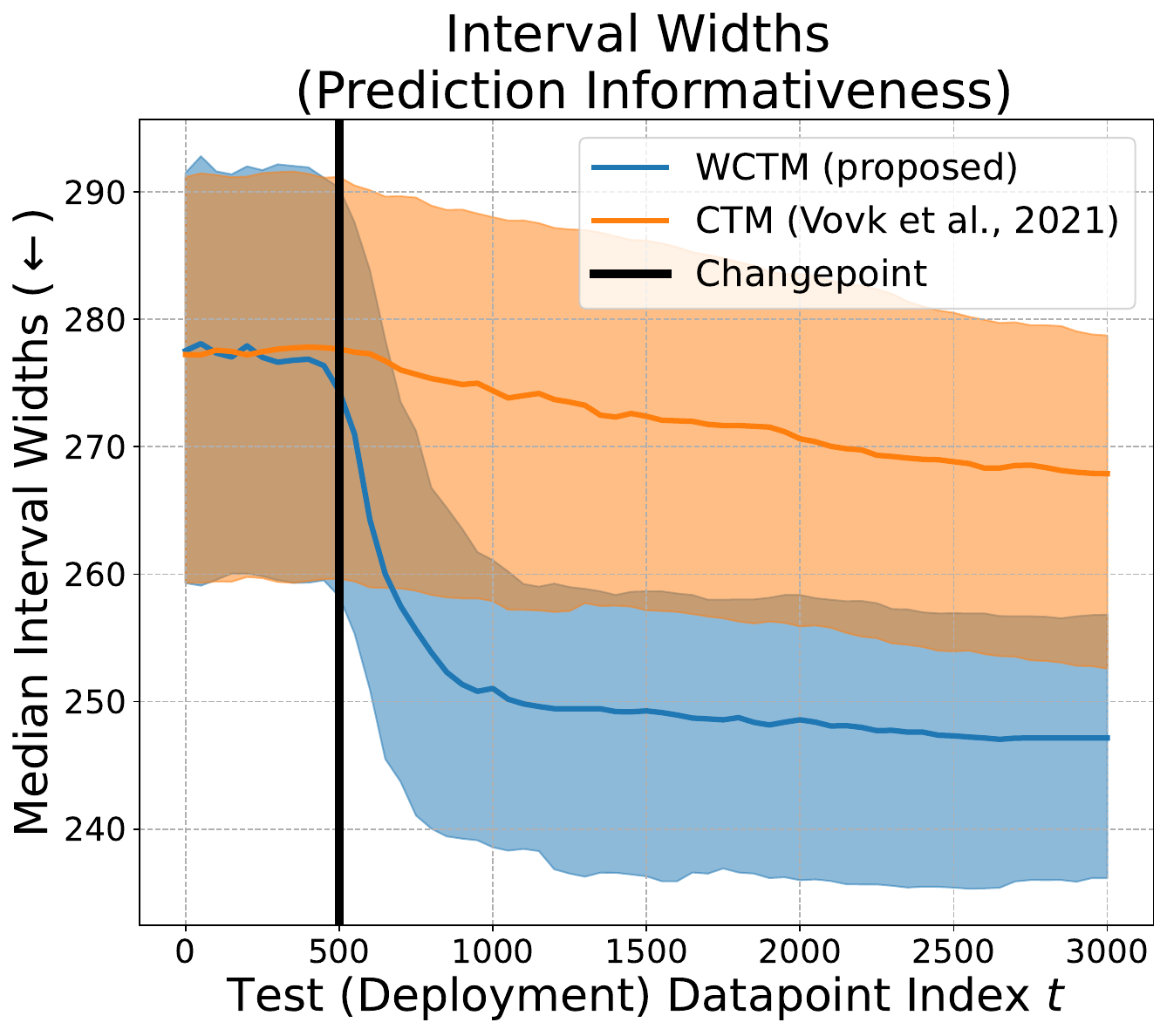}
    \end{subfigure}
    \hfill
    \begin{subfigure}{}
        \includegraphics[width=0.18\textwidth]{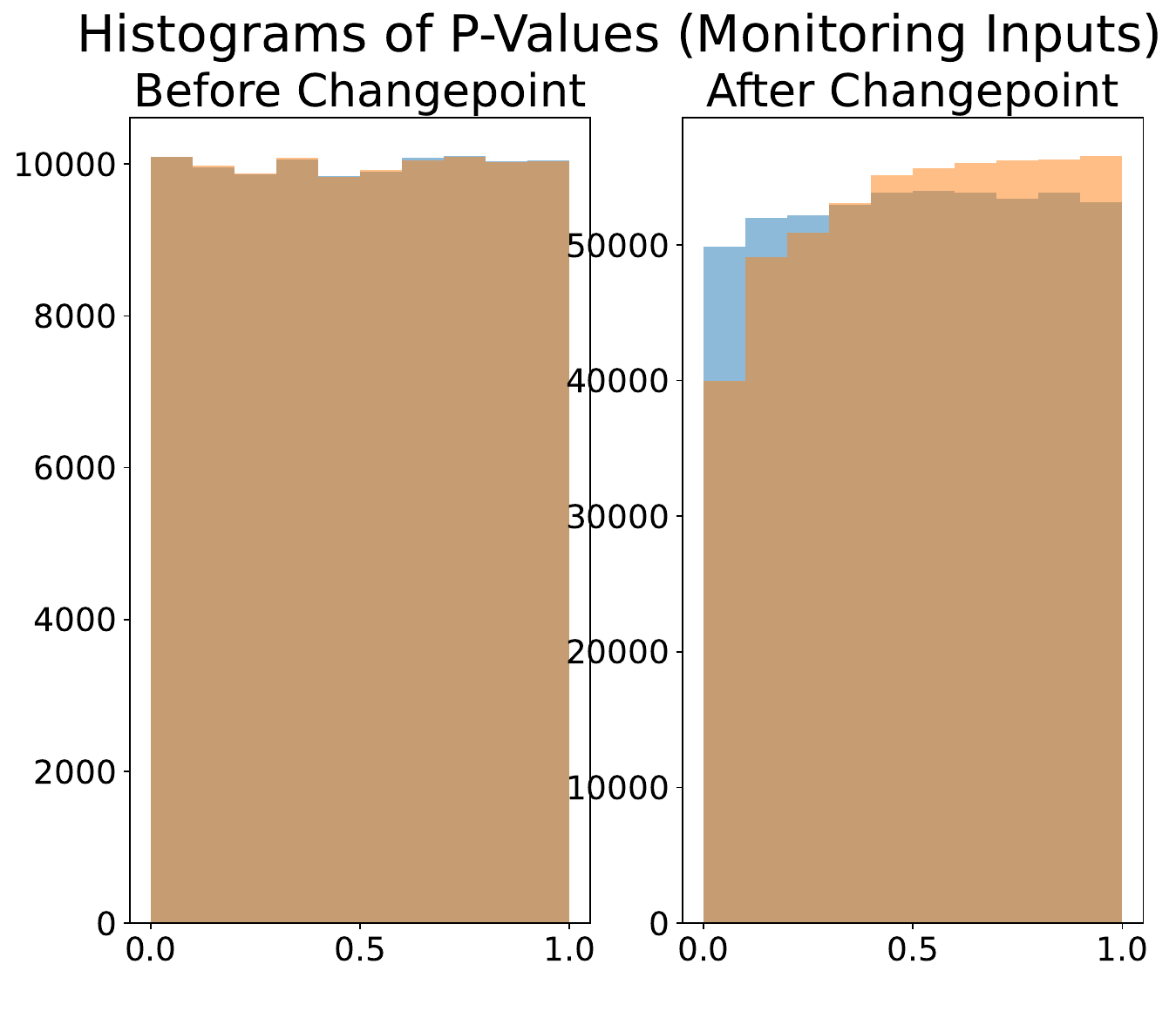}
    \end{subfigure}
    \hfill
    \begin{subfigure}{}
        \includegraphics[width=0.18\textwidth]{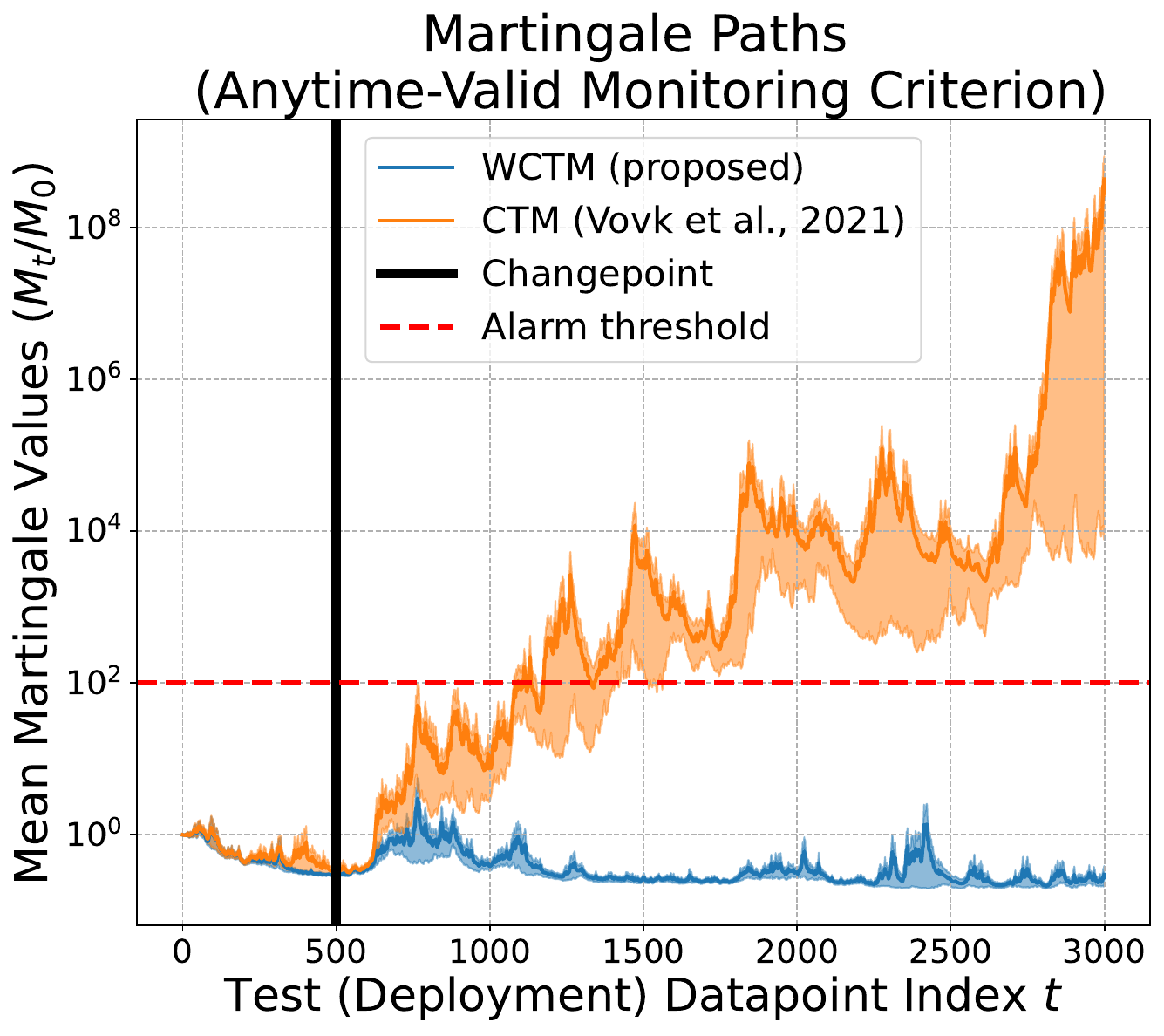}
    \end{subfigure}
    \hfill
    \begin{subfigure}{}
        \includegraphics[width=0.18\textwidth]{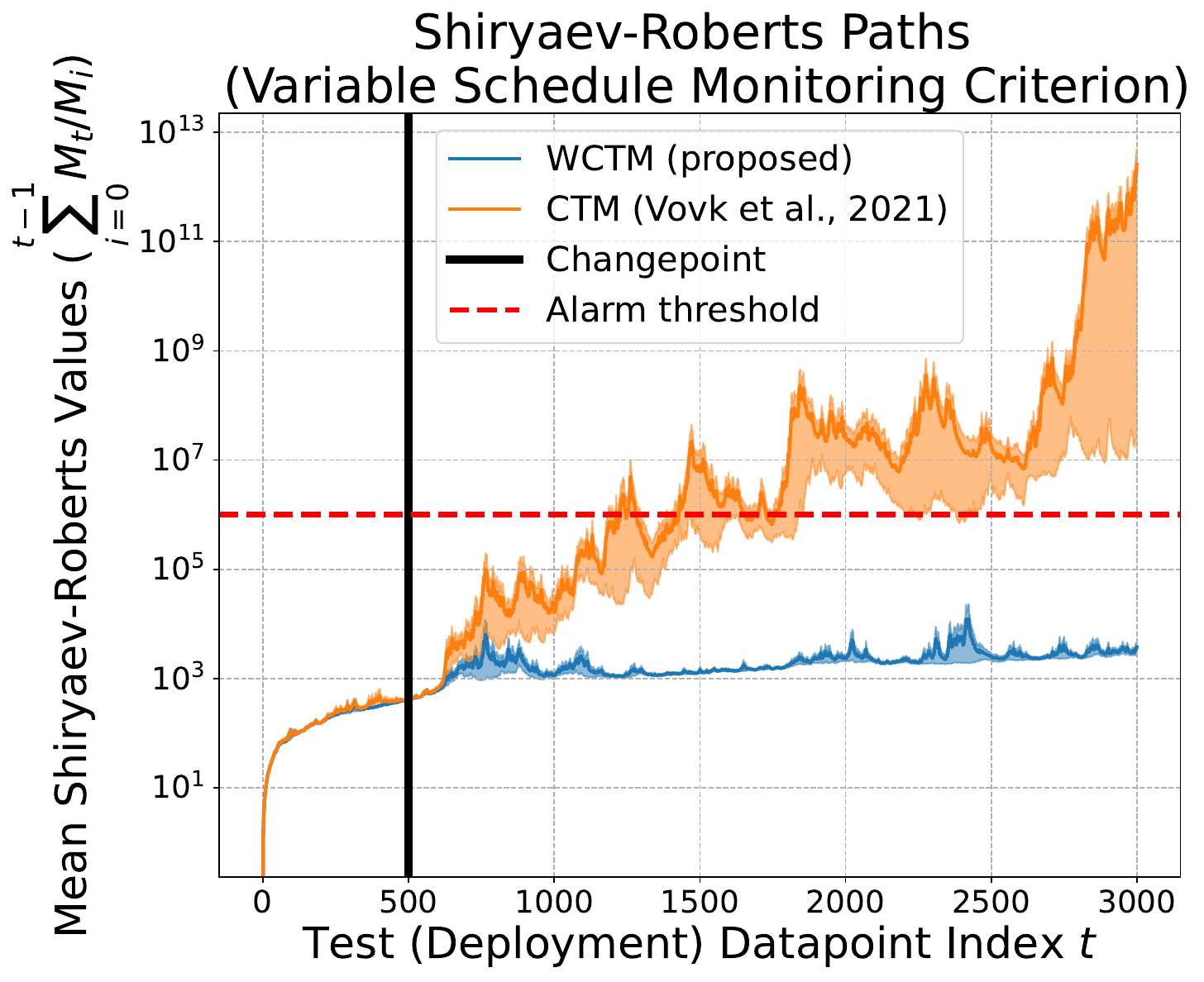}
    \end{subfigure}
    \\
    \vspace{-1em}
\caption{\small Tabular data results for the benign covariate shift setting to evaluate the adaptation ability of proposed WCTM methods (blue); all values are averaged over 200 random seeds. Training and calibration sets were sampled uniformly at random (with 1/3 of the total data used for training and calibration each), while post-changepoint test-set datapoints were bias-sampled from the remaining holdout data with probability proportional to $\exp(\lambda \cdot h(x))$. The shift-magnitude scalar $\lambda$ for each dataset was set as $\lambda_{\text{mep}}=5.0$
, $\lambda_{\text{sup}}=2.5$, and $\lambda_{\text{bik}}=5.0$. The function $h$ was selected to simuluate realistic shifts as described in the main text. Error regions represent standard errors for coverage, martingale paths, and Shiryaev-Roberts paths, and interquartile range for interval widths. Whereas standard CTMs (orange) raise unnecessary alarms with their anytime-valid and scheduled monitoring criteria, WCTMs avoid doing so by adapting online to the shift. That is, WCTMs maintain target coverage, adapt by increasing interval sharpness, and avoid unneeded alarms.}
\label{appfig:unnecessary_alarms}
\end{figure*}

\subsection{What Makes a Covariate Shift Mild, Moderate, or Extreme?}
\label{subsec:mild_mod_ex}

While the distinction between mild, moderate, and extreme covariate shifts can be gradual, problem-specific, and sometimes user-driven, these distinctions are not arbitrary, and WCTMs arguably even provide an approach to such delineation with statistical guarantees. Intuitively, benign shifts can be considered those where the ``safety'' of the CP coverage is maintained nontrivially (i.e., without the prediction set covering the whole label space). This intuition corresponds to the WCTMs' null hypothesis (and, harmful shifts violate it), as follows:
\begin{itemize}
    \item \textbf{Benign:} The betting martingale's null hypothesis is that the WCP $p$-values are IID $\text{Unif}[0, 1]$ (Appendix \ref{subsec:wctm_anytime_pf});  this null implies that coverage is satisfied (exactly) for all $\alpha\in(0,1)$ (i.e., the martingale’s null $\implies$ intuitive definition of “benign” regarding coverage validity).
    \item \textbf{Harmful:} (Contrapositive of the above.) If coverage is \textit{not} satisfied (exactly) for some $\alpha\in (0, 1)$, then the $p_{t}^{\tilde{w}}$ are not IID $\text{Unif}[0, 1]$ (ie, violation of coverage validity $\implies$ violation of martingale’s null, thus possibility for detection). Larger violations are easier to detect, and thus more likely to quickly raise an alarm. Note that this can be due to under-coverage (safety violation) or over-coverage (uninformative prediction sets); we further penalize trivial overcoverage---i.e., when $\widehat{C}(X_{n+t})=\mathcal{Y}$---by using anticonservative WCP $p$-values whenever this occurs. (See pseudocode in Appendix \ref{sec:alg}.)
    \item \textbf{Moderate:} A shift may initially be ``harmful'' as described above, due to density-ratio estimator having insufficient data, but later become ``benign'' once enough data has been collected.
\end{itemize}

Figure \ref{fig_mild_ex} provides an ablation study illustrating synthetic data example of WATCH performance for different magnitudes of covariate shift (in the input $X$ distribution). Each row corresponds to a specific magnitude of covariate shift and illustrates WATCH's response regarding coverage (prediction safety), interval widths (prediction informativeness), and WCTMs (monitoring criteria for alarms). The post-changepoint test points are sampled from the full source distribution with probabilities proportional to $exp(|x-18|*\lambda)$; larger values of $\lambda$ thus correspond to more severe covariate shift toward extreme (and particularly toward large) values of the input $X$. Experiments are averaged over 20 seeds.

\begin{figure*}[!ht]
\centering
    \includegraphics[width=0.9\textwidth]{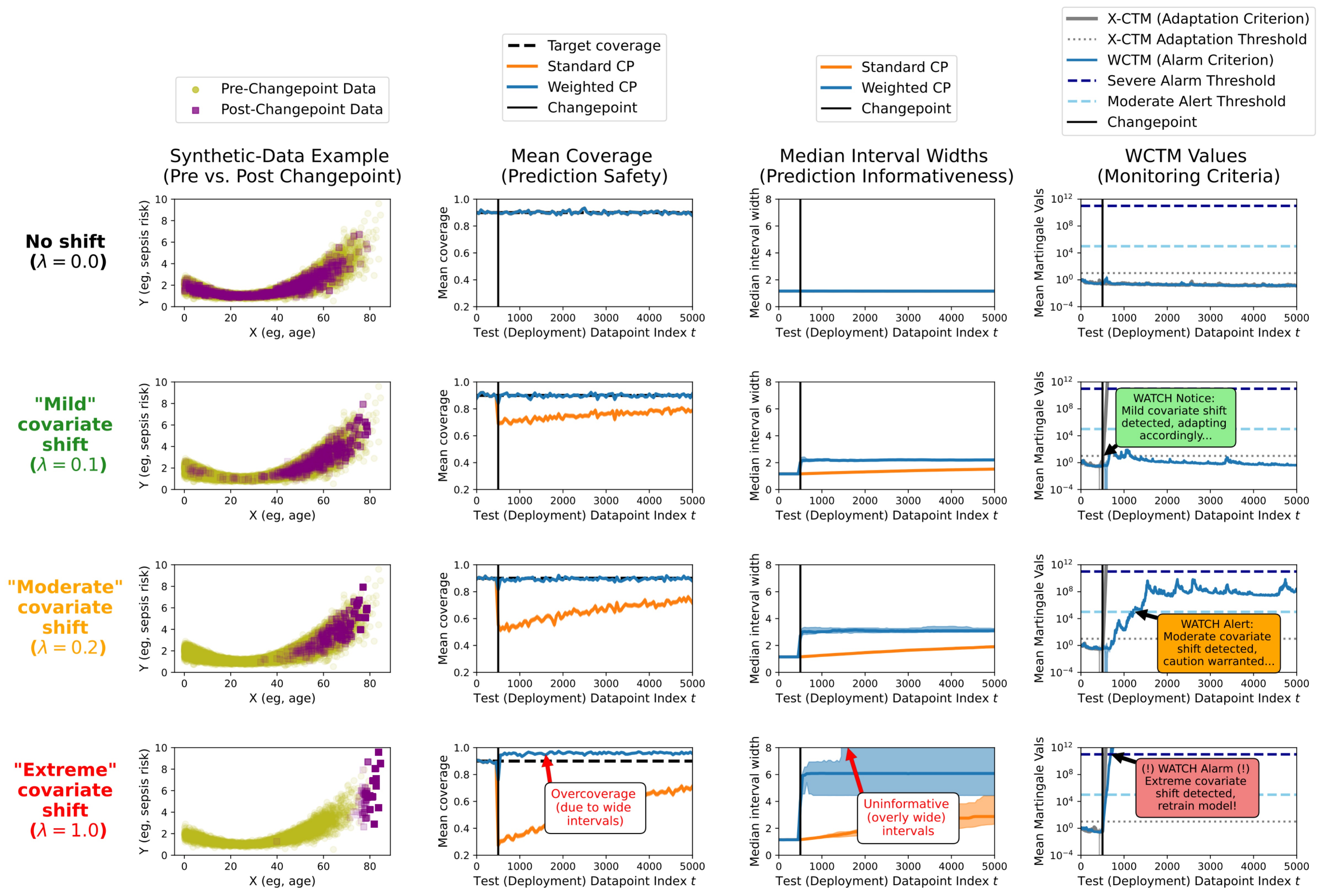}
\caption{Ablation study illustrating synthetic data example of WATCH performance for different magnitudes of covariate shift (in the input $X$ distribution).}
\label{fig_mild_ex}
\end{figure*}

\vspace{0.2cm}

\subsection{Wall-Clock Runtimes}
\label{app:runtimes}

Table \ref{tab:wall_clock_results} presents wall-clock runtimes averaged over 20 random seeds for the same experimental settings as reported in the average detection delay (ADD) experiments in Table \ref{tab:results}, except with the true changepoint occurring 200 datapoints after deployment (whereas the results for Table \ref{tab:results} are for 100 samples after deployment). We focus on the fully online testing and changepoint detection methods, which have better ADD results in Table \ref{tab:results}; the results for the 50-sample batch testing baseline, PR-CD-50, is provided in gray for completeness, but it is excluded from this comparison due to its worse ADD performance and different evaluation frequency.
Across all datasets and for both the anytime-valid and the scheduled monitoring criteria, the CTM and WCTM-based methods have faster empirical runtime than the other online baselines. The starkest difference is between the Shiryaev-Roberts (W)CTM methods and the online changepoint detection method from \citet{podkopaev2021tracking}, due to the former's runtime having complexity $\mathcal{O}(t)$ and the latter's complexity being $O(t^2)$, where $t$ is the number of post-deployment timesteps.

\begin{table*}[htbp!]
\centering
\renewcommand\arraystretch{0.8}
\caption{\small Wall-clock runtimes in seconds, averaged over 20 random seeds for each monitoring method and dataset (error bars are standard errors). Anytime-valid monitoring methods are WCTMs (proposed), CTMs \citep{vovk2021testing}, and sequential testing from \citet{podkopaev2021tracking} (PR-ST); multistage monitoring methods are Shiryaev-Roberts (SR) procedure applied to WCTMs (proposed), SR applied to CTMs \citep{vovk2021testing}, changepoint detection on the miscoverage risk from \citet{podkopaev2021tracking} in the online (PR-CD-online) variant. (Results for the minibatched variant with batches of 50-samples, PR-CD-50, are provided in gray for completeness, but it is excluded from this comparison due to its worse ADD results in Table \ref{tab:results} and different evaluation schedule.) Results corresponding to our method are in {\color{NavyBlue}blue}. The best runtimes are \textbf{bolded}.}
\scalebox{0.9}{
\begin{tabular}{ll |cc |ccc }
\toprule
 & \textit{Monitoring Criterion} ($\rightarrow$)& \multicolumn{2}{c}{\textit{Anytime-Valid  Criterion}}   & \multicolumn{3}{c}{\textit{Scheduled, Multistage Criterion}} \\
\cmidrule(lr){3-4} \cmidrule(lr){5-7} 
 & \textit{Monitoring Method} ($\rightarrow$) & (W)CTM  & PR-ST & SR via (W)CTM & PR-CD-online & {\color{gray}PR-CD-50} \\
\cmidrule(lr){3-4} \cmidrule(lr){5-7} \textit{Dataset} ($\downarrow$)
 & \textit{CP Type} ($\downarrow$) & \textbf{Time ($\pm$ SE)} &   \textbf{Time ($\pm$ SE)} &   \textbf{Time ($\pm$ SE)} &   \textbf{Time ($\pm$ SE)} &   {\color{gray}\textbf{Time ($\pm$ SE)}}  \\
\cmidrule(lr){1-7} 
\multirow{2}{*}{\textbf{MEPS}} 
  & Weighted & \color{NavyBlue} \textbf{0.38 ($\pm 0.01$)} &  7.65 $(\pm 1.29)$  &  \color{NavyBlue}\textbf{0.73 $(\pm 0.05)$} &  81.75 $(\pm 10.17)$ &  {\color{gray}0.15 $(\pm 0.01)$} \\
  & Standard  & \textbf{0.38 ($\pm 0.01$)} & 5.55  ($\pm 1.01$)  &   \textbf{0.73 ($\pm 0.05$)} & 81.97 ($\pm 10.27$) &  {\color{gray}0.15 ($\pm 0.01$)}   \\
\cmidrule(lr){1-7} 
\multirow{2}{*}{\textbf{Bike Sharing}} 
  & Weighted & \color{NavyBlue} \textbf{0.44 ($\pm 0.02$)} &  10.26 $(\pm 1.72)$  &  \color{NavyBlue}\textbf{0.87 $(\pm 	0.08)$} &  138.11 $(\pm 16.36)$ &  {\color{gray}0.14 $(\pm 0.01)$} \\
  & Standard  & \textbf{0.43 ($\pm 0.02$)} &  8.45 ($\pm 1.27$)  &   \textbf{0.88 ($\pm 0.09$)} & 135.94 ($\pm 16.92$) &  {\color{gray}0.14 ($\pm 0.01$)}  \\
\cmidrule(lr){1-7}
\multirow{2}{*}{\textbf{Superconduct}} 
  & Weighted & \color{NavyBlue} \textbf{0.43 ($\pm 0.02$)} &  9.93 ($\pm 1.32$) &   \color{NavyBlue}\textbf{0.88 ($\pm 0.09$)} & 215.72  $(\pm 20.63)$  &  {\color{gray}0.20  $(\pm 0.01)$}  \\
  & Standard  & \textbf{0.42 ($\pm 0.02$)} &  9.94 ($\pm 1.03$)  &   \textbf{0.86 ($\pm 0.09$)} & 213.72 ($\pm 23.82$) &  {\color{gray}0.18 ($\pm 0.01$) }  \\
\bottomrule
\end{tabular}}
\label{tab:wall_clock_results}
\end{table*}

\subsection{Ablation Experiments for Density-Ratio Weight Estimation}

Another factor determining whether a covariate shift can be considered benign or harmful to deployment is whether a deployed density-ratio estimator is well-specified and thus able to approximate the shift. Figure \ref{fig_d_abl} provides selected synthetic-data example where logistic regression is a misspecified probabilistic classifier for distinguishing between pre- and post-changepoint data, but where a neural network (MLP) is able to accurately discriminate between the same pre- and post-changepoint data. That is, in this example the pre- and post-changepoint data are not linearly separable in the input $X$ domain, so logistic regression is not able to reliably discriminate, and thereby it is unable to reliably estimate density-ratio weights via probabilistic classification. The result is that the changepoint causes a large increase in coverage, despite some adaptation (decreasing interval widths); the estimator's misspecification thus causes WCTMs to raise an alarm, indicating that the covariate shift cannot be adapted to by the estimator. In contrast, the MLP estimator is able to appropriately adapt by maintaining target coverage, improving interval sharpness, and avoiding unnecessary alarms.

\begin{figure*}[ht!]
\centering
    \includegraphics[width=\textwidth]{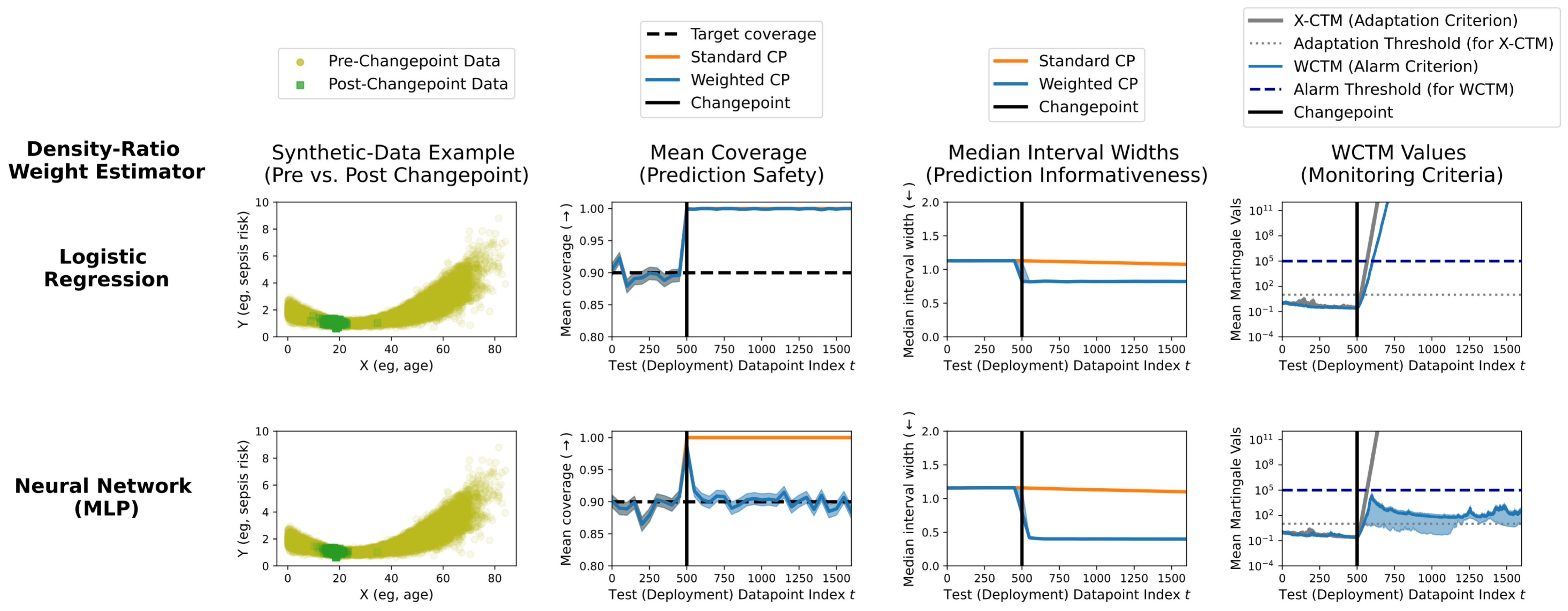}
\caption{Ablation study on density-ratio estimator for synthetic data.}
\label{fig_d_abl}
\end{figure*}

\subsection{Ablation Experiments for Betting Function}

The primary role of the betting function in (W)CTMs and testing-by-betting more broadly is for quickly rejecting the null hypothesis (i.e., raising an alarm) when it is violated. Figure \ref{fig_b_abl} provides ablation experiment on the betting function used for X-CTMs and WCTMs, on three settings of the synthetic-data example. The Composite Jumper betting function is the betting function used in all other experiments, and it is an average of Simple Jumper betting functions over “jumping parameters” $J\in [0.0001, 0.001, 0.01, 0.1, 1]$; here, we set the Simple Jumper baseline to have $J=0.01$. See Vovk et al. (2021) for pseudocode and exposition of the Simple Jumper algorithm. $J=1$ means conservatively spreading bets across all options to avoid cumulative losses, while smaller $J$ encourages “doubling down” on bets that were previously successful. The CTMs with Composite betting are thus lower bounded at $M_t = 0.2$, whereas those with Simple betting continually decrease, resulting in slightly delayed detection speed relative to Composite.

\begin{figure*}[ht!]
\centering
    \includegraphics[width=\textwidth]{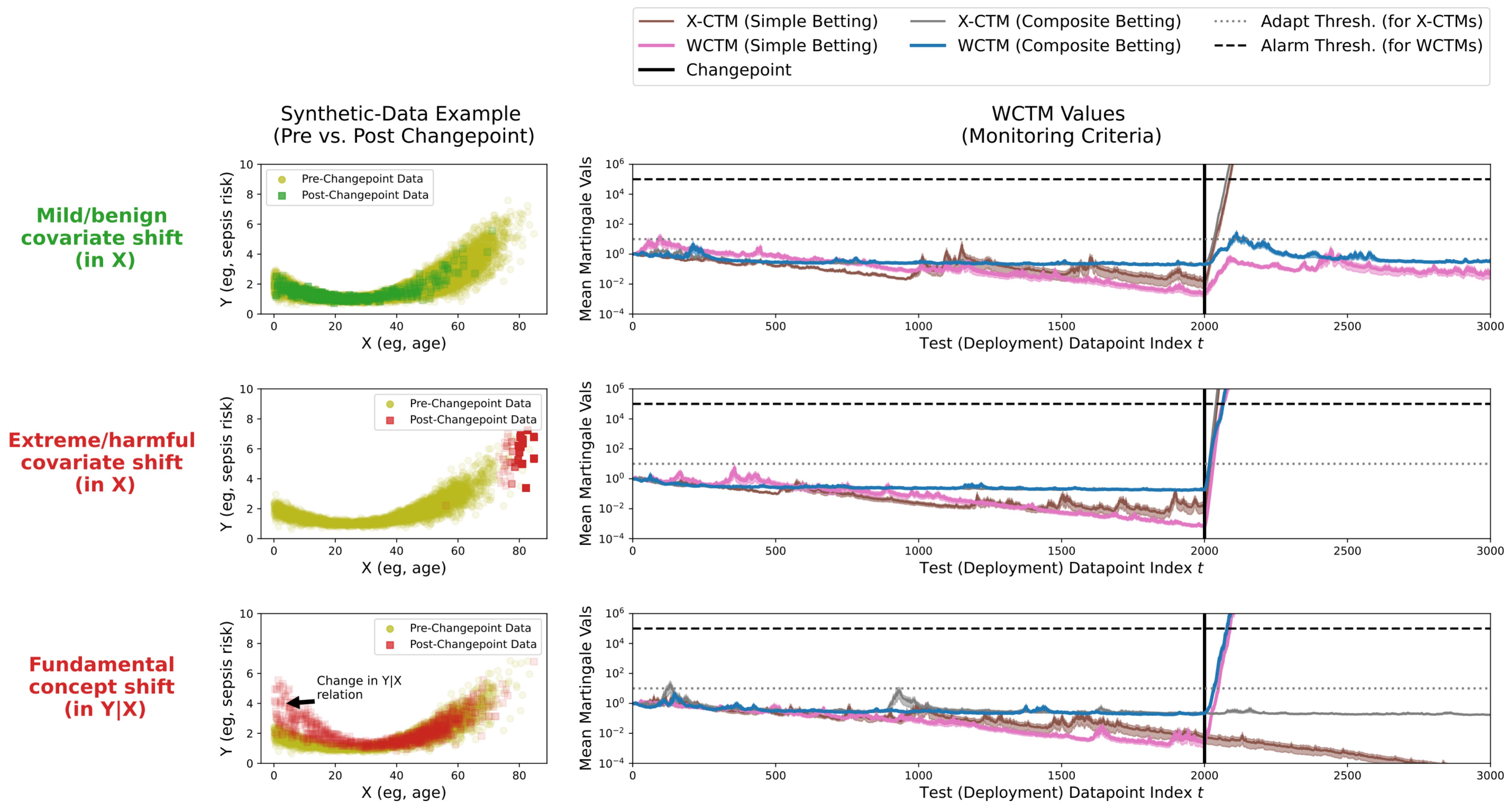}
\caption{Ablation study on betting function.}
\label{fig_b_abl}
\end{figure*}

\subsection{Further Experiments on Corrupted Image Data}
\label{app:more_image_expts}

Here we provide additional experiments on the corrupted image data. 
Our evaluation metrics include the average detection delay (ADD), the average number of unnecessary alarms, and the average number of missed alarms.
First, WCTM exhibits shorter or comparable detection delays relative to standard CTM. In particular, standard CTMs sometimes do not trigger alarms even in the presence of severe corruption, which significantly increases its overall ADD. Additionally, as the corruption level intensifies, WCTM detects shifts more quickly, whereas CTM’s detection speed shows minimal change—consistent with our earlier visualizations. To evaluate unnecessary alarms, we again create a mild, benign shift in the target distribution by mixing clean samples into MNIST and level-1 CIFAR-10 corruption data, while treating higher-level corruptions as harmful. We find that WCTM’s unnecessary alarm rate is roughly \textit{one-third} that of CTM. Moreover, WCTM also exhibits fewer missed alarms, especially under more severe corruptions. These findings highlight the flexibility of our framework, which not only avoids unnecessary alarms but also retains, or may improve, fast detection of harmful shifts.

\begin{table}[t]
\centering
\renewcommand\arraystretch{0.9}
\caption{\small WCTM achieves a lower Average Detection Delay (ADD), especially under severe image corruptions, and also exhibits fewer false alarms and missed alarms. Results are averaged over 10 random seeds and all corruption types.}
\scalebox{0.575}{
\begin{tabular}{l c |c |c |c |c}
\toprule
& & \textbf{MNIST-C} & \textbf{CIFAR10-C L1} & \textbf{CIFAR10-C L3} & \textbf{CIFAR10-C L5} \\
\midrule
\multirow{2}{*}{\textbf{ADD}} 
  & WCTM & \Large \textbf{156.4} & \Large 188.0 & \Large \textbf{163.6} & \Large \textbf{129.7} \\
  & CTM  & \Large 285.3 & \Large \textbf{176.5} & \Large 175.2 & \Large 170.3 \\
\midrule
\multirow{2}{*}{\textbf{Unnecessary Alarm}}
  & WCTM & \Large \textbf{7.3} & \Large \textbf{12.2} & -- & -- \\
  & CTM  & \Large 25.4 & \Large 34.3 & -- & -- \\
\midrule
\multirow{2}{*}{\textbf{Missed Alarm}}
  & WCTM & \Large 2.9 & \Large 6.6 & \Large \textbf{3.8} & \Large \textbf{0.9} \\
  & CTM  & \Large \textbf{2.4} & \Large \textbf{5.9} & \Large 4.6 & \Large 1.2 \\
\bottomrule
\end{tabular}}
\label{tab:image_results}
\end{table}

Whereas Table \ref{tab:image_results} provided metrics averaged over all corruption types, Figure \ref{fig:corruption_set3} provides specific example results on each corruption type separately, focusing on harmful shifts with high severity of corruption. These plots illustrate how WCTMs quickly react to harmful shifts.

\begin{figure*}[htbp]
    \begin{minipage}{0.97\textwidth}
    \centering
    \begin{tabular}{@{\hspace{-2.4ex}} c @{\hspace{-1.5ex}} @{\hspace{-2.4ex}} c @{\hspace{-1.5ex}} @{\hspace{-2.4ex}} c @{\hspace{-1.5ex}} @{\hspace{-2.4ex}} c @{\hspace{-2.4ex}}}
        \begin{tabular}{c}
        \includegraphics[width=.27\textwidth]{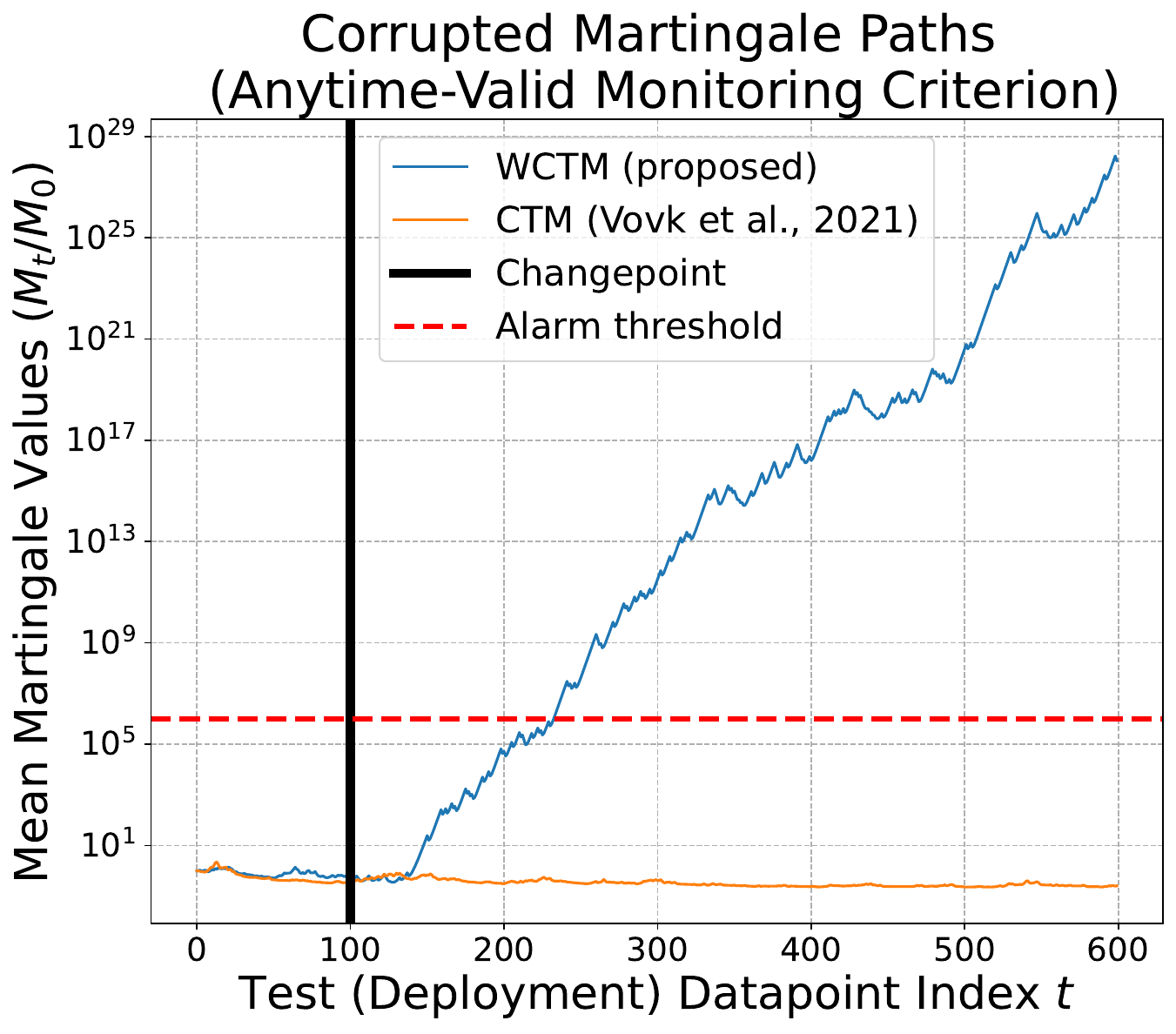}
        \\
        {\small{(a) brightness}}
        \end{tabular} &
        \begin{tabular}{c}
        \includegraphics[width=.27\textwidth]{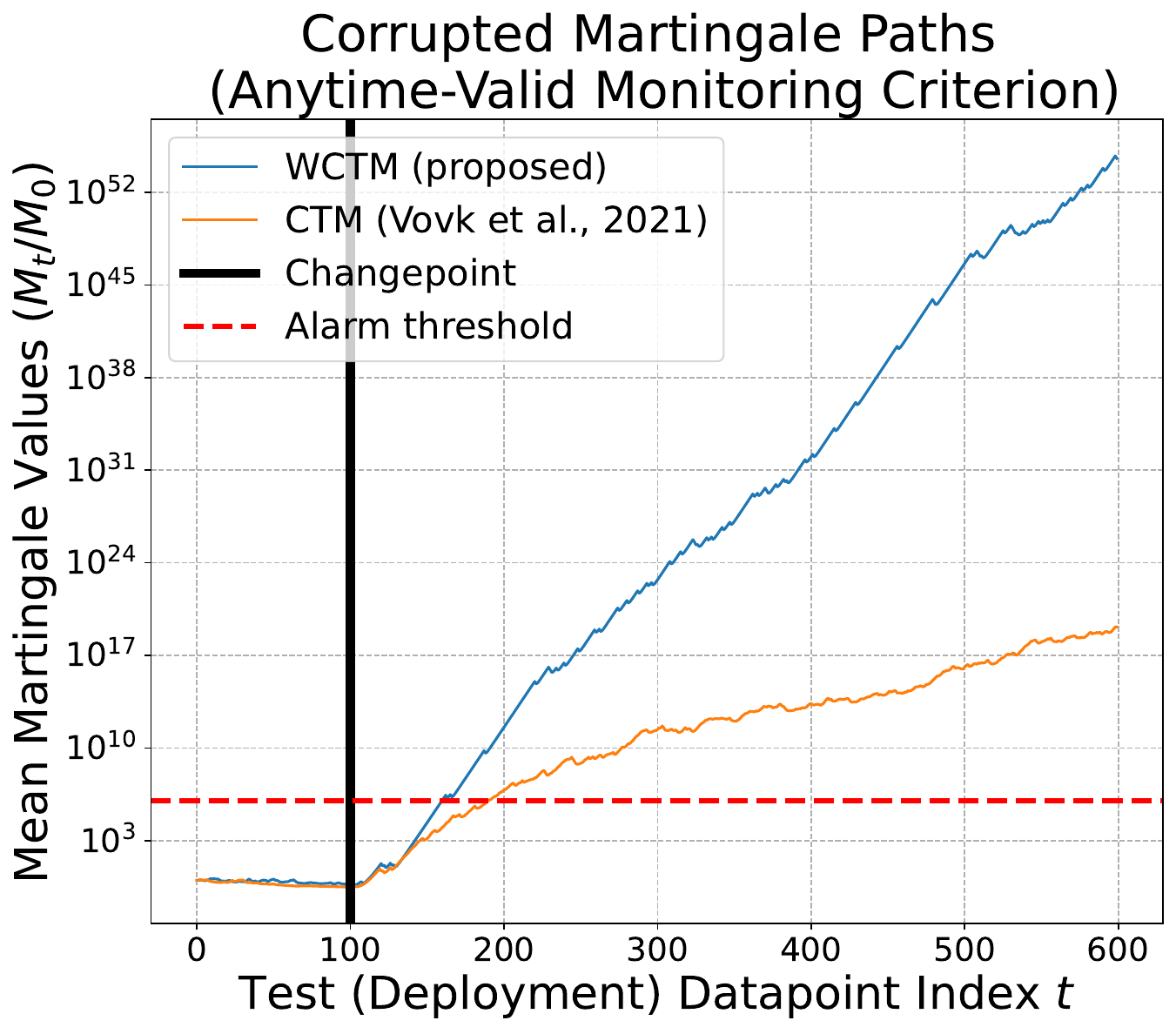}
        \\
        {\small{(b) contrast}}
        \end{tabular} & 
        \begin{tabular}{c}
        \includegraphics[width=.27\textwidth]{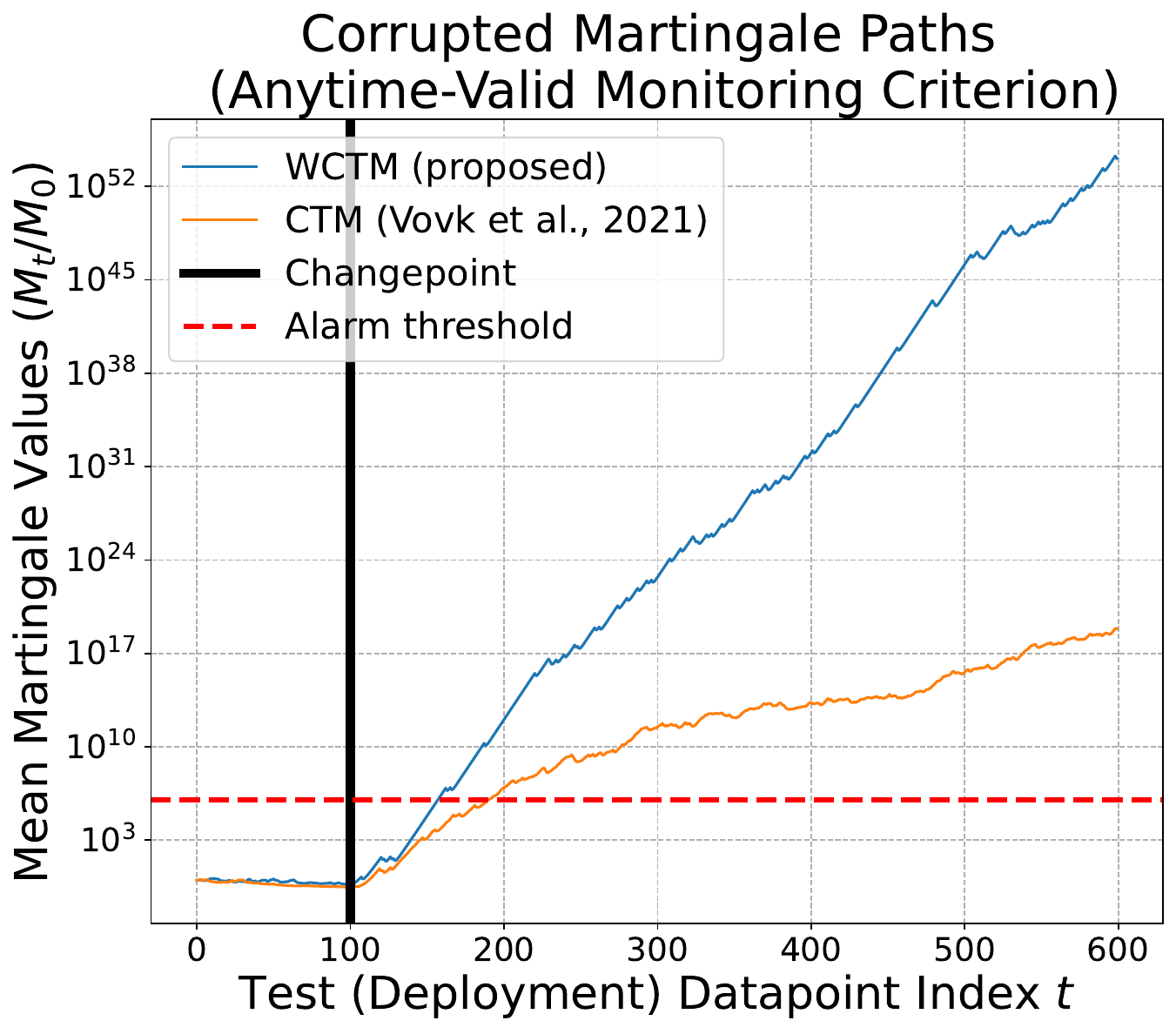} 
        \\
        {\small{(c) defocus blur}}
        \end{tabular} &
        \begin{tabular}{c}
        \includegraphics[width=.27\textwidth]{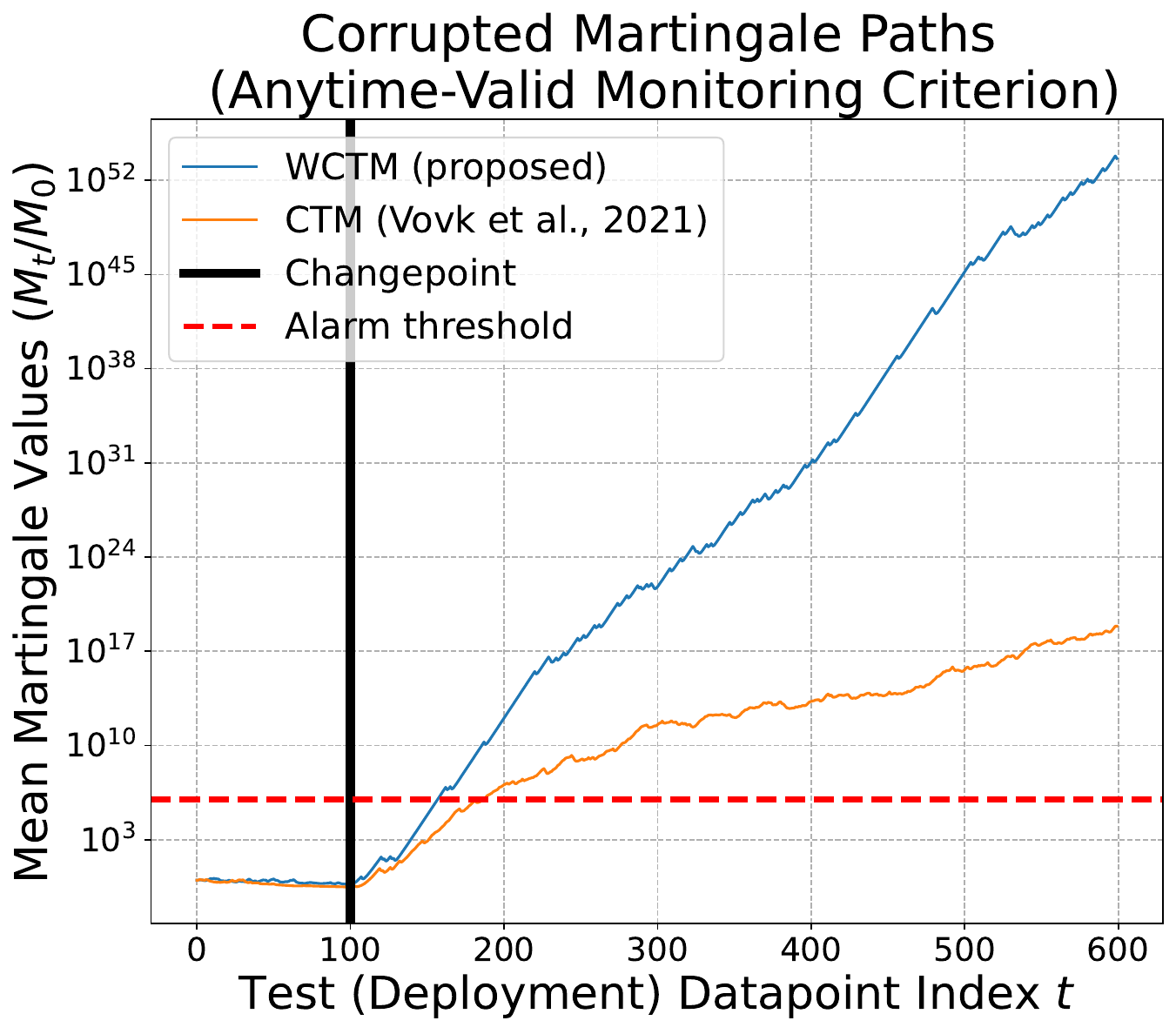}
        \\
        {\small{(d) elastic transform}}
        \end{tabular} \\
        \end{tabular}
    \end{minipage}
    \label{fig:corruption_set1}
\end{figure*}
\begin{figure*}[htbp]
    \begin{minipage}{0.97\textwidth}
    \centering
    \begin{tabular}{@{\hspace{-2.4ex}} c @{\hspace{-1.5ex}} @{\hspace{-2.4ex}} c @{\hspace{-1.5ex}} @{\hspace{-2.4ex}} c @{\hspace{-1.5ex}} @{\hspace{-2.4ex}} c @{\hspace{-2.4ex}}}
        \begin{tabular}{c}
        \includegraphics[width=.27\textwidth]{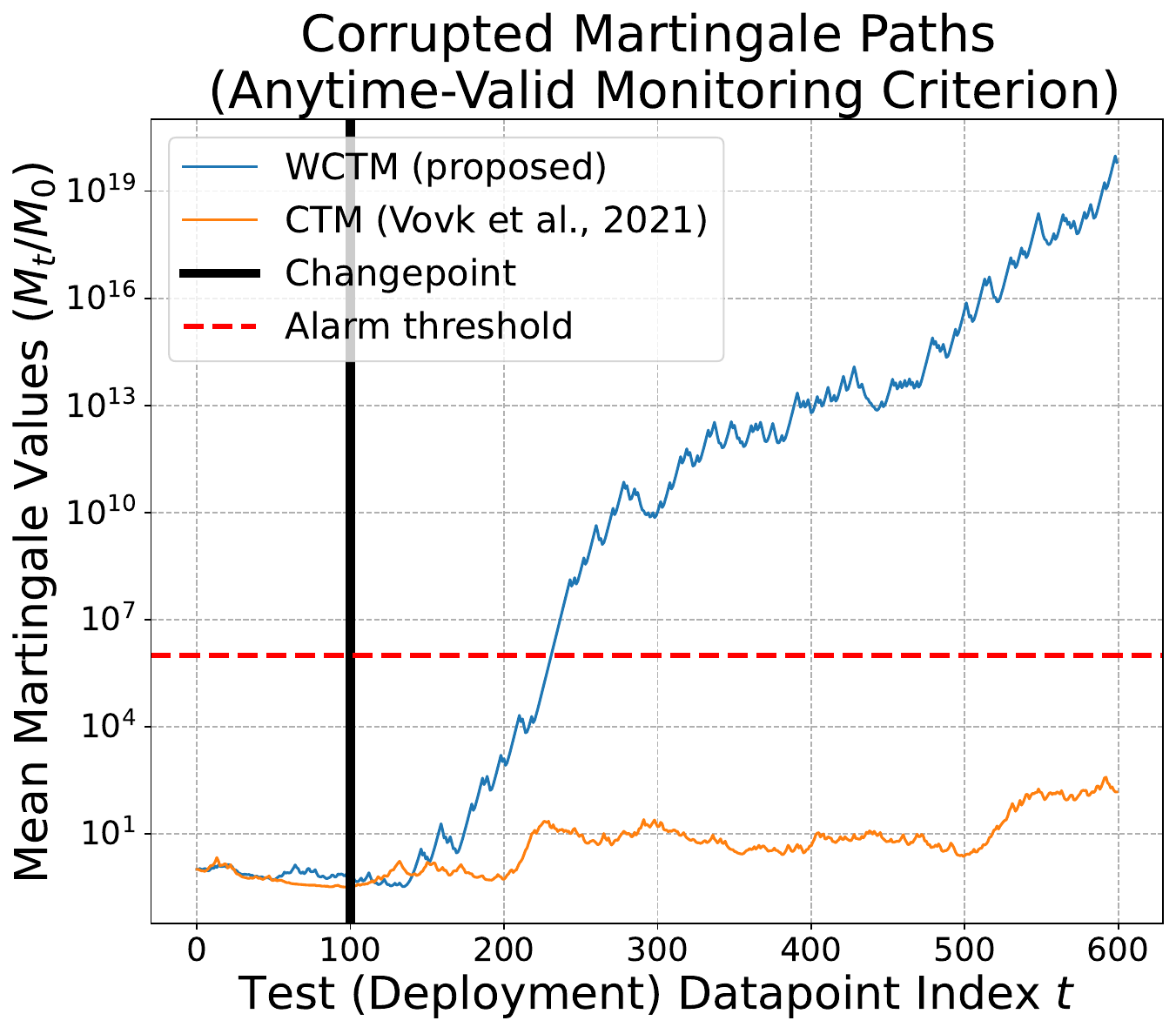}
        \\
        {\small{(e) fog}}
        \end{tabular} &
        \begin{tabular}{c}
        \includegraphics[width=.27\textwidth]{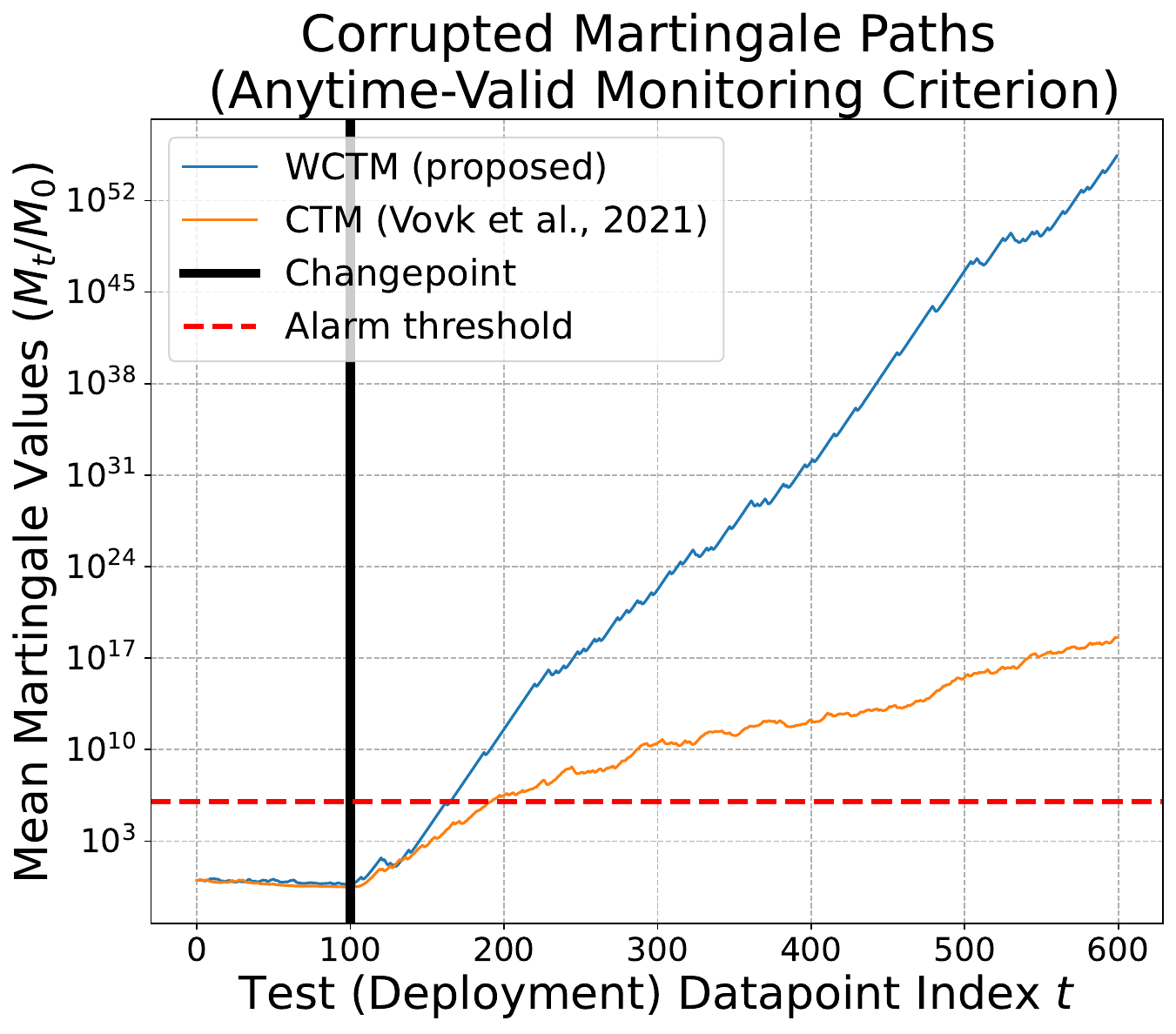}
        \\
        {\small{(f) frost}}
        \end{tabular} & 
        \begin{tabular}{c}
        \includegraphics[width=.27\textwidth]{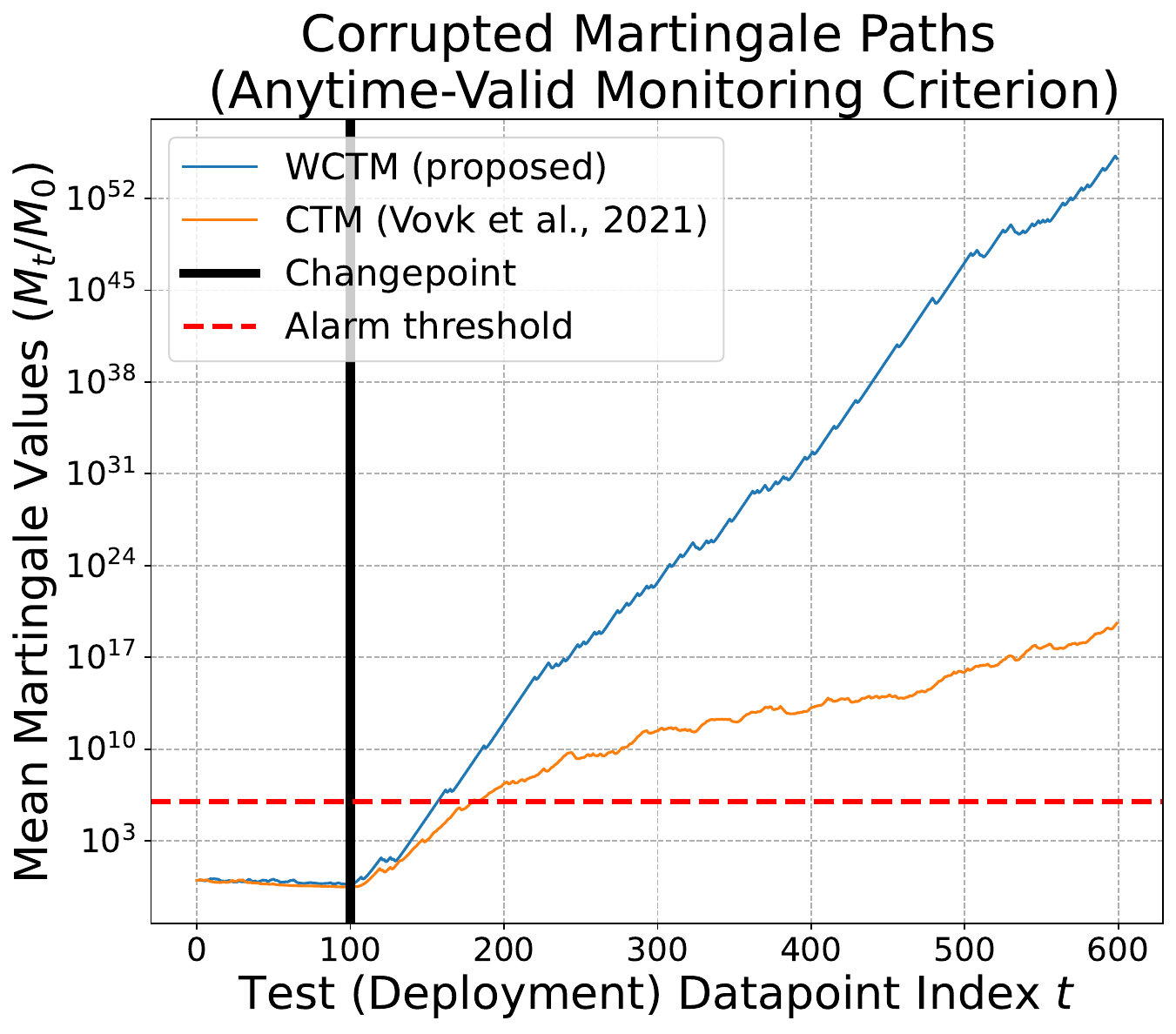} 
        \\
        {\small{(g) gaussian blur}}
        \end{tabular} &
        \begin{tabular}{c}
        \includegraphics[width=.27\textwidth]{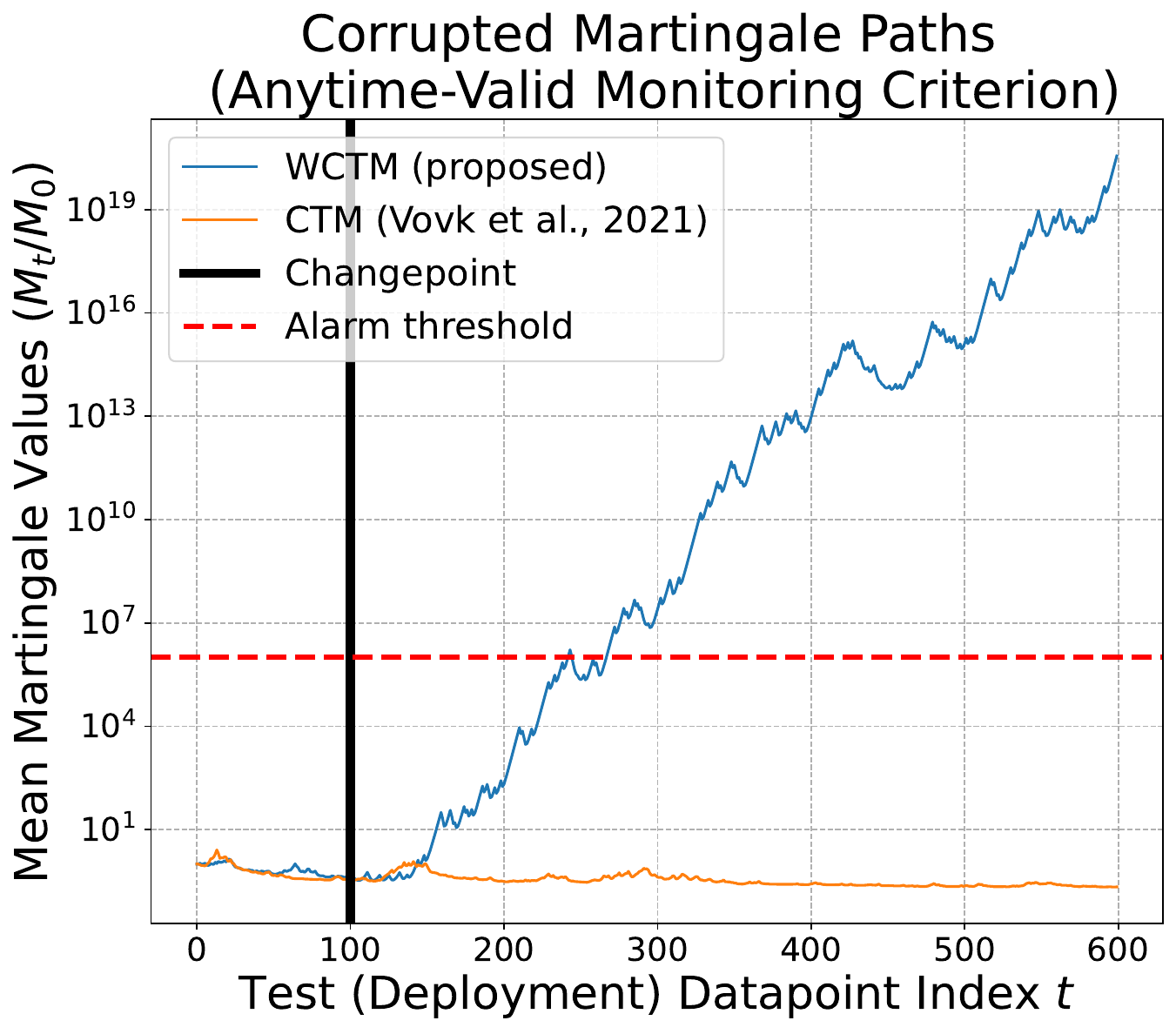}
        \\
        {\small{(h) gaussian noise}}
        \end{tabular} \\
        \end{tabular}
    \end{minipage}
    \label{fig:corruption_set2}
\end{figure*}
\begin{figure*}[htbp]
    \begin{minipage}{0.97\textwidth}
    \centering
    \begin{tabular}{@{\hspace{-2.4ex}} c @{\hspace{-1.5ex}} @{\hspace{-2.4ex}} c @{\hspace{-1.5ex}} @{\hspace{-2.4ex}} c @{\hspace{-1.5ex}} @{\hspace{-2.4ex}} c @{\hspace{-2.4ex}}}
        \begin{tabular}{c}
        \includegraphics[width=.27\textwidth]{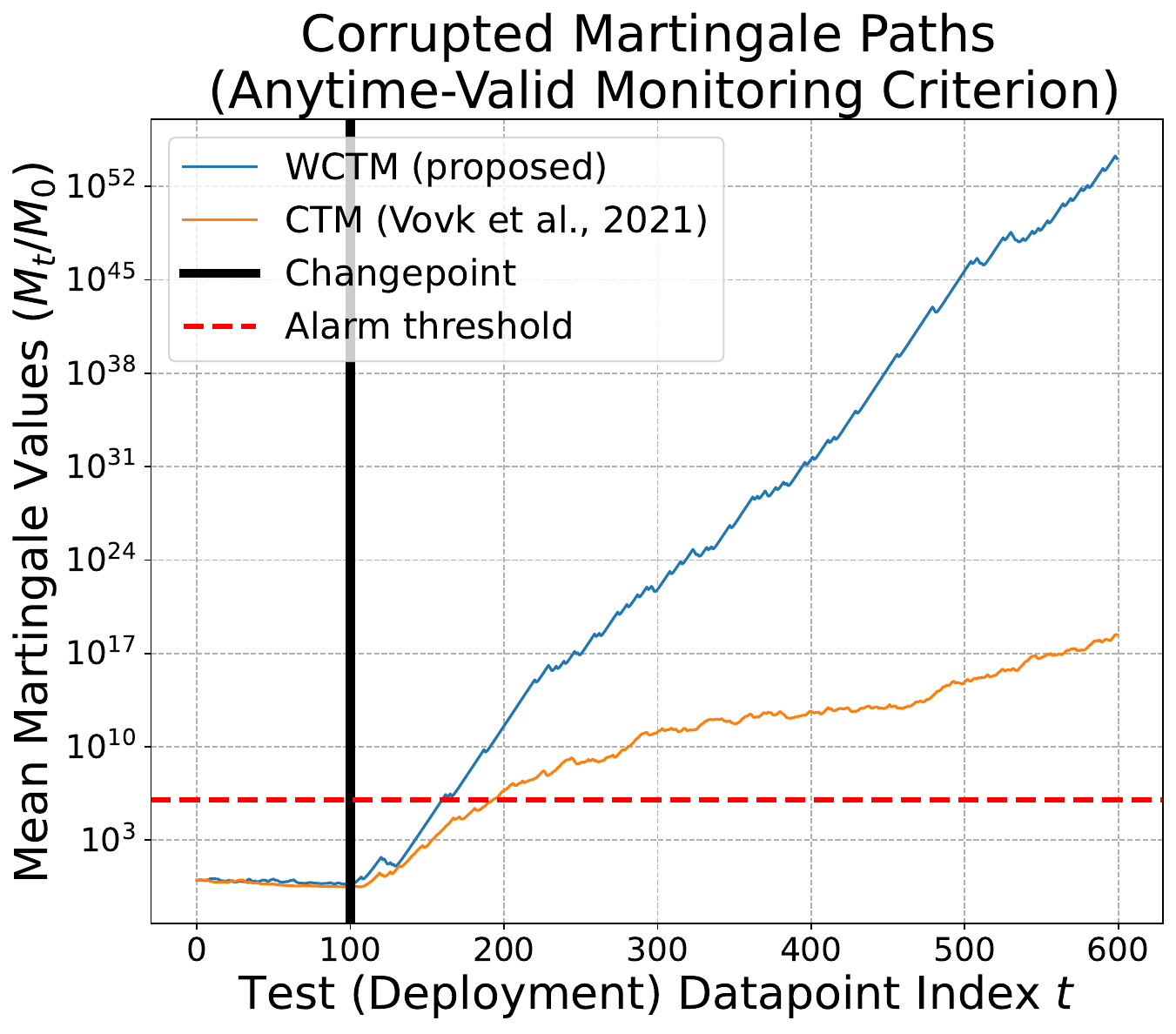}
        \\
        {\small{(i) glass blur}}
        \end{tabular} &
        \begin{tabular}{c}
        \includegraphics[width=.27\textwidth]{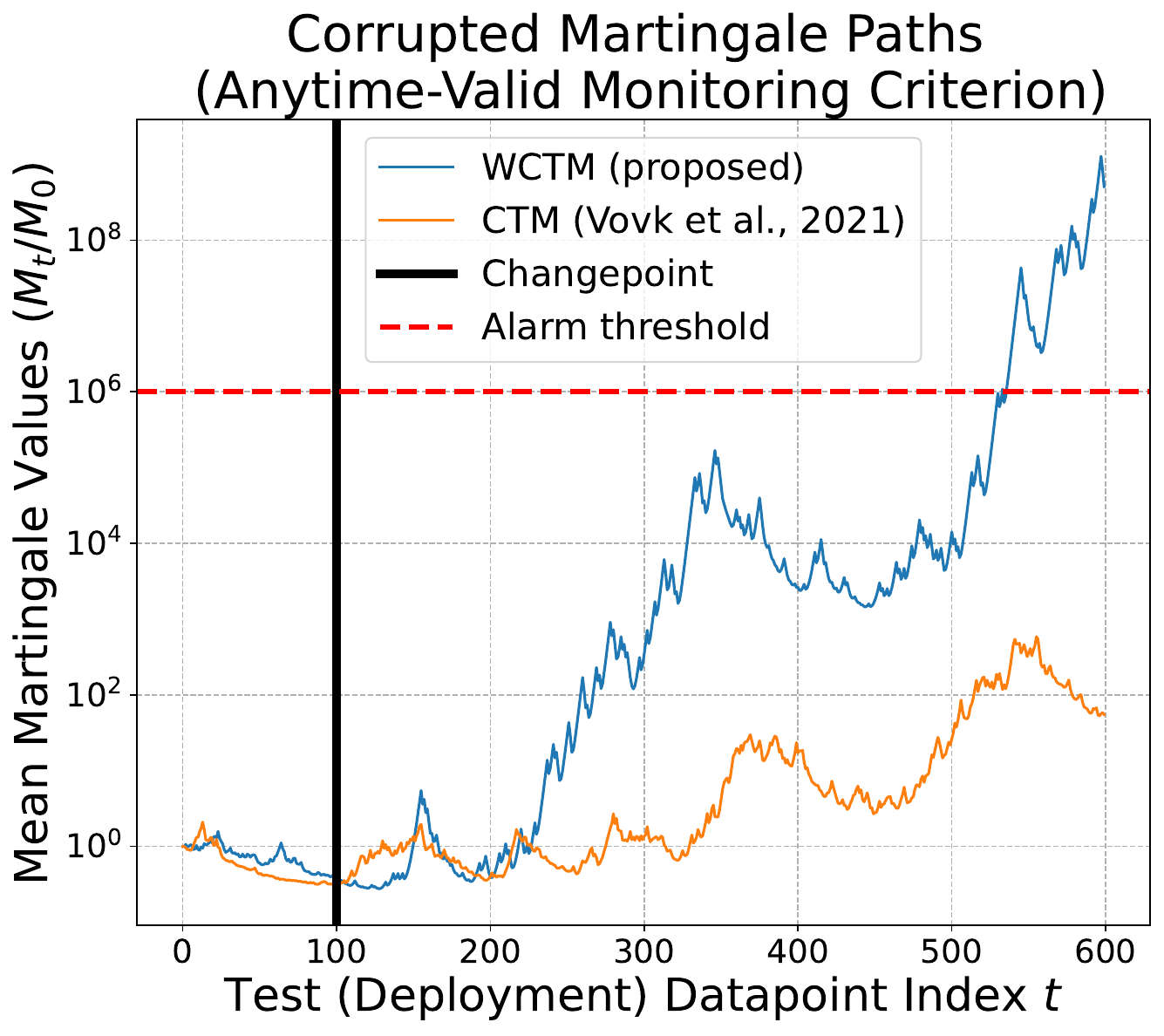}
        \\
        {\small{(j) impulse noise}}
        \end{tabular} & 
        \begin{tabular}{c}
        \includegraphics[width=.27\textwidth]{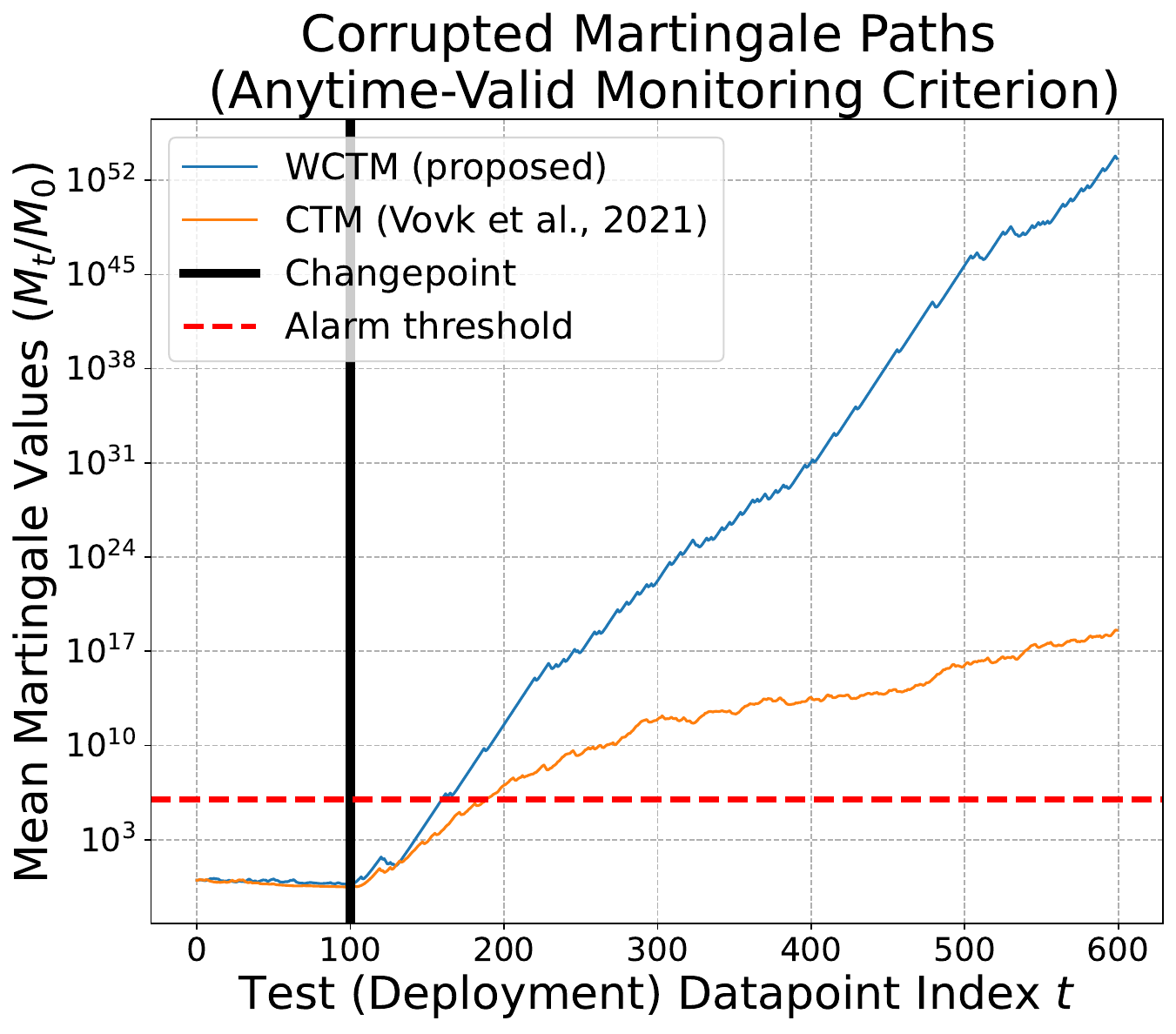} 
        \\
        {\small{(k) jpeg compression}}
        \end{tabular} &
        \begin{tabular}{c}
        \includegraphics[width=.27\textwidth]{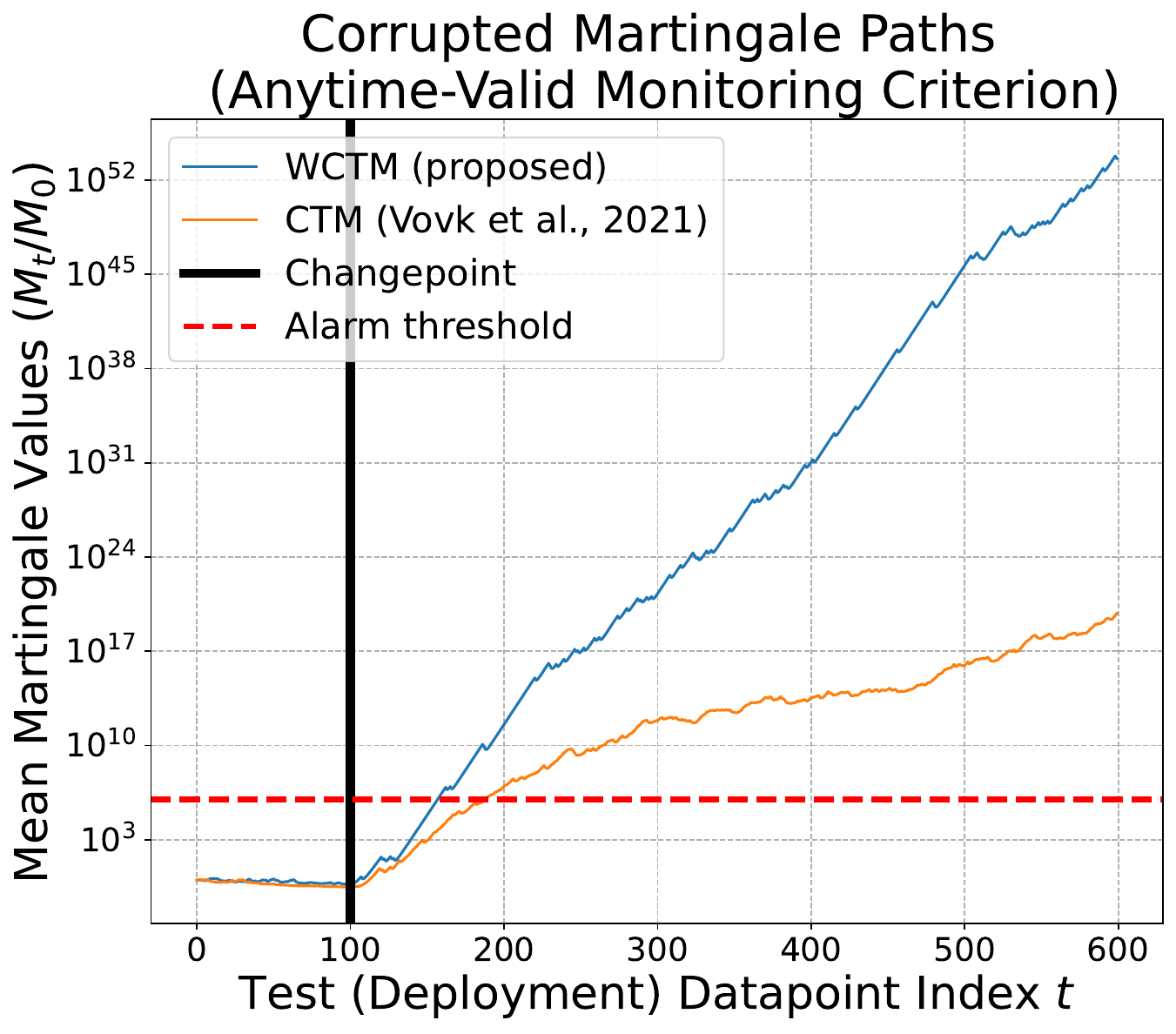}
        \\
        {\small{(l) motion blur}}
        \end{tabular} \\
        \end{tabular}
    \end{minipage}
    \caption{Results on CIFAR-10 with various corruption types, all at the highest severity level. WCTM reacts more quickly than the standard CTM under these conditions. Moreover, with several types of corruption, the standard CTM does not raise any alarms at all.}
    \label{fig:corruption_set3}
\end{figure*}

\section{Experiment Details}
\label{app:expt_details}

\subsection{Datasets} 
The evaluation datasets include both real-world tabular and image datasets. The tabular datasets are for regression tasks (where conformal methods used the absolute value residual nonconformity score $\widehat{S}(x,y)=|y-\widehat{\mu}(x)|$), and the datasest span various sizes and dimensionalities: the Medical Expenditure Panel Survey (MEPS) dataset (33005 samples, 107 features) \citep{cohen2009medical}, the UCI Superconductivity dataset (21263 samples, 81 features) \citep{superconductivty_data_464}, and the UCI bike sharing dataset (17379 samples, 12 features) \citep{bike_sharing_275}. The image datasets for classification tasks (where conformal methods used the one-minus-softmax score $\widehat{S}(x,y)=1-\widehat{p}(y_i\mid x_i)$) were the MNIST-corruption \citep{mu2019mnist} (60000 clean samples, 10000 corrupted samples) and CIFAR-10-corruption \citep{hendrycks2018benchmarking} (50000 clean samples, 10000 corrupted samples), which are standard benchmarks for assessing distribution shifts. 

\subsection{Simulating Shifts in Tabular Data} 
\label{app:simulating_shifts}
To evaluate the online adaptation performance of our proposed WCTMs on the tabular, regression-task datasets, we simulated mild or benign covariate shifts by exponentially tilting (i.e., up-sampling) the test samples from the full dataset based on selected covariates.  
On the MEPS healthcare dataset, $h$ was selected to simulate a demographic shift towards younger, higher-educated patients (expected to be a benign shift due to youth tending to correlate with fewer, less variable health issues). On the bike-sharing dataset, $h$ simulated a shift in weather patterns to colder, windier days; meanwhile, a more complex shift was simulated on the superconductivity dataset by sampling proportional to the projection onto the first principal component representation of the training data. 
Harmful concept shifts in tabular data were simulated by biasing the label values as a function of selected input covariates for each dataset.

\subsection{Further Details of Image Classification Experiments}
We provide further discussions on the image classification experiments from Section \ref{subsec:benign_to_extreme_cs}. Figure \ref{fig:eight_fig} demonstrates the coverage rates and set sizes of different corruption levels as discussed in the main paper. Details of model architectures and training configurations can be found in Table \ref{tab:mnist_discriminator_details} - Table \ref{tab:resnet20_details_std}.

\begin{figure*}[t!]
    \begin{minipage}{\textwidth}
    \centering
    \begin{tabular}{@{\hspace{-2.4ex}} c @{\hspace{-2.4ex}} @{\hspace{-2.4ex}} c @{\hspace{-2.4ex}} @{\hspace{-2.4ex}} c @{\hspace{-2.4ex}} @{\hspace{-2.4ex}} c @{\hspace{-2.4ex}}}
        \begin{tabular}{c}
        \includegraphics[width=.276\textwidth]{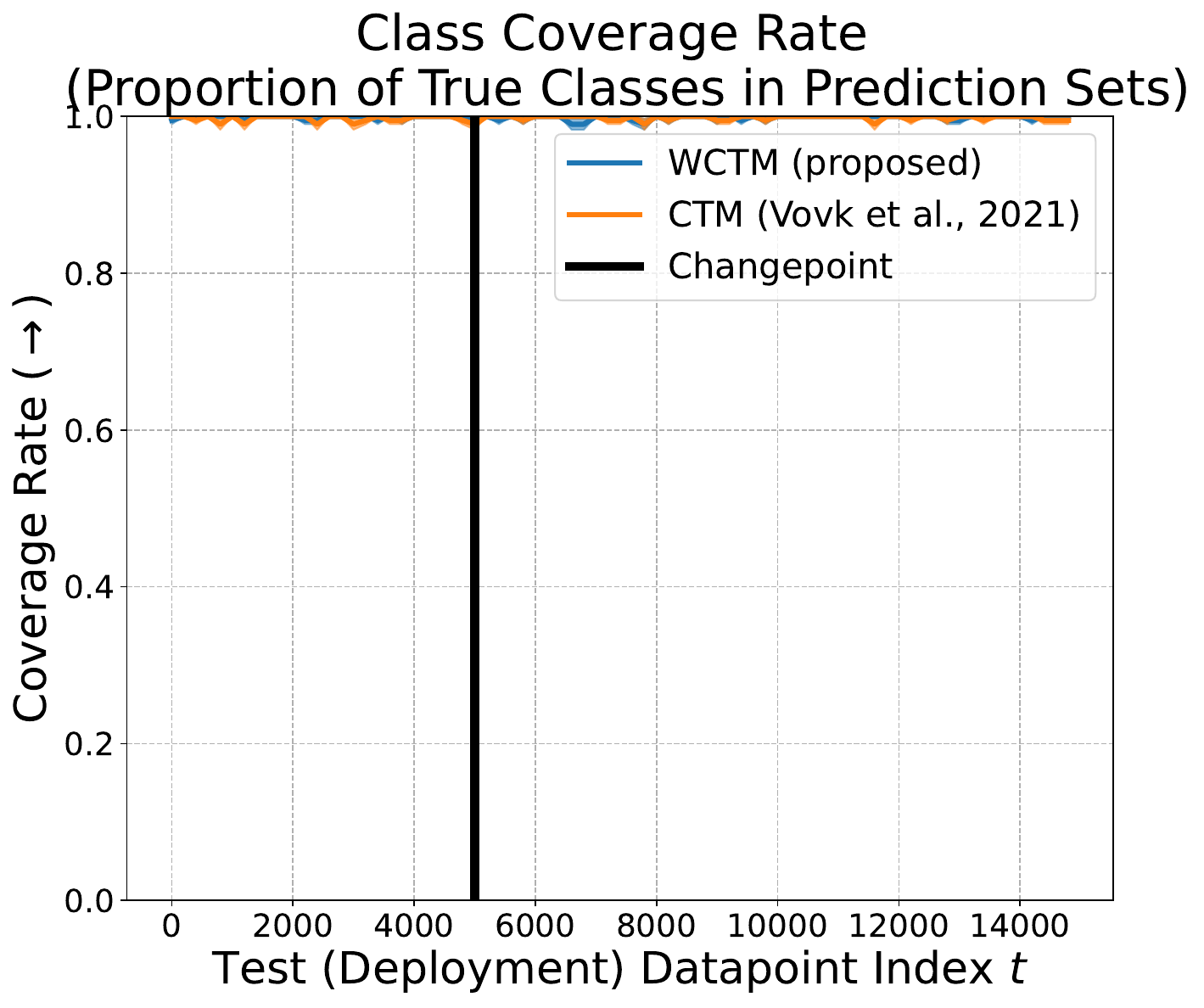}
        \\
        \end{tabular} &
        \begin{tabular}{c}
        \includegraphics[width=.276\textwidth]{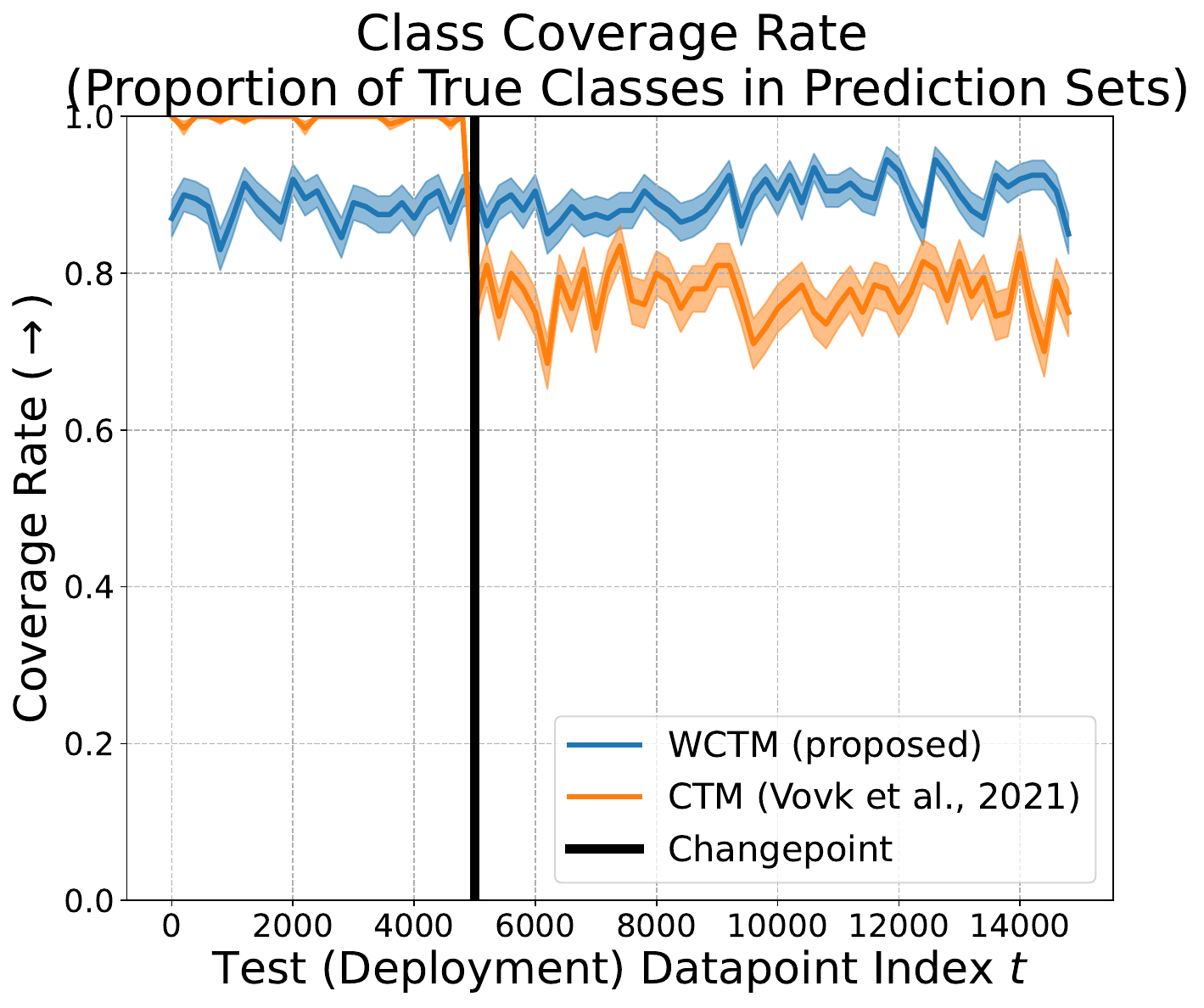}
        \\
        \end{tabular} & 
        \begin{tabular}{c}
        \includegraphics[width=.276\textwidth]{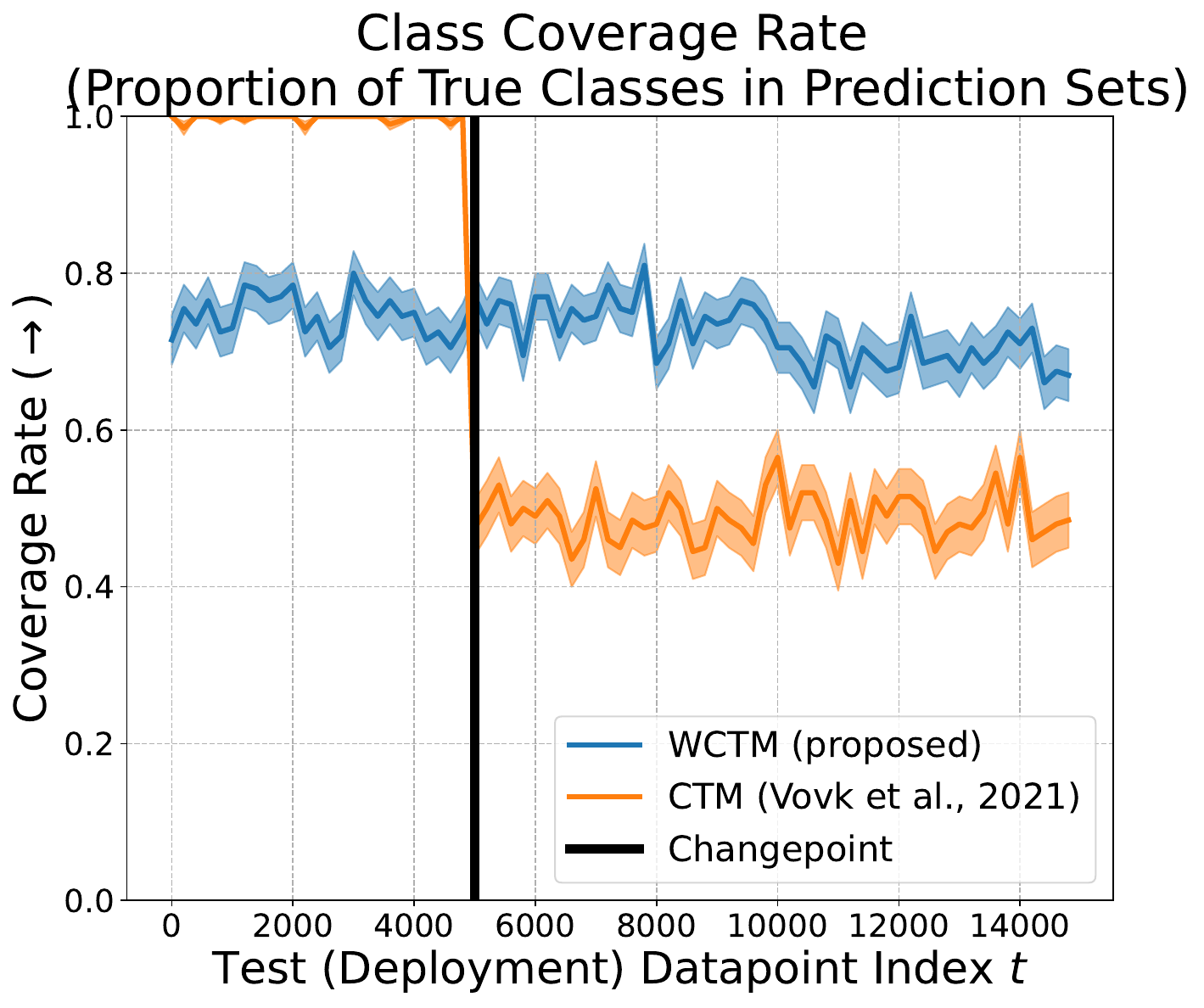} 
        \\
        \end{tabular} &
        \begin{tabular}{c}
        \includegraphics[width=.276\textwidth]{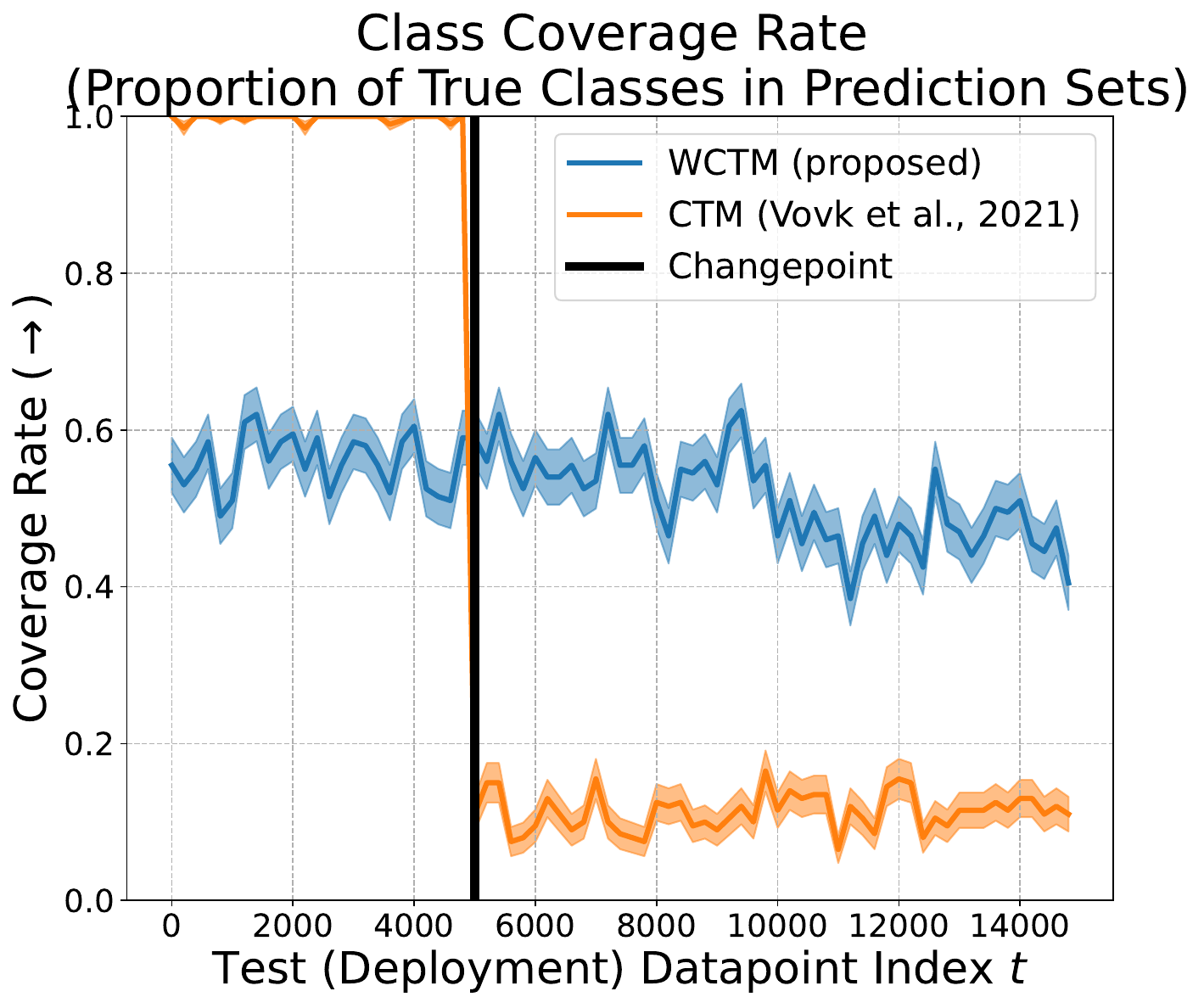}
        \\
        \end{tabular} \\
        \begin{tabular}{c}
        \includegraphics[width=.276\textwidth]{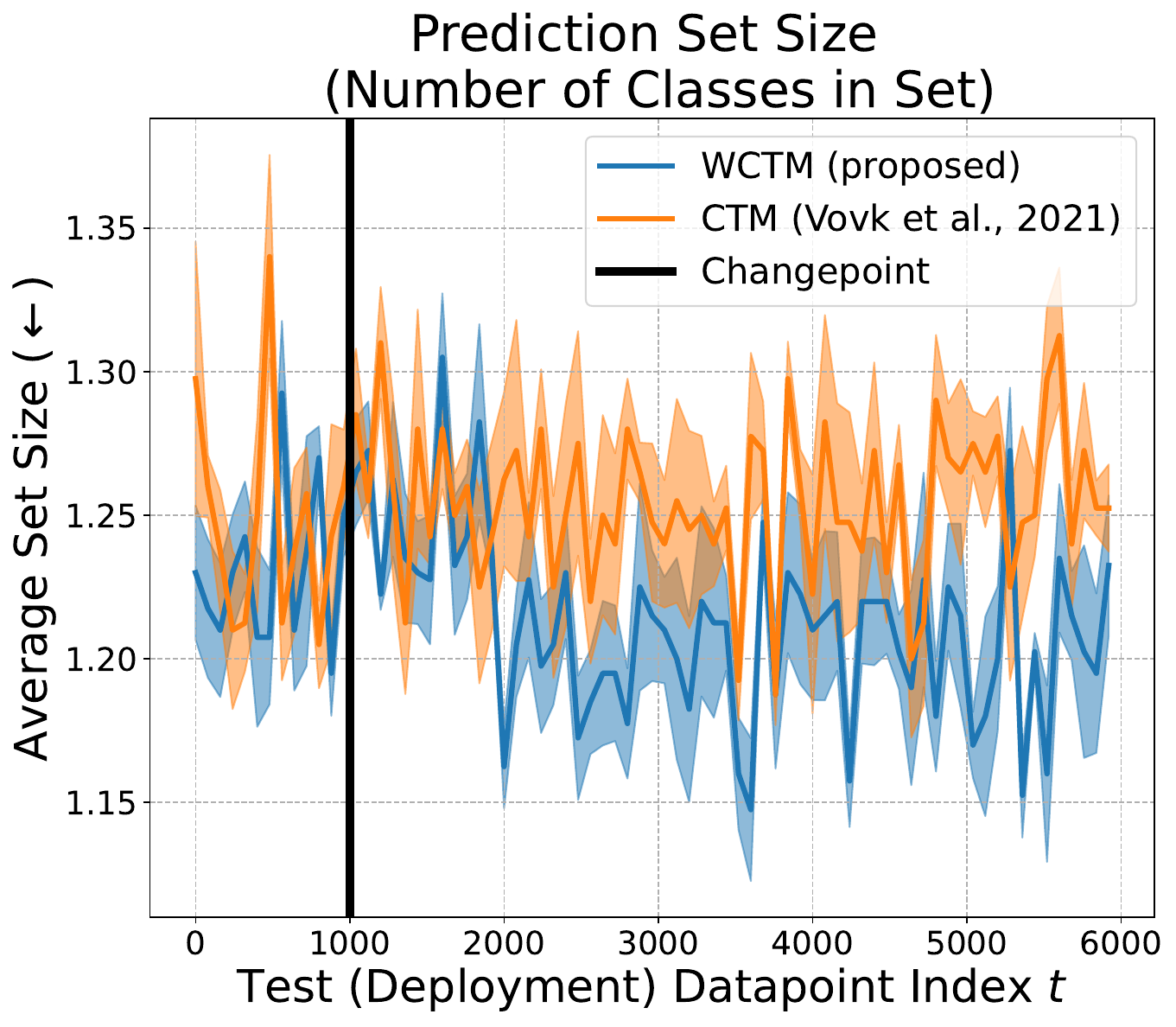}
        \\
        {\small{(a) No Corruption}}
        \end{tabular} &
        \begin{tabular}{c}
        \includegraphics[width=.276\textwidth]{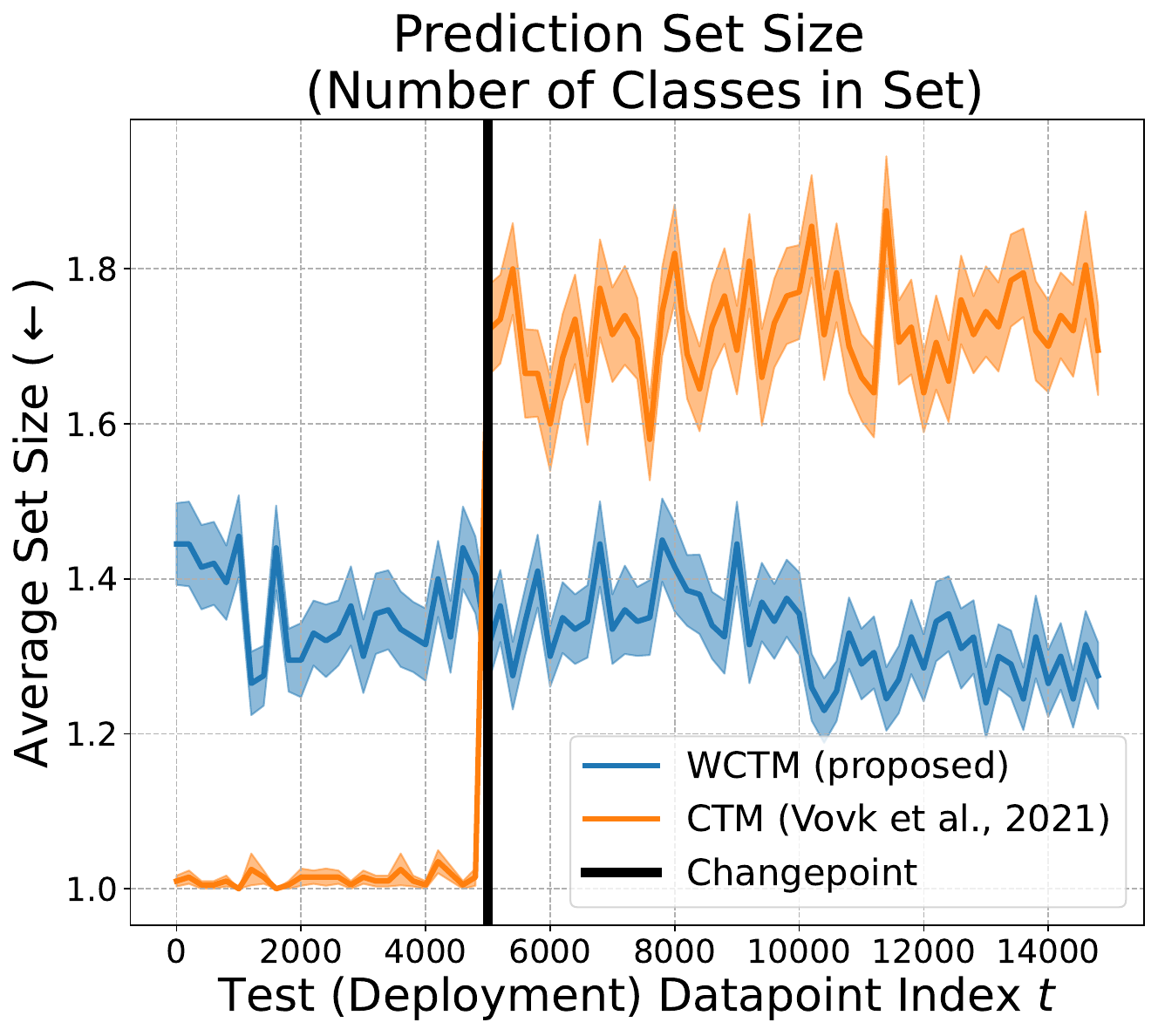}
        \\
        {\small{(b) Minor Benign Corruption}}
        \end{tabular} & 
        \begin{tabular}{c}
        \includegraphics[width=.276\textwidth]{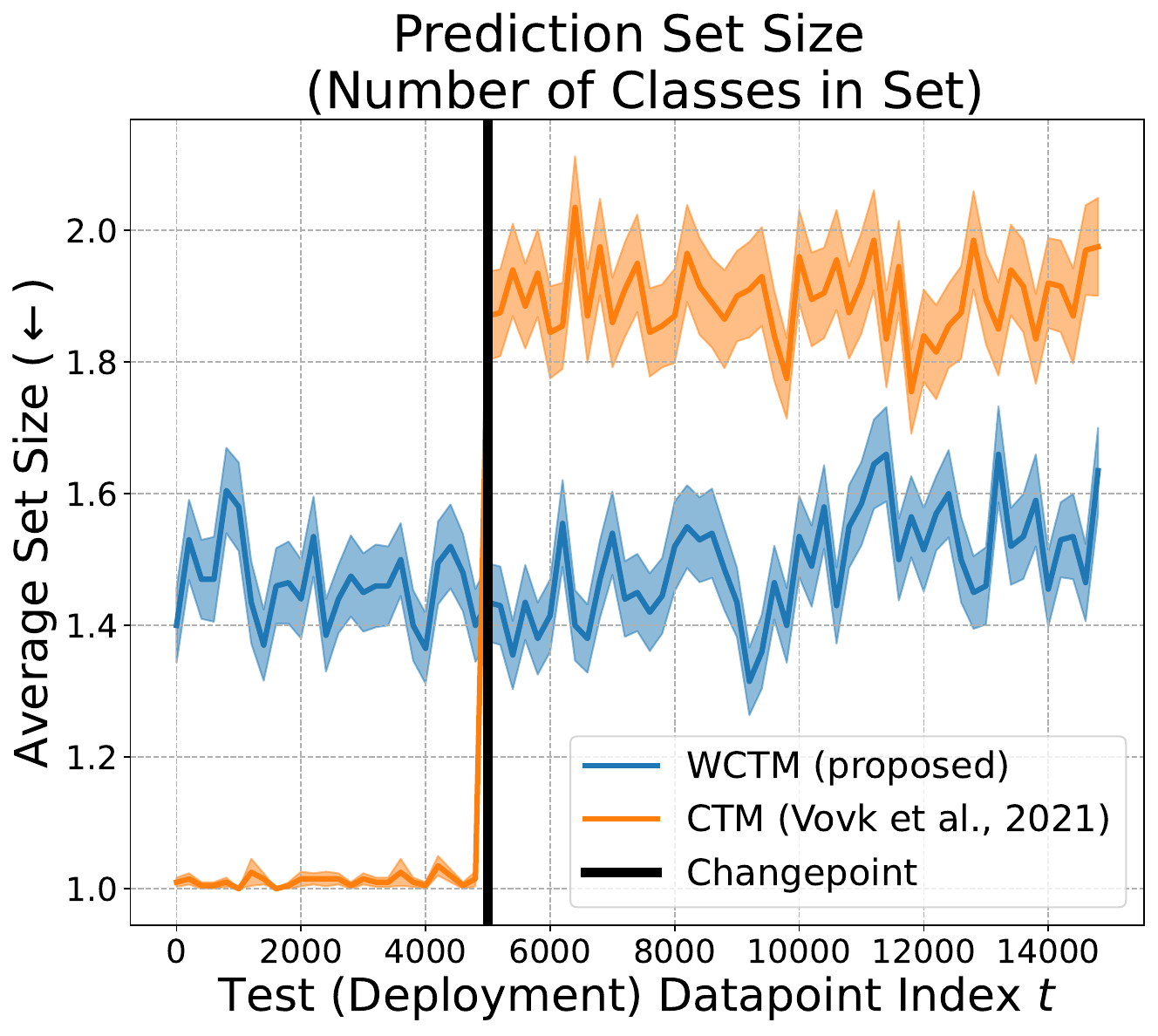} 
        \\
        {\small{(c) Benign Corruption}}
        \end{tabular} &
        \begin{tabular}{c}
        \includegraphics[width=.276\textwidth]{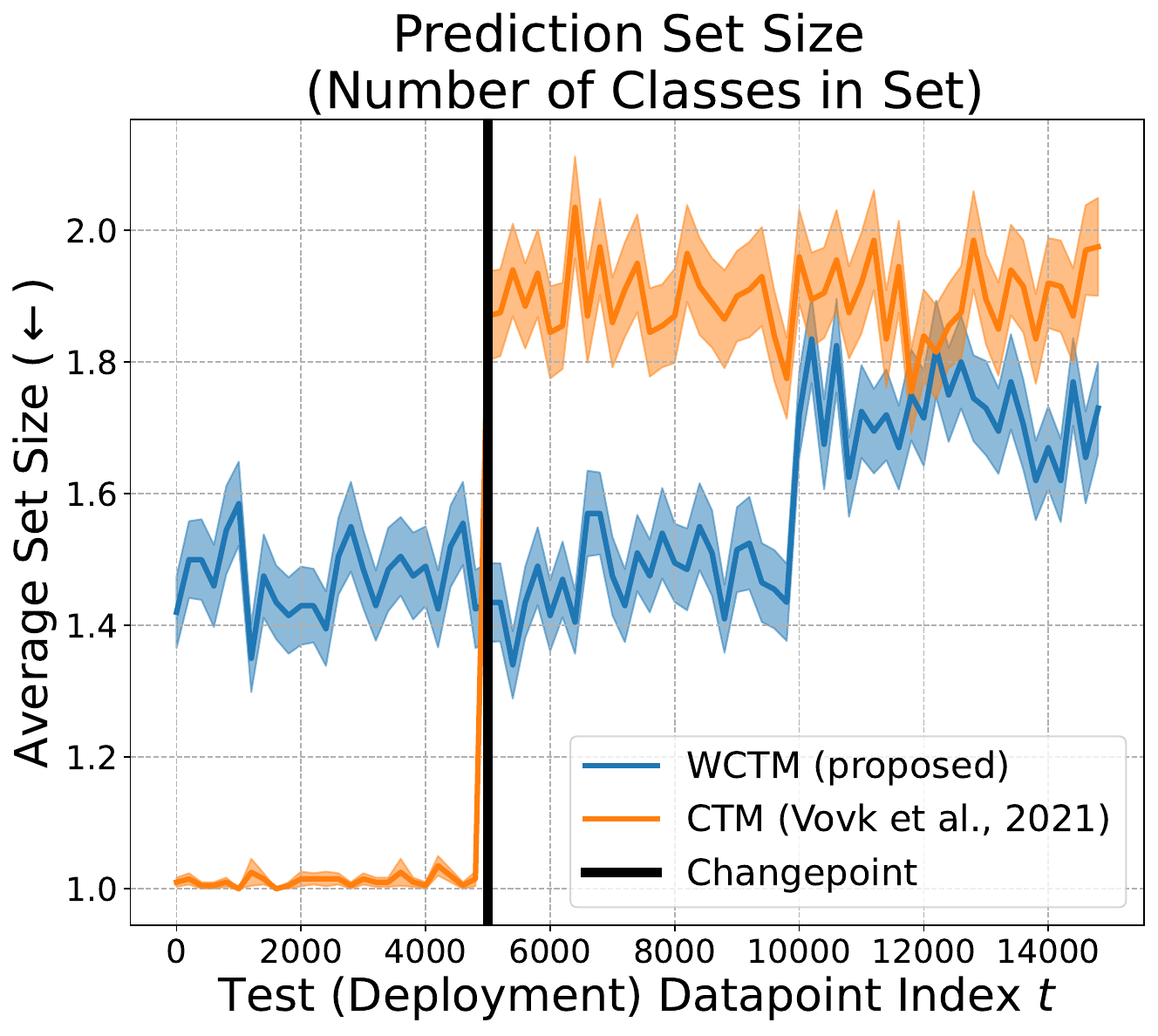}
        \\
        {\small{(d) Extreme Corruption}}
        \end{tabular} \\
        \end{tabular}
    \end{minipage}
    \caption{The results supplement Figure \ref{fig:image_benign} in the main paper. They demonstrate the coverage rate and prediction set size under four different corruption scenarios. In the multi-class classification setting, we adopt metrics different from those used in the regression experiments and follow \citet{romano2020classification} to measure the prediction set size and coverage rate (defined as the proportion of true classes in the specified range, NOT conformal coverage) as principled risk metrics for distinguishing benign from harmful shifts. We increased the size of the validation set and the number of samples visualized to yield more robust performance and provide a clearer view of the trajectory; the results are averaged over a window size of 200, while all other configurations remain unchanged from the original setting. As discussed in the paper, we mixed test samples (target corrupted) with validation samples (source clean) to improve the estimation of weights for WCTMs. So under corrupted scenarios, the “starting points” before the change points for WCTM and CTM differ, as the mixture allows the validation set to contain corrupted data; however, this difference is not clearly reflected in the martingale paths. Overall, the models initially exhibit relatively high classification performance under the clean setting, while CTM rapidly declines to a lower performance level under all corruption conditions. Although WCTM adapts to changes in benign scenarios, it eventually demonstrates severe metric changes under extreme shifts as well, which corresponds to the results in Figure \ref{fig:image_benign}.}
    \label{fig:eight_fig}
\end{figure*}

\begin{table}[htb!] 
    \centering
    \caption{MNISTDiscriminator Model Details}
    \label{tab:mnist_discriminator_details}
    \begin{tabular}{|l|p{0.65\textwidth}|} 
        \hline
        \textbf{Component} & \textbf{Details} \\
        \hline
        Model Name & MNISTDiscriminator \\
        \hline
        Purpose & Binary classifier to distinguish between source (uncorrupted) and target (corrupted) MNIST data \\
        \hline
        Architecture Type & Convolutional Neural Network (CNN) with 2 conv blocks + FC layers \\
        \hline
        Input Shape & (Batch\_size, 1, 28, 28) - Grayscale MNIST images \\ 
        \hline
        Output Shape & (Batch\_size, 2) - Binary classification logits \\ 
        \hline
        Total Parameters & $\approx$1.3M parameters \\ 
        \hline
        Layers &
            First Conv Block: \texttt{Conv2d}(1$\rightarrow$32, 3$\times$3) $\rightarrow$ \texttt{BatchNorm} $\rightarrow$ \texttt{ReLU} $\rightarrow$ \texttt{MaxPool}(2$\times$2) $\rightarrow$ \texttt{Dropout} \newline 
            Second Conv Block: \texttt{Conv2d}(32$\rightarrow$64, 3$\times$3) $\rightarrow$ \texttt{BatchNorm} $\rightarrow$ \texttt{ReLU} $\rightarrow$ \texttt{MaxPool}(2$\times$2) $\rightarrow$ \texttt{Dropout} \newline
            Fully Connected: \texttt{Linear}(64$\times$7$\times$7$\rightarrow$128) $\rightarrow$ \texttt{BatchNorm} $\rightarrow$ \texttt{ReLU} $\rightarrow$ \texttt{Dropout} \newline
            Output: \texttt{Linear}(128$\rightarrow$2) \\
        \hline
        Regularization & \texttt{Dropout} (rate=0.3), \texttt{BatchNormalization} \\
        \hline
        Calibration Method & Temperature scaling (initial temp=1.5) \\
        \hline
        Activation Functions & \texttt{ReLU} \\
        \hline
        Loss Function & Cross-entropy (implicit in the code) \\
        \hline
        Training Epochs & 30 \\ 
        \hline
        Batch Size & 64  \\ 
        \hline
        Learning Rate & 0.001 \\ 
        \hline
        Optimizer & \texttt{Adam} \\
        \hline
        Special Features & Temperature scaling parameter for improved probability calibration \\
        \hline
        Usage Context & Used for estimating likelihood ratios in weighted conformal prediction \\
        \hline
    \end{tabular}
\end{table}

\begin{table}[htb!] 
    \centering
    \caption{CIFAR10 Discriminator Model Details} 
    \label{tab:cifar10_discriminator_details_std} 
    \begin{tabular}{|l|p{0.65\textwidth}|} 
        \hline 
        \textbf{Component} & \textbf{Details} \\
        \hline 
        Model Name & CIFAR10Discriminator \\
        \hline
        Purpose & Binary classifier to distinguish between source (uncorrupted) and target (corrupted) CIFAR-10 data \\
        \hline
        Architecture Type & Convolutional Neural Network (CNN) with 3 conv blocks + FC layers \\
        \hline
        Input Shape & (Batch\_size, 3, 32, 32) - RGB CIFAR-10 images \\ 
        \hline
        Output Shape & (Batch\_size, 2) - Binary classification logits \\ 
        \hline
        Total Parameters & $\approx$4.8M parameters \\ 
        \hline
        Layers & 
            First Conv Block: \texttt{Conv2d}(3$\rightarrow$64, 3$\times$3) $\rightarrow$ \texttt{BatchNorm} $\rightarrow$ \texttt{ReLU} $\rightarrow$ \texttt{MaxPool}(2$\times$2) $\rightarrow$ \texttt{Dropout} \newline
            Second Conv Block: \texttt{Conv2d}(64$\rightarrow$128, 3$\times$3) $\rightarrow$ \texttt{BatchNorm} $\rightarrow$ \texttt{ReLU} $\rightarrow$ \texttt{MaxPool}(2$\times$2) $\rightarrow$ \texttt{Dropout} \newline
            Third Conv Block: \texttt{Conv2d}(128$\rightarrow$256, 3$\times$3) $\rightarrow$ \texttt{BatchNorm} $\rightarrow$ \texttt{ReLU} $\rightarrow$ \texttt{MaxPool}(2$\times$2) $\rightarrow$ \texttt{Dropout} \newline
            First FC Layer: \texttt{Linear}(256$\times$4$\times$4$\rightarrow$512) $\rightarrow$ \texttt{BatchNorm} $\rightarrow$ \texttt{ReLU} $\rightarrow$ \texttt{Dropout} \newline
            Second FC Layer: \texttt{Linear}(512$\rightarrow$128) $\rightarrow$ \texttt{BatchNorm} $\rightarrow$ \texttt{ReLU} $\rightarrow$ \texttt{Dropout} \newline
            Output: \texttt{Linear}(128$\rightarrow$2) \\
        \hline
        Regularization & \texttt{Dropout} (rate=0.3), \texttt{BatchNormalization} at each layer \\
        \hline
        Calibration Method & Temperature scaling (initial temp=1.5) \\
        \hline
        Activation Functions & \texttt{ReLU} throughout the network \\
        \hline
        Loss Function & Cross-entropy \\
        \hline
        Training Epochs & 30 \\
        \hline
        Batch Size & 64 \\
        \hline
        Learning Rate & 0.001 \\
        \hline
        Optimizer & \texttt{Adam} \\
        \hline
        Special Features & Temperature scaling parameter for improved probability calibration \\
        \hline
        Usage Context & Used for estimating likelihood ratios in weighted conformal prediction \\
        \hline 
    \end{tabular}
\end{table}

\begin{table}[htb!] 
    \centering 
    \caption{RegularizedMNISTModel Details} 
    \label{tab:regularized_mnist_model_details_std} 
    \begin{tabular}{|l|p{0.65\textwidth}|} 
        \hline 
        \textbf{Component} & \textbf{Details} \\
        \hline 
        Model Name & RegularizedMNISTModel \\
        \hline
        Purpose & Classification model for MNIST digits (0-9) with robustness to corrupted images \\
        \hline
        Architecture Type & Convolutional Neural Network (CNN) with 3 conv blocks + FC layers \\
        \hline
        Input Shape & (Batch\_size, 1, 28, 28) - Grayscale MNIST images \\ 
        \hline
        Output Shape & (Batch\_size, 10) - Logits for 10-class digit classification \\ 
        \hline
        Total Parameters & $\approx$600K parameters \\ 
        \hline
        Layers & 
            First Conv Block: \texttt{Conv2d}(1$\rightarrow$32, 3$\times$3) $\rightarrow$ \texttt{BatchNorm} $\rightarrow$ \texttt{ReLU} $\rightarrow$ \texttt{MaxPool}(2$\times$2) $\rightarrow$ \texttt{Dropout} \newline
            Second Conv Block: \texttt{Conv2d}(32$\rightarrow$64, 3$\times$3) $\rightarrow$ \texttt{BatchNorm} $\rightarrow$ \texttt{ReLU} $\rightarrow$ \texttt{MaxPool}(2$\times$2) $\rightarrow$ \texttt{Dropout} \newline
            Third Conv Block: \texttt{Conv2d}(64$\rightarrow$128, 3$\times$3) $\rightarrow$ \texttt{BatchNorm} $\rightarrow$ \texttt{ReLU} $\rightarrow$ \texttt{MaxPool}(2$\times$2) $\rightarrow$ \texttt{Dropout} \newline
            Fully Connected: \texttt{Linear}(128$\times$3$\times$3$\rightarrow$256) $\rightarrow$ \texttt{BatchNorm} $\rightarrow$ \texttt{ReLU} $\rightarrow$ \texttt{Dropout} \newline
            Output: \texttt{Linear}(256$\rightarrow$10) \\
        \hline
        Regularization & \texttt{Dropout} (rate=0.3), \texttt{BatchNormalization} at each layer \\
        \hline
        Activation Functions & \texttt{ReLU} throughout the network \\
        \hline
        Loss Function & Cross-entropy \\
        \hline
        Training Epochs & 30 \\
        \hline
        Batch Size & 64 \\
        \hline
        Learning Rate & 0.001 \\
        \hline
        Optimizer & \texttt{Adam} \\
        \hline
        Special Features & Extensive regularization for robustness to corrupted images \\
        \hline
        Usage Context & Primary classification model for MNIST digits in conformal prediction framework \\
        \hline 
    \end{tabular}
\end{table}

\begin{table}[htb!] 
    \centering 
    \caption{ResNet20 Model Details (Standard Format)} 
    \label{tab:resnet20_details_std} 
    \begin{tabular}{|l|p{0.65\textwidth}|} 
        \hline 
        \textbf{Component} & \textbf{Details} \\
        \hline 
        Model Name & ResNet20 \\
        \hline
        Purpose & Classification model for CIFAR-10 images \\
        \hline
        Architecture Type & Residual Network (ResNet v1) with basic blocks \\
        \hline
        Input Shape & (Batch\_size, 3, 32, 32) - RGB CIFAR-10 images \\ 
        \hline
        Output Shape & (Batch\_size, 10) - Logits for 10-class CIFAR-10 classification \\ 
        \hline
        Total Parameters & $\approx$270K parameters \\ 
        \hline
        Depth & 20 layers (1 initial conv + 18 layers in blocks + 1 final linear) \\
        \hline
        Block Structure & \texttt{BasicBlock}: \texttt{Conv}$\rightarrow$\texttt{BN}$\rightarrow$\texttt{ReLU}$\rightarrow$\texttt{Conv}$\rightarrow$\texttt{BN} + shortcut connection, followed by \texttt{ReLU} \\ 
        \hline
        Network Architecture & 
            Initial Layer: \texttt{Conv2d}(3$\rightarrow$16, 3$\times$3) $\rightarrow$ \texttt{BatchNorm} $\rightarrow$ \texttt{ReLU} \newline
            Stage 1: 3 \texttt{BasicBlocks} (16 channels, stride=1) \newline
            Stage 2: 3 \texttt{BasicBlocks} (32 channels, stride=2) \newline
            Stage 3: 3 \texttt{BasicBlocks} (64 channels, stride=2) \newline
            Output: \texttt{Global AvgPool} $\rightarrow$ \texttt{Linear}(64$\rightarrow$10) \\
        \hline
        Regularization & BatchNormalization in each block \\
        \hline
        Activation Functions & ReLU \\
        \hline
        Weight Initialization & Kaiming normal for convolutional layers, constant for batch normalization \\
        \hline
        Loss Function & Cross-entropy \\
        \hline
        Training Epochs & 30 \\
        \hline
        Batch Size & 64 \\
        \hline
        Learning Rate & 0.001  \\
        \hline
        Optimizer & Adam \\
        \hline
        Special Features & Skip connections (residual learning) for better gradient flow \\
        \hline
        Usage Context & Primary classification model for CIFAR-10 in conformal prediction framework \\
        \hline 
    \end{tabular}
\end{table}

\newpage

\vspace{4cm}

\section{Algorithms}\label{sec:alg}

Only limited algorithm pseudocode is provided at this time, focused on algorithms that explain how we penalized noninformative conformal prediction intervals (i.e., with anticonservative $p$-values). 

\begin{algorithm}[tb!]
   \caption{Calculate weighted conformal prediction set for covariate shift \citep{tibshirani2019conformal}.}
   \label{alg:example1}
\begin{algorithmic}
   \STATE {\bfseries Input:} Calibration data $Z_{1:n}:=\{(X_i, Y_i)\}_{i=1}^n$; Test point input $X_{n+1}$; Score function $\hat{s}:(\mathcal{X}\times \mathcal{Y})^*\rightarrow (\mathbb{R}^{+})^*$; Density-ratio weight function $\hat{w}: \mathcal{X}^*\rightarrow [0, 1]^*$; Significance level (target miscoverage) $\alpha$. \\
    \vspace{0.1cm}
    \STATE {\bfseries Output:} Conformal prediction set $\widehat{C}_{n, \alpha}(X_{n+1}) \subseteq \mathcal{Y}$.\\
    \vspace{0.3cm}
   \STATE {\fontfamily{qcr}\selectfont Calculate $n$ nonconformity scores and $n+1$ density-ratio weights:} \\
    \STATE $v_{1}, ..., v_{n}=\hat{s}(Z_1, ..., Z_{n})$ \\
    \STATE $\tilde{w}_{1}, ..., \tilde{w}_{n+1}=\hat{w}(X_1, ..., X_{n+1})$ \\
    \vspace{0.3cm}
    \STATE {\fontfamily{qcr}\selectfont Calculate $1-\alpha$ quantile over weighted score distribution (with conservative adjustment $v_{n+1}=\infty$, as $Y_{n+1}$ is unknown):} \\
    \STATE $q_{1-\alpha} = \widehat{Q}_{1-\alpha}\big(\sum_{i=1}^n\tilde{w}_i\cdot \delta_{v_i} + \tilde{w}_{n+1}\cdot \delta_{\infty}\big)$ \qquad \COMMENT{{\fontfamily{qcr}\selectfont\color{red}Note: $\tilde{w}_{n+1} \geq \alpha\implies q_{1-\alpha}=\infty.$}}
    \vspace{0.3cm}
    \STATE {\fontfamily{qcr}\selectfont Calculate weighted conformal prediction set:} \\
    \STATE $\widehat{C}_{n, \alpha}(X_{n+1}) = \big\{y\in \mathcal{Y} : \hat{s}(X_{n+1}, y)\leq q_{1-\alpha}\big\}$ \qquad \COMMENT{{\fontfamily{qcr}\selectfont\color{red}Note: $q_{1-\alpha}=\infty \implies \widehat{C}_{n, \alpha}(X_{n+1}) = \mathcal{Y}.$}}
    \vspace{0.3cm}
    \STATE {\bfseries Return:} $\widehat{C}_{n, \alpha}(X_{n+1})$
\end{algorithmic}
\end{algorithm}

\begin{algorithm}[tb!]
   \caption{Calculate weighted conformal $p$-value that penalizes noninformativeness.}
   \label{alg:example2}
\begin{algorithmic}
   \STATE {\bfseries Input:} Calibration scores $V_{1:n}:=\{V_1, ..., V_n\}_{i=1}^n$; Test point score $V_{n+1}$; Weight vector $\widetilde{w}:=(\widetilde{w}_1, ..., \widetilde{w}_{n+1})$; Significance level (target miscoverage) $\alpha$. \\
    \STATE {\bfseries Output:} Weighted conformal $p$-value $p_{n+1}^{\tilde{w}}$.
   \IF{$\widetilde{w}_{n+1} < \alpha$}
    \STATE {\fontfamily{qcr}\selectfont Calculate exact weighted conformal $p$-value (Eq. \eqref{eq:weighted_p_vals_arbitrary_def}):} \\
    \STATE Sample $u_{n+1} \stackrel{iid}{\sim}\text{Unif}[0,1]$ \\
    \STATE $p_{n+1}^{\tilde{w}} = \sum_{i=1}^{n+1}\tilde{w}_i\Big[\mathbbm{1}\{v_i > v_{n+1}\} + u_{n+1}\mathbbm{1}\{v_i = v_{n+1}\}\Big].$
     \\
   \ELSE 
   \STATE \COMMENT{{\fontfamily{qcr}\selectfont\color{red}Note: This "else" condition implies the corresponding CP set for the test point was noninformative, i.e., $\widehat{C}_{n, \alpha}(X_{n+1}) = \mathcal{Y}$; thus, penalize noninformativeness by using derandomized, anticonservative $p$-value, equivalent to setting $u_{n+1}=0$. These will cause WCTM to raise an alarm more quickly, as they intentionally break IID uniformity based on the magnitude of $\tilde{w}_{n+1}$.}} \\
   \STATE $p_{n+1}^{\tilde{w}} = \sum_{i=1}^{n+1}\tilde{w}_i\Big[\mathbbm{1}\{v_i > v_{n+1}\}\Big].$
     \\
   \STATE 
   \ENDIF
   \STATE {\bfseries Return:} $p_{n+1}^{\tilde{w}}$
   
\end{algorithmic}
\end{algorithm}




\end{document}